\documentclass[12pt]{article}
\usepackage{fullpage}

\RequirePackage{amsthm,amsmath,amsfonts,amssymb}
\RequirePackage[numbers,sort&compress]{natbib}

\RequirePackage{graphicx}
 \input{MKdefRS.tex}

\title{Optimal Sequential Recommendations: \\Exploiting User and Item Structure}

\date{}

\author{Mina Karzand$^1$ \and Guy Bresler$^2$}

\date{
	$^1$Department of Statistics\\
 University of California, Davis \\ \texttt{mkarzand@ucdavis.edu}\\%
	$^2$Department of Electrical Engineering and Computer Science\\
 Massachusetts Institute of Technology \\ \texttt{guy@mit.edu}\\[2ex]%
}

\usepackage{afterpage}

\begin{document}

\maketitle

\begin{abstract}

We consider an online model for recommendation systems, with each user being recommended an item at each time-step and providing 'like' or 'dislike' feedback. A latent variable model specifies the user preferences: both users and items are clustered into types. The model captures structure in both the item and user spaces, as used by item-item and user-user collaborative filtering algorithms. We study the situation in which the type preference matrix has i.i.d. entries. Our main contribution is an algorithm that simultaneously uses both item and user structures, proved to be near-optimal via corresponding information-theoretic lower bounds. In particular, our analysis highlights the sub-optimality of using only one of item or user structure (as is done in most collaborative filtering algorithms). 
\end{abstract}

\tableofcontents
\clearpage
\setcounter{page}{1}
\pagenumbering{arabic}

\section{Introduction}

The music we listen to, the movies we watch, and the products we buy: more often than not, they are recommended to us by algorithms. Given the importance of these recommendation algorithms, it makes sense to try to design optimal ones. A basic criterion for optimality, that captures the first-order experience of users in a recommendation system, is to maximize the proportion of recommendations that are liked,\footnote{A common strategy for developing recommendation algorithms is to frame it as a matrix completion task. This entails recovering unobserved matrix entries given a subset of observed ones, often assuming low-rank. However, optimizing for accuracy of the entire estimated matrix ignores the user experience: A highly accurate estimate of bad recommendations provides no benefit to the user.
} similar to~\cite{bresler2021regret,heckel2017sample}

The goal of this paper is to gain insight into the design of recommendation algorithms by finding a statistically optimal algorithm within the context of a natural model for recommendation systems. One of our findings is that the best way to obtain information about users and items in order to make good recommendations depends on the time horizon and its relation to various system parameters including the number of users, the diversity of users, and richness of the items; there are a number of operating regimes depending on these parameters. 
It goes without saying that the nature of any insight obtained is intertwined with the choice of model. We use the same model as~\cite{bresler2021regret}, closely related to those studied in~\cite{bresler2014latent,bresler2015regret}. The model is different from those in other papers on the topic; we now motivate its key features. 

\subsection{Our Recommendation System Model}

Recommendation systems are inherently dynamic. Each recommendation gives an opportunity to observe user behavior and thus influences the system's understanding of user preferences and item characteristics, thereby shaping the potential effectiveness of subsequent recommendations. This creates a fundamental challenge: balancing the need to explore  to obtain new information with the desire to leverage existing knowledge to provide high-quality recommendations. This exploration-exploitation dilemma is a central theme in the study of multi-armed bandit (MAB) and related problems (e.g., \cite{bubeck2012regret,lai1985asymptotically,russo2014learning}). A key difference exists, however; whereas MAB algorithms ultimately converge to a single, repeated action, users would generally find repeated recommendations of the same item undesirable. To address this, we impose the constraint that an item can be recommended to a given user at most once, as done previously in~\cite{bresler2014latent,bresler2015regret,ariu2020regret, heckel2017sample}.

Our recommendation system model has a fixed set of users, each of whom is recommended an item at each time-step. The system then receives binary feedback, `like' or `dislike', from each of the users.
The user preferences are described by a latent variable model in which each user is associated with a user type and each item is associated with an item type. 
Users of the same type have identical preferences for all of the items and items of the same type result in identical feedback when recommended to any given user.\footnote{A similar model of data to ours, in which there is an underlying clustering of  rows and columns, has been studied in other settings~\cite{shen2009mining,xu2014jointly}.} This model for user preferences has been motivated empirically in~\cite{bresler2014latent} and captures structure amongst both users and items.
The measure of performance is expected regret, equal to the expected number of bad recommendations made per user over a time horizon of interest. The precise formulation of our model is given in Section~\ref{sec:model}. 

In this model, information about users and items is only obtained via user feedback. In particular, there is no feature data on the users or items (such as age, location, and gender of users or genre, actors, and director of movies) as used by content filtering algorithms; this allows us to better focus on the dynamics driven by information gain over time.

Broadly, algorithms based on observed user preferences are called collaborative filtering (CF)~\cite{goldberg1992using}
and are used by virtually all industrial recommendation systems. There are two main variants: User-user CF, where recommendation to a user is done by finding similar users and recommending items liked by these users; and item-item CF, where items similar to those liked by the user are found and then recommended. 
Versions of item-item~\cite{bresler2021regret,bresler2015regret} and user-user~\cite{bresler2021regret, bresler2014latent,heckel2017sample} CF were analyzed previously.
 The papers \cite{bresler2014latent,bresler2015regret} did not prove lower bounds and hence could not make any claims regarding optimality. Within the same model as we study here,~\cite{bresler2021regret} proved information-theoretic lower bounds showing item-item or user-user CF to be optimal in certain extreme parameter regimes of the model, with structure only in the user space or only in the item space.  
The present paper goes significantly beyond that work by addressing the general situation with nontrivial structure in both item and user space. 

\subsection{Our Contributions}
There are two main contributions of this paper: (1) We prove a novel multi-part information-theoretic lower bound on the regret suffered by any algorithm; (2) 
We propose a new and essentially optimal algorithm that uses both user and item structure in a novel way. The performance guarantee of our algorithm matches our lower bound to within a logarithmic factor. 
Our paper is the first to characterize optimal regret in a model that has both user and item structure, and does so across the entire spectrum of relative richness of item space versus user space. 

Characterizing the best possible performance in a simple model of recommendation systems yields a variety of engineering insights. As described in Section~\ref{sec:mainresult}, there are five different operating regimes for the regret curve in our main result. In each of these there is a different optimal pathway for obtaining information about preferences of users for items. Our algorithm makes use of the optimal pathway in each regime, while our lower bound shows optimality by carefully accounting for the various possible ways of obtaining information.

Our lower bounds are, to the best of our knowledge, the first lower bounds for a natural model of recommendations in the online setting that correctly capture dependence on the item and user structure. 
Our lower bound is obtained by 
arguing that reduced uncertainty about the preference of a user for an item can only be achieved via highly informative and necessarily uncertain recommendations. 
Due to there being both item and user structure, there are several natural ways to get information about the preference of a user for an item. Our lower bound is based on identifying complementary scenarios in which we show that there is insufficient information to have confidence that a given recommendation will be liked, together with showing constraints on the number of occurrences of these scenarios. 

One insight from the analysis in~\cite{bresler2021regret} was that the item-item algorithm must limit the exploration to only a subset of the items types, where the size of this subset depends on the system parameters and time-horizon. It was noted that the straightforward approach to Item-Item CF algorithms is to learn the whole preference matrix, which results in a highly suboptimal cold-start time. This feature also exists in the current paper, where it is crucial to avoid exploring too much -- instead, obtaining just enough information about the model leads to optimal regret.

Our system model, defined precisely in Section~\ref{sec:model}, is quite simple, but already exhibits a complex and rich set of phenomena. While we believe that this model captures the most salient features of item and user structure, a natural next step is to generalize to more intricate item and user structures. In order to better focus on the latent information structures in the problem, feedback from users is noiseless, but it is straightforward to incorporate noise into our model. One might also wish to model arrival and departure of users or items, change of users over time, and so forth.
We hope that this paper can be used as a stepping stone towards the principled design of near-optimal recommendation systems in more complex scenarios.

\subsection{Comparison to Related Work}

We now make some comparisons between our work and several lines of work in the literature. There are four main points of distinction between our work and prior work, each of which is very different from ours in some subset of the following: (i) performance metric;
(ii) exploration due to online formulation;
(iii) optimal use of user and item structure;
(iv) no repetition constraint.

Traditional approaches to recommendation systems often focus on recommending new items by completing a partially observed user-item preference matrix. A core assumption is that a limited number of factors influence user preferences, implying low-rank structure in the preference matrix and allowing to draw on the large literature on low-rank matrix completion
~\cite{tsybakov2011nuclear,recht2011simpler,candes2012exact,srebro2004learning,candes2010power,keshavan2010matrix,bhojanapalli2014universal,candes2010matrix,kerenidis2017quantum}. As noted in the introduction, this approach fails to address the online nature of the problem, in which there is an opportunity to explore in a targeted way, and also the specific loss function over the matrix fails to capture user experience of being recommended items.
 
Hazan et. al.  in~\cite{hazan2024partial} make the observation that not all entries of a partially observed low rank matrix are equally difficult to estimate. Their goal is
 \emph{partial matrix completion}: identifying a large subset of entries that can be completed with high accuracy. 
 This is conceptually related to the fact that our algorithm sometimes ignores a portion of the preference matrix due to high cost of exploration. 
However, \cite{hazan2024partial} is nevertheless in an offline setting and gives no insight towards how best to explore, which is a focus of our work. 
They do also consider a semi-online setting in which at each time-step an entry is revealed adversarially,
and the algorithm proposes a partial completion; it suffers regret based on the worst-possible matrix consistent with the revealed entries. As before, because the task is purely prediction-based without a choice of revealed entry, there is no exploration involved.

Several works carry out online matrix completion using iterative gradient based methods~\cite{jin2016provable,ma2020implicit,tanner2016low,dadkhahi2018alternating}. In these works, the objective is to continually achieve more accurate estimations of all entries of the preference matrix. This is in contrast to the  performance measure in our work, which entails making as many good recommendations as possible in any given time horizon. One key observation that emerges from the analysis of our proposed optimal algorithm is that these two objectives can be contradictory in small time horizons: It is beneficial to intentionally avoid learning some part of the preference matrix completely. One other distinction from those works, as in the prior paragraph, is that the entries are revealed either adversarially or uniformly at random, so there is no exploration. 

In multi armed bandit (MAB) modeling of recommendation systems, each user or user cluster is modeled as an agent and each item corresponds to an arm, as exemplified by~\cite{pal2022online,baby2024online,sen2017contextual,pal2023optimal}. The users receive distinct, possibly correlated expected rewards from each arm. The purpose of exploration is identifying the arm with the highest expected reward for each user, implying that the best arm (item) can be used (recommended) many times after successful identification. Their results quantify the number of samples needed to identify the best arm in terms of the `gap' in the expected reward of the best arm from the suboptimal ones, number of users, etc. In contrast, we impose the constraint that each item can only be recommended at most once to each user, which dramatically reduces the benefit of finding a good item and also exploration must be continually done in order to find new items to recommend. This turns out to completely change the manner in which exploration must be carried out at all time horizons and leads to entirely new optimal algorithms.

The paper~\cite{ariu2020regret} considers a setup similar to ours, but slightly more general, with probabilities in the preference matrix. This captures inhomogeneity in the quality of items. Their results exhibit dependence on certain gaps in these probabilities, as is typical of MAB results, but fail to capture any nontrivial dependence on the structure in the item and user space, in any parameter regime. In contrast, our focus is on the impact of structure among items and users.

A few papers including~\cite{bresler2014latent,bresler2015regret,dabeer2013adaptive,heckel2017sample} have  analyzed online collaborative filtering. The papers~\cite{biau2010statistical,bresler2014latent,heckel2017sample,bresler2021regret} analyze a user-user CF algorithm  and~\cite{bresler2015regret,bresler2021regret} analyzes an item-item CF algorithm in a slightly more flexible model with structured noise. 
All of these papers only consider either user or item structure, but not both simultaneously.
Relative to these, our main contribution is in proposing an algorithm that makes use of structure in both user and item space---and in addition, obtaining nearly tight lower bounds in almost all parameter regimes. 

Dabeer and coauthors~\cite{dabeer2013adaptive,barman2012analysis,aditya2011channel} also use a quite similar model to ours, but with noisy feedback from users. 
They focus on the exploitation phase of an online recommendation system: Given  arbitrary prior noisy ratings, 
they aim to identify the highest probability liked item for each user. They show approximate optimality of their algorithm, for their performance criterion, in various regimes of noise parameter.  
In doing so, their objective and regime of interest ignore the benefit of further exploration and, as a result, the algorithms do not attempt to recommend items which reveal information about the underlying model that will in turn be useful for future exploitations. 

We remark that disallowing repetition in recommendations to a given user has been used before~\cite{ariu2020regret, heckel2017sample,marsden2017sequential,dabeer2013adaptive}.  
The authors in~\cite{marsden2017sequential} propose interesting heuristic algorithms based on information directed sampling; however, they do not derive any bounds on regret.

Algorithms that exploit structure in both the user and item space have been studied before in~\cite{song2016blind,wang2006unifying,kim2010imp,borgs2017thy,borgs2022iterative}. For example,~\cite{song2016blind,borgs2017thy,borgs2022iterative} consider a more flexible latent variable model and study \emph{offline} matrix completion where a subset of the entries are revealed and the goal is to guess the remaining entries. As noted above, the offline problem formulation is very different from our online setting and it does not give insight into how best to explore.

We remark that there have been several exciting developments in offline matrix completion, including a model where the observed entries are not uniformly random (e.g., users are more likely to watch a movie when they have a premonition that they will like it) \cite{agarwal2023causal}. In future work, it may be interesting to combine elements of that model with ours.

\subsection{Notation}
For an integer $a$ we write $[a]=\{1,\cdots,a\}$. For real-valued $x$ let $(x)_+ =\max\{x,0\}$. $\lfloor x \rfloor$ denotes the greatest integer less than or equal to $x$ and $\lceil x \rceil$ denotes the smallest integer greater than or equal to $x$ and $\big\lfloor x\big\rfloor_{+}=\max\{0, \lfloor x \rfloor\}$. All logarithms are to the base of $2$. The set of natural numbers (positive integers) is denoted by $\mathbb{N}$. We note here that variables or parameters in Figure~\ref{f:notation} have the same meaning throughout the paper, but any others may take different values in each section. 
We define $\big\lfloor x\big\rfloor_{+}=\max\{0, \lfloor x \rfloor\}$. 
Numerical constants ($c, c_1,c_2$ and so forth) may take different values in different theorem statements unless explicitly stated otherwise.

\section{Model} \label{sec:model}
\subsection{Problem Setup and Performance Metrics}
We consider $N$ users, represented by the set $\{1,\dots,N\}$. At each time  $t= 1,2,3,\dots$, the algorithm recommends an item $a_{u,t}\in\mathbb{N}$ to each user $u$ and receives binary feedback $L_{u,a_{u,s}}\in \{+1,-1\}$ (representing `like' or `dislike', respectively). Following our earlier constraints, each item is recommended at most once to each user. We assume there is an infinite set of items, identified by the natural numbers, ensuring the algorithm can always make recommendations.

The history of interactions, denoted as 
$\H_{t}$, is defined as the set of all actions and feedback up to time $t$:
 $\H_{t}=\{a_{u,s}, L_{u,a_{u,s}}, \text{ for } u\in [N], s\in [t]\}$. We are interested in online collaborative filtering algorithms. In these algorithms, the action $a_{u,t}$ is determined by a (possibly random) function of the history from the previous time step $\H_{t-1}$. This randomness is encoded by a random variable $\zeta_{u,t}$, which is independent of all other variables. Thus, we have $a_{u,t}=f_{u,t}(\H_{t-1},\zeta_{u,t})$ for a deterministic function $f_{u,t}$.

Algorithm performance is evaluated after an arbitrary number  of time-steps $T$. The performance metric is expected regret (simply called regret), defined as the expected number of disliked recommended items  per user: 
\begin{equation}\label{eq:regdef}
\reg(T)=\Ex\sum_{t=1}^{T}\frac{1}{N}\sum_{u=1}^{N}\ident[L_{u,a_{u,t}}=-1]\,,
\end{equation}
where the expectation is with respect to the randomness in the model and the algorithm.
Our algorithm uses knowledge of the time-horizon $T$. However, using a standard doubling trick (discussed in Appendix~\ref{s:any-time-reg-alg}) makes it possible to convert the proposed algorithm to one achieving the same (up to constant multiplicative factors) regret without this knowledge~(see~\cite{cesa2006prediction,lattimore2020bandit} for an example).
 This alternative, where the algorithm is agnostic to $T$ and needs to perform well over any time interval, is called \emph{anytime regret} in the literature.

Aside from regret, the time at which recommendations become nontrivial is a key performance consideration, as users initially receive little value for their interaction with the system. The notorious ``cold start'' problem in recommendation systems arises from initial scarcity of information. To formalize this, we define the \emph{cold start time} as the first point at which the slope of the regret curve with respect to $T$ falls below the threshold $\gamma/ \log(NT)$:
$$
\textsf{coldstart}(\gamma) = \min\Big\{ T:\frac{\reg(T)}{T}\leq \frac{\gamma}{\log(NT)}\Big\}\,.
$$
Equivalently, this is also the first time that the expected number of bad recommendations is substantially sublinear in the total number of recommendations:
\begin{align}
    \label{eq:coldstartdef}
\textsf{coldstart}(\gamma) = \min\Big\{ T: N\reg(T) \leq \gamma \frac{NT}{\log(NT)}\Big\}\,.
\end{align}
Requiring sublinearity is stricter than the definition of cold start time in~\cite{bresler2015regret} and~\cite{bresler2021regret}, where cold-start is defined as the first stime the slope of regret is bounded by constant $\gamma$. 

\subsection{User Preferences}\label{ss:pref}
We study a latent-variable model for the preferences $L_{u,i}\in \{+1,-1\}$ of the users for the items, based on the idea that there are relatively few \emph{types of users} and/or few \emph{types of items}.
 Each user $u \in [N]$ has a user type $\tau_U(u)$ i.i.d. uniform on $[\qu]=\{1,\dots, \qu\}$, where $\qu$ is the number of user types. 
We assume for ease of presentation of the results that 
$20\qu\log^2 \qu <N$. In this regime there is nontrivial structure in the user space: there are at least $0.5\cdot N/\qu$ users of each type with high probability.
In the complementary regime $\qu\gg N$ most users have their own type and the analysis reduces to replacing $\qu$ by $N$.\footnote{
The remaining (narrow) range $\qu< N< \qu\log^2 \qu $ is addressed in our earlier work \cite{bresler2021regret}.}

 Similarly, each item $i\in\mathbb{N}$ has a  random item type $\tau_I(i)$ i.i.d. uniform on $[\qi]$, where $\qi$ is the number of item types.%
The random variables $\{\tau_{U}(u)\}_{1\leq u \leq N}$ and $\{\tau_I(i)\}_{1\leq i}$ are assumed to be jointly independent. 
We are interested in understanding the system behavior as a function of $\qu$ and $\qi$, which parameterize the complexity of the user preferences. It turns out that somewhat degenerate behavior emerges when these are too small, so
throughout the paper we assume that both $\qi$ and $\qu$ are at least $100\log N$. The factor of 100 is merely there to simplify the mathematical arguments. 
We will also assume $N>100$.

\begin{figure}
\begin{minipage}{.6\textwidth}
 \centering
		\begin{tabular}{ |c|c| } 
			\hline
		$N$ & Number of users \\ 
		$\qu$ & Number of user types \\ 
            $\qi$ & Number of item types \\ 
            $T$ & Time horizon \\ 
		$\tau_U(u)$ & User type of user $u$ \\ 
            $\tau_I(i)$ & Item type of item $i$ \\ 
            $\xi_{\tu,\ti}$ & Preference of user type $\tu$ for item type $\ti$\\
		$\Xi$ & Preference matrix \\
		$L_{u,i}$ & Rating of user $u$ for item $i$\\
            $a_{u,t}$ & Item recommended to user $u$ at time $t$ \\ 
		\hline
		\end{tabular}
	\caption{Notation for the System Model}
	\label{f:notation}
\end{minipage}%
\begin{minipage}{.4\textwidth}
 \centering
 \renewcommand{\arraystretch}{1.4}
		\begin{tabular}{| c| } 
        \hline
		$L_{u,i}=\xi_{\tau_U(u),\tau_{I}(i)}$\\
		\hline 
		 $\tau_U(u) \sim \mathsf{unif}([\qu])$ i.i.d.\\ 
		$\tau_I(i) \sim \mathsf{unif}([\qi])$ i.i.d.\\ 
		$\xi_{\tu,\ti} \sim \mathsf{unif}(\{-1,+1\})$ i.i.d.\\
		\hline
				$N>100$ \\ 
				$N>20\qu\log^2 \qu$ \\
		 $\qu,\qi > 100\log N$\\ 
         \hline
		\end{tabular}
	\caption{Model Assumptions}
	\label{f:assumptions}
\end{minipage}
\end{figure}
 
All users of a given type have identical preferences for all the items, and similarly all items of a given type are rated in the same way by any particular user. The entire collection of user preferences $(L_{u,i})_{u,i}$ is therefore encoded into a much smaller \textit{preference matrix} $\Xi=(\xi_{\tu,\ti})\in\{-1,+1\}^{\qu\times \qi}$, which specifies the preference of each user type for each item type. The preference $L_{u,i}$ of user $u\in [N]$ for item  $i\in \mathbb{N}$ is the preference $\xi_{\tau_U(u),\tau_I(i)}$ of the associated user type $\tau_{U}(u)$ for the item type $\tau_{I}(i)$ in the matrix $\Xi$, \textit{i.e.},
$$L_{u,i}=\xi_{\tau_U(u),\tau_{I}(i)}\,.$$ 
We assume that the entries of $\Xi$ are i.i.d., $\xi_{\tu,\ti}=+1$ w.p. $1/2$ and $\xi_{\tu,\ti}=-1$ w.p. $1/2$. Generalizing our results to i.i.d. entries with bias $p$ is  straightforward. However, the independence assumption among entries of $\Xi$ is quite strong and an important future research direction is to obtain results for more realistic preference matrices. 

A modest generalization of this model is to have noisy feedback with $L_{u,i}=\xi_{\tau_U(u),\tau_{I}(i)} \cdot z_{u,i}$, where $z_{u,i}$ are i.i.d. with $\Pr[z_{u,i} = +1]=1-\gamma$ and  $\Pr[z_{u,i} = -1]=\gamma$. 
In~\cite{bresler2021regret},
the proposed user-user algorithm is modified to handle noisy feedback of this form. However, since the challenge of dealing with noise is a well-studied theme in online learning, and also technically straightforward to handle in our setting, we focus here on the latent-variable structure of users and items and refrain from studying the effect of noise.

\section{Main result}\label{sec:mainresult}

In this section we state our main result, which gives an optimal curve for regret in our recommendation system model. For the various regimes and quantities to be interpretable, we first discuss some conceptual items.

\subsection{Conceptual Preliminaries}

\subsubsection{Comparison of items and users}\label{ss:comparison}
The basic way that information can be extracted from the system is via \emph{comparison} of two items or two users. Two items $i_1$ and $i_2$ are understood to be \emph{similar} if a subset of users $\Uc$ of appropriate size $\log (\qi)$ rates them identically, i.e. $L_{u,i_1} = L_{u,i_2}$ for all $u\in \Uc$. Analogously, two users $u_1$ and $u_2$ are understood to be similar if they rate a random subset of items $\mathcal{I}$ of size $\log (\qu)$ identically, i.e. 
$L_{u_1,i}=L_{u_2,i}$ for all $i\in\mathcal{I}$.

Note that there is an inherent asymmetry between users and items: multiple users can rate a given item simultaneously, but each user rates one and only one item at each time step. Thus, similarity of two items can be determined in as few as two time-steps, while determining similarity between two users requires $\log (\qu)$ time-steps.

\subsubsection{Clustering of Items and Users}

The users and items have types, and any given algorithm will at some point have enough information to cluster the users or (some of) the items. The operating regimes are in part determined by the extent to which items and/or users ought to be clustered within the given time horizon. The nuance here is that even if it is \emph{possible} to cluster users or items, there is a regret cost that must be incurred to obtain the needed information, and this must be weighed against the benefit. In Section~\ref{sec:AlgHybrid} we carry out a careful heuristic analysis of the cost versus benefit of obtaining different types of information that explains the regret curve and the various operating regimes. 

\subsubsection{Partial Learning of Preference Matrix}

Aside from item and user types, the basic uncertainty faced by a recommendation algorithm consists of the preference matrix encoding whether each user type likes or dislikes each item type. The basic fact that all items from a liked type can be recommended to a user, because by definition likes all of them, implies that for small time horizons optimal algorithms avoid exploring the entire preference matrix. This phenomenon is responsible for the specific shape of the regret curve (and enters into the heuristic derivation in Section~\ref{sec:AlgHybrid}).  

\subsection{Theorem Statement}
We next state our main result. We will momentarily interpret it, but for now the main features of the result are: (i) there are five regimes and the regret behaves differently in each; (ii) this regret curve is achievable by the algorithm we propose; (iii) this regret curve is optimal, in that no algorithm can beat it. Thus, the different operating regimes and transition points are inherent to the recommendation system problem.

Our algorithm, {\recsys}, is described at a high level in Section~\ref{sec:AlgHybrid} together with heuristic analysis.
Pseudocode is given in Section~\ref{sec:pseudocode}.

\begin{theorem}[Main Result]
\label{thm:MainResult} 
Under the modeling assumptions in Figure~\ref{f:assumptions},
there are universal constants $c$ and $C$ such that the following holds. For any time horizon $T$, algorithm {\recsys}  achieves $N\reg(T) \leq C\,N\mathtt{R}(T) \log^{3/2} (N\mathtt{R}(T))$, where $\mathtt{R}(T)$ is defined below. Conversely, if  $\qu> (\log\qi)^2$ and $\qi>(\log N)^5$  then any algorithm must incur $N\reg(T)\geq c\,\frac{N\mathtt{R}(T)}{\log (N\mathtt{R}(T))}$. Thus, the optimal total regret is 
$$N\reg(T) = \widetilde{\Theta}(N\mathtt{R}(T))\,.$$

The function $\mathtt{R}(T)$ 
is defined in a piece-wise manner in the table below, with the support of each piece given in terms of the following functions of system parameters: 
\begin{align*}
    T_1 = \log\qu,  \quad
    T_2 = \qi /N,\quad
    T_3 = \qi / \qu,
    \quad
    T_4 = N(\log\qu)^2/\qi,\quad
    T_5 = \qi\qu \,.
\end{align*}

\renewcommand{\arraystretch}{1.6}
\begin{center}
\begin{tabular}[c]{ c | c | c | c}
 $\mathtt{R}(T)$ & Operating Regime &  Range of $T$ & Conditions For Occurring \\
\hline\hline
T &$\TintCold$ & $[1,\, \min\{T_1,T_2\} ]$ \\ 
\hline
$1+\sqrt{\frac{\qi T}{N}}$&$ \TintItem$ & $(T_2,\, T_4]$ & $T_2 \leq T_1$\\ 
\hline
$\log\qu +\frac{\qu}{N}T$ &$\TintUser$ & $(T_1, \, T_3]$ & $T_1<T_2$ and $\log\qi\leq \qu$ \\ 
 \hline
 $ \log \qu+\frac{\sqrt{\qi\qu T}}{N}$ 
& $\TintHyb$ &$(T_3,\, T_5]$ & $T_1<T_2$ and $\log\qi\leq \qu$ \\  
  &&$(T_4,\, T_5]$ & $T_2\leq T_1$ and $\log\qi\leq \qu$ \\  
 \hline
$\log\qu + \frac{\qi \qu}{N}+\frac{\log \qi}{N} T$  & & $(T_5,\, \infty) $ &  $\log\qi\leq \qu$ \\
$\log\qu + \frac{\qu}{N} T$ &  $\TintAsym$  & $ (T_1,\, \infty)$ & $T_1<T_2$ and  $\qu<\log\qi$\\
$\log\qu + \frac{\qu}{N} T $ &   & $(T_4,\, \infty) $ & $T_2\leq T_1$ and  $\qu<\log\qi$
\\ \hline
\end{tabular}
\end{center}

\end{theorem}
\begin{figure}
\centering     
\subfigure[$\Rtt(T)$ when $T_1< T_2$ and $\log\qi\leq \qu$.]{\label{fig:a}\includegraphics[width=0.45\textwidth]{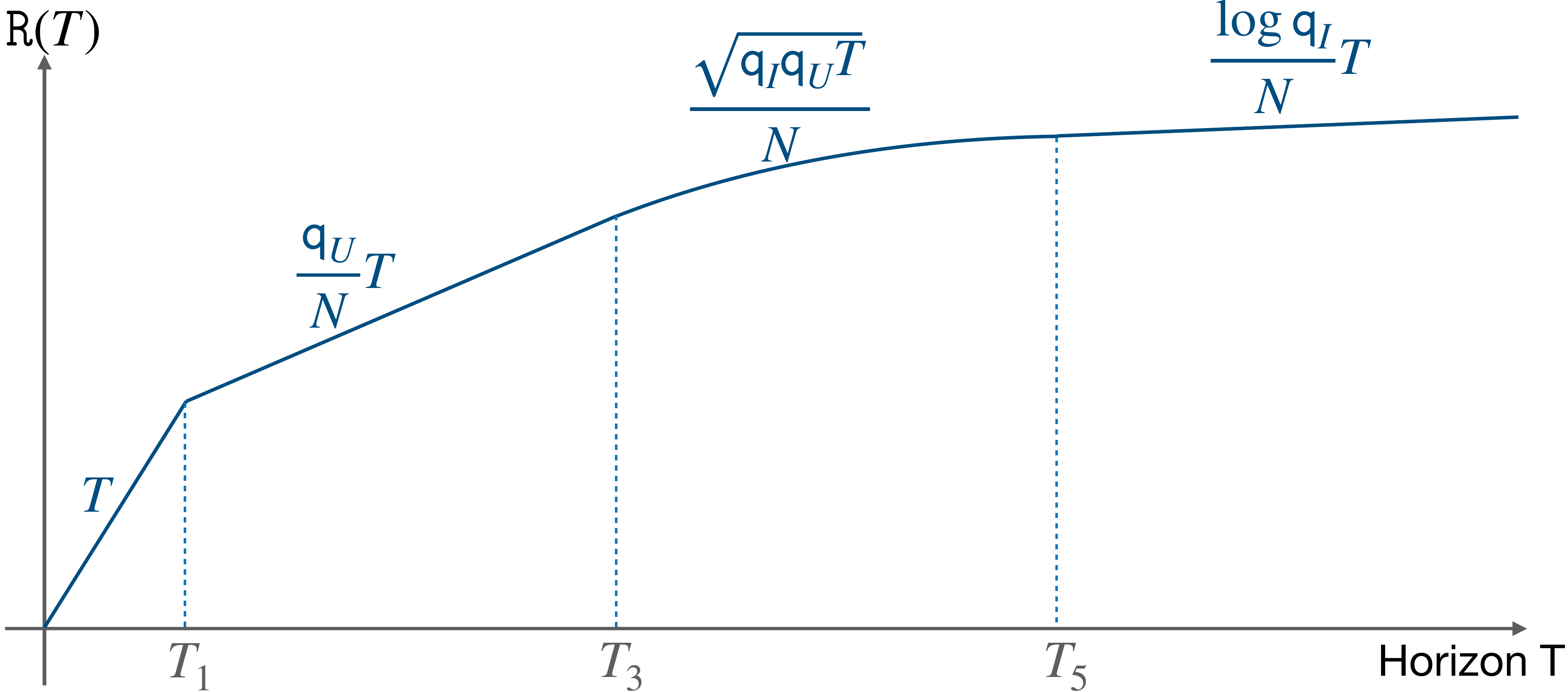}}\hspace{5mm}
\subfigure[$\Rtt(T)$  when $T_2\leq T_q$ and $\qu<\log\qi$]{\label{fig:b}\includegraphics[width=0.45\textwidth]{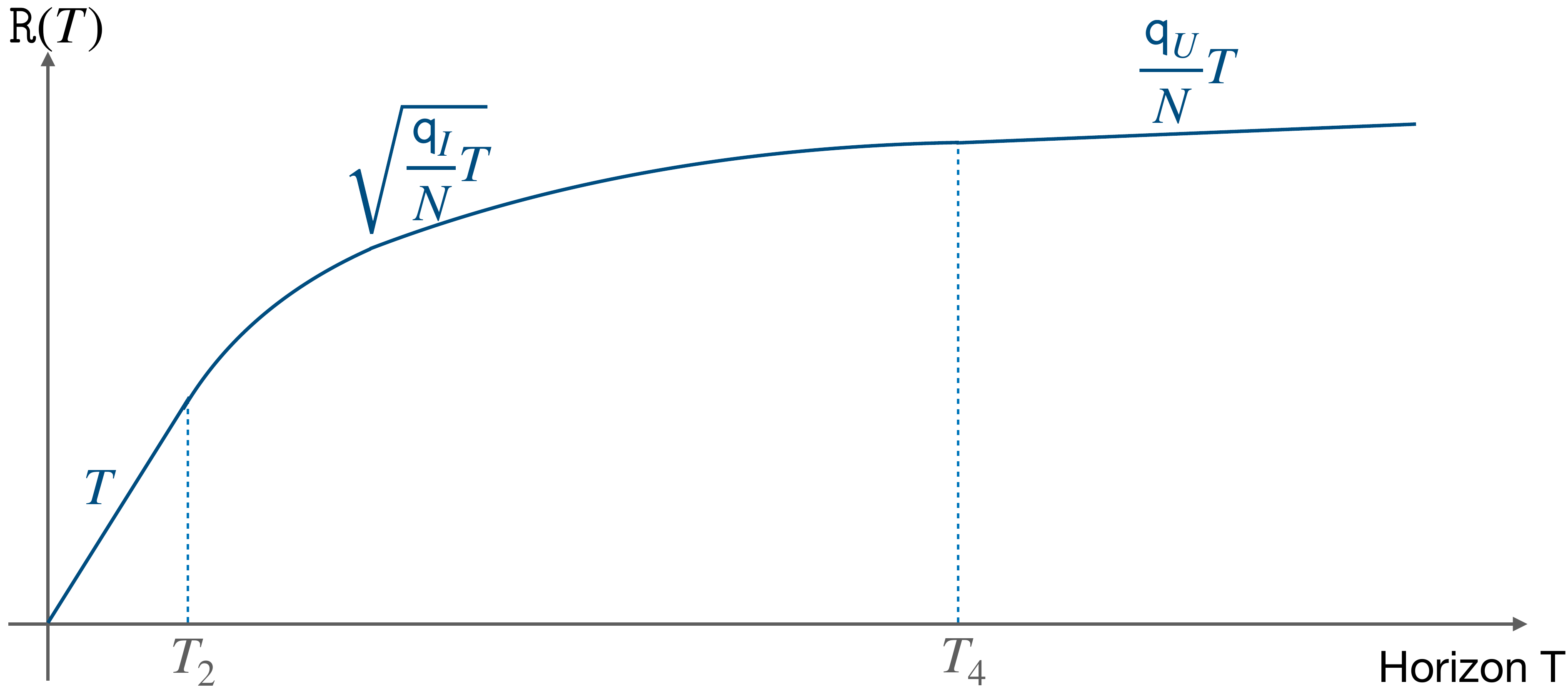}}
\caption{Two possible regret curves $\mathtt{R}(T)$ and their various operation regimes. The piece-wise curves are scaled by the appropriate constant factor so that regret is continuous.}
\label{fig:RegCurve}
\end{figure}

This theorem is proved in Appendix~\ref{sec:ProofMainResult} by combining the lower bound on regret in Theorem~\ref{t:joint-L} and the upper bound in Theorem~\ref{th:Item-upper}.

\begin{corr}[Cold-Start]
\label{cor:coldstart}
Under the modeling assumptions in Figure~\ref{f:assumptions},
and assuming  $\qu> (\log\qi)^2$ and $\qi>(\log N)^5$, 
there are universal constants $\gamma, \gamma' >0$ such that the following hold:
\begin{enumerate}
    \item The cold start time of $\recsys$ is upper bounded as
\[\textsf{coldstart}(\gamma) \leq 
\min\Big\{  \log^2 N, \max\big\{\frac{\qi \log \qi}{N}\log^2(N\qi),\, 16\big\}
\Big\}\,.\]
\item 
Also, the cold start time of any algorithm is lower bounded as
\[\textsf{coldstart}(\gamma') \geq 
\min\Big\{  \log N\, \log\qu, \max\big\{\frac{\qi \log^2 \qi}{N},\, 16\big\}
\Big\}\,.\]
\end{enumerate}
\end{corr}
Thus, the cold start time of our algorithm is essentially optimal up to logarithmic factors. We remark that the maximum with constant 16 is needed, because even as $N\to \infty$ the cold start time does not tend to zero---it is always at least one.
The proof appears in Section~\ref{sec:ProofColdStart}.

\subsection{Discussion of Main Result}

The regret curve in Theorem~\ref{thm:MainResult} is fairly complex, with its various regimes that may appear in different orders or not at all. However, each transition point and portion of regret curve has a simple explanation that gives insight into the basic features of optimal recommendation algorithms. As mentioned above, Section~\ref{sec:AlgHybrid} gives a full explanation of the regret curve. For now, we limit ourselves to some remarks about the implications of Theorem~\ref{thm:MainResult}:
\begin{itemize}
\setlength\itemsep{2pt}
\item For time-horizons in $\TintItem$, $\TintUser$, and $\TintAsym$, it is optimal to use only one of item structure or user structure. 
\item For time-horizons in $\TintHyb$, both item and user structures are used. The cost of learning the preference for an item is shared by all users in a given cluster, and additional items of the same type as a liked item are then also recommended.
\item $T_1$ is the cold-start time (up to multiplicative logarithmic factors) for a pure user-user collaborative filtering algorithm: on the order of $\log \qu$ ratings of the same item are both necessary and sufficient to compare two users. 
\item $T_2$ is the cold-start time (up to multiplicative logarithmic factors) of a pure item-item collaborative filtering algorithm: on the order of $\qi$ items must be seen in order to have two items of the same type, and this effort is divided by the $N$ users. 
\item The condition $T_1<T_2$ corresponds to the regime when the cold-start of an (optimal) user-user algorithm is shorter than the cold-start time of an (optimal) item-item algorithm. An implication of our theorem is that the minimum of these two cold-start times is best-possible. 
\item If $T_1<T_2$ and $\log\qi\leq \qu$, for time horizons larger than $T_3$, using both item and user structure improves on only using user structure. 
\item The significance of $T_4$ is that an optimal item-item comparison algorithm will at that point have recommended enough items to the users that the user clusters become evident, at which point there is no additional cost to using both user and item structure. 
\item For time horizons in $\TintAsym$, an optimal algorithm will have learned the user types, a clustering over a set of items, and the preference matrix encoding the preferences of users for item types. 

However, due to the constraint that recommendations to a given user cannot be repeated, new items must be introduced. Discovering the types of these new items (i.e., how they relate to already-seen items) incurs regret. Alternatively, the preference for the new items can be learned via similarity between the users, completely ignoring item structure. For large time horizons in $\TintAsym$, the optimal algorithm chooses the better of the two options (decided by comparison of $\log \qi$ and $\qu$), rather than both at once. 

\item After the cold-start time, the function $\Rtt(T)$ is the sum of a part that does not depend on $T$ and a part that does. The latter determines the rate of increase of regret per unit of time. This rate is nonincreasing in $T$. The former is the cost of reaching each regime. For example, roughly speaking $\log \qu$ recommendations to each user is necessary to reach the Hybrid regime. This interpretation will be expanded in Section~\ref{sec:AlgOverview}.

\item One surprising consequence of the various regimes is that if for large time horizons in $\TintAsym$ user-user CF is better than item-item CF, then for \emph{all time horizons} it is not helpful to use both item and user structure together.
 
\end{itemize}

\section{Information-theoretic Lower Bound on Regret}

Theorem~\ref{thm:MainResult} constitutes a performance guarantee for our algorithm as well as an information-theoretic lower bound which holds for any algorithm and shows optimality of the one we propose. While both results require new ideas, we feel that the conceptual novelty leading to the lower bound is more significant. For this reason we start with the lower bound.

 In this section we describe the main components of the argument.
We remark that the lower bound implicit in Theorem~\ref{thm:MainResult} is a slightly looser version of the precise bound that we obtain, which is stated as Theorem~\ref{t:joint-L}.

\paragraph*{High-level Strategy}
There are three high-level components to the argument: \begin{enumerate}
    \item Identifying \emph{bad recommendation events}, in which some key portion of information is unavailable, and showing that they induce regret (Section~\ref{sec:BadRecommScenario}).
    \item Deriving various lower bounds on the number of bad events in terms of system trajectories (Section~\ref{sec:ConstraintsBad}). This part is purely combinatorial and entails careful tracking of recommendations in terms of the type of information they can provide.
    \item Combining these bounds to get a lower bound on regret (Section~\ref{sec:CombBad}). Here we make crucial use of the fact that each item can be recommended to each user at most once, as well as proving facts about the empirical frequencies of item and user types.
\end{enumerate}

\subsection{Bad Recommendations and Associated Scenarios}
\label{sec:BadRecommScenario}

\paragraph*{Bad Recommendations}
The history $\H_t$ at time $t$ denotes all feedback received from the users at times $1,\dots, t$. 
To lower bound regret, we define a recommendation $a_{u,t}$ to user $u$ at time $t$ to be a \textit{bad} (or uncertain) recommendation when the posterior probability of the event $\{L_{u,a_{u,t}}=+1\}$ given the history $\H_{t-1}$ is close to (or smaller than) a constant, e.g., $1/3$. 
Note that this refers only to the \emph{confidence} that the recommendation is liked given the history at the moment the recommendation is made:  a bad recommendation is not always disliked.
 A lower bound on regret follows directly from a lower bound on the expected number of bad recommendations (Lemma~\ref{l:jl-reg-lower}). We next identify events leading to bad recommendations.

\paragraph*{Bad Recommendation Events}
We identify four scenarios in which recommending item $i$ to user $u$ at time $t$ is necessarily bad, captured by events $\Bsf{1}_{u,t},\dots, \Bsf{4}_{u,t}$ defined precisely in Section~\ref{sec:LBthrParam} and proved to be bad in Proposition~\ref{p:badUncertain}: 

\begin{enumerate}
    \item $\Bsf{1}_{u,t}$: If at time $t$ item $i$ and user $u$ are associated to too few ratings, then the type of item $i$ and  the  type of user $u$ are both uncertain leading to uncertainty in the \emph{location} of the relevant entry of the preference matrix, which has roughly half $+1$ and half $-1$. 

\item $\Bsf{2}_{u,t}$: Even if the type of user $u$ is known, as long as the type of item $i$ is uncertain and item  $i$ has not been recommended to any user with the same type as $u$, then recommending $i$ to $u$ is a `bad' recommendation.
The uncertainty is in the type of item $i$, which corresponds to choosing the entry in the row of the preference matrix for user $u$'s type, and again the typical row has roughly half each of $+1$ and $-1$. 

\item $\Bsf{3}_{u,t}$: Even if the type of item $i$ is known, as long as the type of user $u$ is uncertain and no item with the same type as $i$ has been recommended to $u$ before, then recommending $i$ to $u$ is a `bad' recommendation.
The uncertainty here is a symmetric situation to the prior one with the roles of user and item flipped.

\item $\Bsf{4}_{u,t}$: Even if the types of item $i$ and user $u$ are both known, 
 if no user with the same type as $u$ has seen an item with the same type as $i$,  then 
 the actual value of the entry of the preference matrix corresponding to the types of item $i$ and user $u$ is uncertain. 
\end{enumerate}
  Figure~\ref{fig:BadRecomm} illustrates bad recommendations of the various types.

\begin{SCfigure}[2][h]
\label{fig:BadRecomm}
  \includegraphics[width=0.40\textwidth]%
    {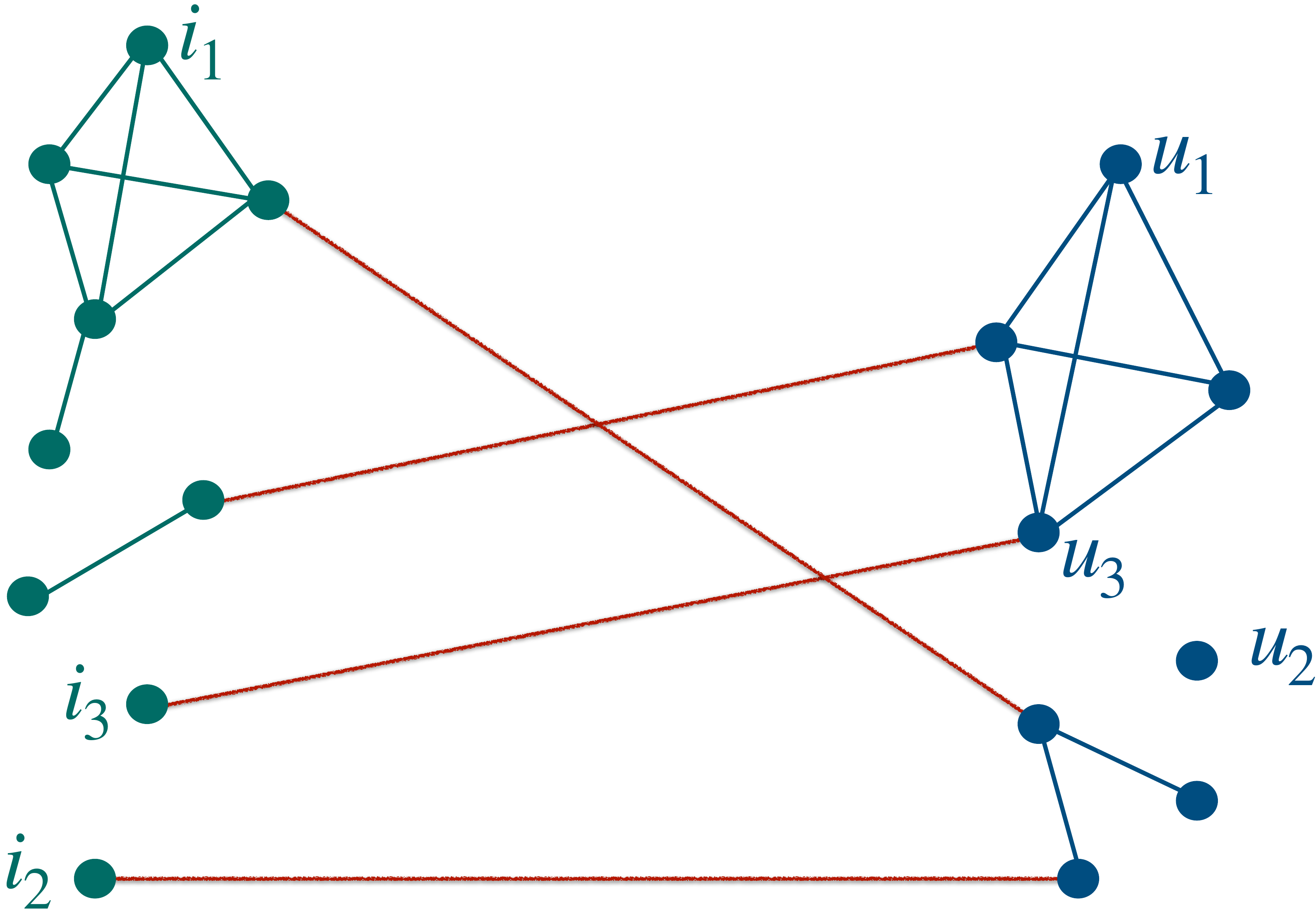}\hspace{-6mm}
      \caption{
      \small The green nodes on the left denote a subset of items and the blue nodes on the right denote a subset of users. Edges are between two items (or two users) that have been compared. Red edges between items and users denote previous recommendations. 
      There are bad recommendations of four types: (i) $i_2\longrightarrow u_2$ is of type $\Bsf{1}_{u,t}$; (ii) $i_2\longrightarrow u_1$ is of type $\Bsf{2}_{u,t}$; (iii) $i_1\longrightarrow u_2$ is of type $\Bsf{3}_{u,t}$; (iv) $i_1\longrightarrow u_1$ is of type $\Bsf{4}_{u,t}$. $i_3\longrightarrow u_1$ is not necessarily a bad recommendation since $i_3$ has been recommended to $u_3$ before, which is presumed to be of the same type as $u_1$.} 
\end{SCfigure}
 \vspace{-5mm}
  
\begin{remark}
The first three categories of bad recommendations can occur even if the preference matrix is entirely known and are due to uncertainty in what is the relevant entry of this matrix. The last category is due to uncertainty in the value of an entry whose relevance may already be nearly certain.
\end{remark}

\subsubsection{Bad Recommendation Scenarios: Formal Definitions}

\label{sec:LBthrParam}

\paragraph*{Knowledge of types} We now introduce notation summarizing the feedback seen by the algorithm, i.e. the history $\H_t$, as turns out to be relevant for learning the type of a user or item. 
\begin{enumerate}
    \item
Let $c_i^t$ be \emph{the number of user types} which have rated item $i$ (strictly) before time $t$, i.e.,
\begin{align}\label{def:jl-cit}
c_i^t &:= 
\sum_{v=1}^{\qu}
\ident\b\{ a_{u,s} =i\text{ for some } s<t \text{ and } u \text{ such that } \tau_U(u) = v \b\}\,.
\end{align}
Inspired by the intution in Section~\ref{ss:comparison},
our proxy for having nontrivial knowledge of the type of an item $i$ is $c_i^t \geq \til $
 where $\til \approx \log \qi$ and will be specified later in~\eqref{eq:jl-thru}.

\item  Let $d_u^t$ be \emph{the number of item types} that were rated by user $u$ strictly before time $t$, i.e., 
\begin{align}\label{def:jl-dut}
d_u^t &:=
\sum_{\ti=1}^{\qi}\ident\b\{\tau_I(a_{u,s})=\ti\text{ for some } s<t\b\}
\,. 
\end{align}
Inspired by the intution in Section~\ref{ss:comparison},
our proxy for having nontrivial knowledge of the type
of a user $u$ is $d_u^t \geq \tul$ where $\tul \approx \log \qu$ will be specified later in~\eqref{eq:jl-thru}. 
\end{enumerate}

Note that  $c_i^t$ and $d_u^t$ are nondecreasing in $t$.

\paragraph*{Events capturing knowledge of preferences}
We will need events (labeled $\mathcal{K}$ for ``knowledge'') capturing information available about the preference of user $u$ for item $i$ assuming that users have been clustered according to types, items have been clustered according to types, or both. These events are motivated by the $\Bsf{b}_{u,t}$ scenarios for $b=2,3,4$ described above.
\begin{enumerate}
    \item  \emph{Preference of a user type for an item.} 
        Let $\BBjul{u}{i}{t}$ be  the event that a user with type $\tau_U(u)$ has rated item $i$ by  the end of time $t-1$, i.e.,
\begin{equation}   \label{def:jl-B-u-tauU}
    \BBjul{u}{i}{t} =
 \Big\{{\text{there exists } u'\in[N]: a_{u',s}=i} \text{ for some } s< t \text{ with } \tau_U(u) = \tau_U(u')\Big\} \,.
\end{equation}
    \item \emph{Preference of a user for an item type.} 
    Let $\BBjil{u}{i}{t}$ be  the event that user $u$ has rated an item of the same item type as item $i$ by the end of time $t-1$, i.e.,
\begin{equation}
    \label{def:jl-B-u-tauI} \BBjil{u}{i}{t} 
= \Big\{
{\text{there exists } i'\in\mathbb{N}: a_{u,s}=i'} \text{ for some } s< t \text{ with } \tau_I(i) = \tau_I(i')\Big\}\,.
\end{equation}
    \item \emph{Preference of a user type for an item type.}
    Let $\BBjjl{u}{i}{t}$ be  the event that a user with type $\tau_U(u)$ has rated an item with type $\tau_I(i)$ by the end of time $t-1$, i.e.,
\begin{align}\notag
\BBjjl{u}{i}{t} = \Big\{ & {\text{there exists } i'\in\mathbb{N}, u'\in[N]: a_{u',s}=i'} 
\text{ for some } s< t 
\\ &\text{ with } \tau_I(i') = \tau_I(i), \tau_U(u')=\tau_U(u)
\Big\}\,.\label{def:jl-B-tauU-tauI}
\end{align}
\end{enumerate}

\paragraph*{Bad recommendations} We now define disjoint events of making bad recommendations corresponding to the four scenarios identified in the prior subsection:
\begin{equation}
\begin{split}\label{eq:jl-def-B}
\Bsf{1}_{u,t} &:=\{c_{a_{u,t}}^t < \til, d_{u}^t < \tul\} \,,
\\
\Bsf{2}_{u,t} &:=\{c_{a_{u,t}}^t < \til, d_{u}^t \geq \tul, (\BBjul{u}{
a_{u,t}}{t})^c\}\,,
\\
\Bsf{3}_{u,t} &:=\{c_{a_{u,t}}^t \geq \til, d_{u}^t < \tul, (\BBjil{u}{a_{u,t}}{t})^c\}\,,
\\
\Bsf{4}_{u,t} &:= \{c_{a_{u,t}}^t \geq\til, d_{u}^t \geq \tul, (\BBjjl{u}{a_{u,t}}{t})^c\}\,.
\end{split}
\end{equation}

These events were described informally near the beginning of Section~\ref{sec:BadRecommScenario}.
The number of bad recommendations up until time $T$ is
\begin{align}
\label{eq:jl-def-bad}
\bad(T)& := \sumTN \ident\b\{\Bsf{1}_{u,t}\cup\Bsf{2}_{u,t} \cup\Bsf{3}_{u,t} \cup\Bsf{4}_{u,t} \b\}\,.
\end{align}

\subsubsection{Bad Recommendations Induce Regret}
    The \emph{regularity property} of the preference matrix, defined in Section~\ref{sec:regularity}, converts uncertainty over columns and rows into uncertainty in the value of the relevant entry. Our i.i.d. preference matrix is easily shown to satisfy this property with probability at least $1-1/(4N)$. 

\begin{prop}[Bad recommendations are uncertain] 
\label{p:badUncertain}
Denoting by $\Omega\in \sigma(\Xi)$ the event that the preference matrix is regular, then
\begin{align*}
    \Pr\b(L_{u,a_{u,t}} &= -1 |\Bsf{b}_{u,t}, \Omega\b)\geq 1/3
\quad\text{for all} \quad b=1,2,3\,,
\\
\Pr\b(L_{u,a_{u,t}} &= -1 |\Bsf{4}_{u,t}\b)\geq 1/2\,.
\end{align*}
\end{prop}
The proposition for $b=3$ will be proved in Section~\ref{sec:BadproofMainBody} and for $b=1, 2$ and $4$ in Appendix~\ref{sec:ProofBadAppendix}. 
We remark that the proofs rely on the assumption that the entries of the preference matrix are independent. Next corollary is a direct consequence of this proposition and is proved in Appendix~\ref{sec:ReginTermsofBadProof}.
\begin{corr}
\label{l:jl-reg-lower}
The regret is lower bounded as
\[\reg(T)\geq  \frac{1}{3N}\,\Exp{\bad(T)}- \frac{1}{12N}T\,.\]
\end{corr}

\subsection{Constraints on Bad Recommendation Scenarios}
\label{sec:ConstraintsBad}
Constraints on the number of bad recommendations arise for two reasons:
\begin{enumerate}
    \item In order for a recommendation to be good (i.e., not bad), any algorithm by definition of the bad scenarios has to have made a certain number of bad recommendations. This constrains the relationship between the number of good and bad recommendations. There are four such constraints, captured by Lemmas~\ref{l:BoundBad1} through~\ref{l:BoundBad4}.
    \item Each user is recommended $T$ distinct items. This constrains the total number of recommendations, as captured by Lemma~\ref{l:BoundBad5}.
\end{enumerate}
We now derive the various constraints on the number of bad recommendations.

\paragraph*{Constraint I: Weakly Explored Items} 
We require a bit of notation. As noted when defining $c_i^{t}$ above, we think of items as being \emph{weakly explored} if they have not been rated by sufficiently many user types to confidently compare them to other items and \emph{strongly explored} if they have. We make the following definitions:
\begin{itemize}
    \item Let $\ISstrong^{t} = \{i: c_i^{t+1} \geq \til\}$  be the set of items that have been rated by at least $\til$ user types by the end of time $t$ and denote its cardinality by
\begin{equation}
\label{eq:jl-fT}
\Istrong^{t} := |\ISstrong^{t}| =\sum_{i\in\mathbb{N}} \ind{c_{i}^{t+1}\geq \til}\,.
\end{equation}
\item %
 Let $\ISweak^{t} =  \{i: 0<c_i^{t+1} < \til\}$ be the set items that have been rated by at least one and fewer than $\til$ user types by the end of time $t$. Let $\Iweak^{t} $ be the number of times any item in $\ISweak^{(t)}$ has been recommended to a user type for the first time. Note that
\begin{equation}
\label{eq:jl-gT}
\Iweak^{t} := 
 \sum_{i\in\mathbb{N}} c_{i}^{t+1} \,\,\ind{0 < c_{i}^{t+1}< \til}.
\end{equation}
\item Let $\IStotal^{t} = \ISstrong^{t} \cup \ISweak^{t}$ be the set of all items that have been recommended to at least one user by the end of time $t$ and denote its cardinality by
\begin{equation}
\label{eq:jl-nT}
\Itotal^{t} := |\Itotal^{t}|=
 \sum_{i\in\mathbb{N}} \ind{0 < c_{i}^{t+1}}\,.
\end{equation}
\end{itemize}
We now bound the number of bad recommendations in terms of $\Istrong^{T}$ and $\Iweak^{T} $.

\begin{lemma}
\label{l:BoundBad1}
The  number of bad recommendations $\bad(T)$ satisfies 
\begin{equation}\label{eq:jl-bad-gT-fT}
\bad(T)\geq
\sumTN
 \b(\ind{\Bsf{1}_{u,t}} +\ind{\Bsf{2}_{u,t}}\b)
 \geq 
\Istrong^{T} \til + \Iweak^{T}\geq \Itotal^{T} \,.
\end{equation}
\end{lemma}
\begin{proof} The inequality $\Istrong^{T}\til + \Iweak^{T}\geq \Itotal^{T}$ is obvious. We proceed with the middle inequality.
First, observe that
$$
\ind{\Bsf{1}_{u,t}} +\ind{\Bsf{2}_{u,t}} \geq \ind {c_{a_{u,t}}^t < \til, (\BBjul{u}{
a_{u,t}}{t})^c}\,.
$$
Recall from~\eqref{def:jl-B-u-tauU} that $(\BBjul{u}{
a_{u,t}}{t})^c$ occurs if $u$ is the first of its type to rate item $a_{u,t}$. 
We will check that
$$
\sumTN \ind {c_{a_{u,t}}^t < \til, (\BBjul{u}{
a_{u,t}}{t})^c} = \Istrong^{T}\til + \Iweak^{T}.
$$
First, consider an item in $\ISstrong^{T}$.  
For an item to be in $\ISstrong^{T}$, it (by definition) has been rated by at least $\til$ user types, and 
the first $\til$ of these each contribute to $\ind {c_{a_{u,t}}^t < \til, (\BBjul{u}{
a_{u,t}}{t})^c}$. Next, consider the items in $\ISweak^{T}$. 
By definition, $\Iweak^{T}$ is precisely the number of contributions to $\ind {c_{a_{u,t}}^t < \til, (\BBjul{u}{
a_{u,t}}{t})^c}$
\end{proof}

\paragraph*{Constraint II: Weakly Explored Users} 
As noted when defining $d_u^t$ above, we think of users as being \emph{weakly explored} if they have not  rated sufficiently many item types to compare them to other users.

Let $\USweak^t=\{u\in [N]: d_u^{t+1} < \tul\}$ be the set of users that have rated fewer than $\tul$ distinct item types by the end of time $t$ and let 
\begin{equation}\label{eq:jl-wT}
\Uweak^t:= \frac{1}{N}|\USweak^t|=\frac{1}{N}\sum_{u\in[N]}\ind{d_u^{t+1}< \tul}
\end{equation} 
be the fraction of users in $\USweak^t$.
Note that since $d_u^t$ is nondecreasing in $t$, for any $u\in\USweak^T$ we have $d_u^t<\tul$ for all $t\leq T+1$. 

\begin{lemma}
\label{l:BoundBad2}
The  number of bad recommendations $\bad(T)$,  defined in Equation~\eqref{eq:jl-def-bad}, satisfies 
\begin{equation}\label{eq:jl-wT-N-rU}
\bad(T)\geq 
\sumTN \b(\ind{\Bsf{1}_{u,t}} +\ind{\Bsf{3}_{u,t}}\b) \geq 
(1-\Uweak^T) N \tul\,.
\end{equation}
\end{lemma}
\begin{proof}
First, observe that
\begin{equation}
    \label{e:B1B3}
    \ind{\Bsf{1}_{u,t}} +\ind{\Bsf{3}_{u,t}} \geq \ind {d_u^t < \tul, (\BBjil{u}{
a_{u,t}}{t})^c}\,.
\end{equation}
Recall from~\eqref{def:jl-B-u-tauI} that $(\BBjil{u}{
a_{u,t}}{t})^c$ is the event that $a_{u,t}$ is the first item of its type to be recommended to $u$.
For a user to not be in $\Wc^{T}$, it (by definition) has rated at least $\tul$ item types, and 
the first $\tul$ of these each contribute to $\ind {d_u^t < \tul, (\BBjil{u}{
a_{u,t}}{t})^c}$. 
\end{proof}


\paragraph*{Constraint III: New Item Types to (Possibly) Explored Users}
It will be useful to consider $ \gamma_*^T$, defined as the minimum over all users $u$, of the number of item types that have been rated by $u$ through time $T$. More formally,
\begin{equation}
   \gamma_*^T := \min_{u} \B|\B\{\ti\in[\qi]: \sum_{t\in[T]
   }\ident[\tau_I(a_{u,t})=j]> 0\B\}\B| = \min_u d_u^{T+1}\,.
\label{def:jl-gamma-U}
\end{equation}
The latter equality follows from the definition of $d_u^t$ in~\eqref{def:jl-dut} as the number of item types that have been rated by $u$ before time $t$.

\begin{lemma}
\label{l:BoundBad3}
The  number of bad recommendations $\bad(T)$,  defined in Equation~\eqref{eq:jl-def-bad}, satisfies 
\begin{align*}
\bad(T)
\geq \sumTN
 \b(\ind{\Bsf{1}_{u,t}} +\ind{\Bsf{3}_{u,t}}\b)
\geq
\Uweak^T N\gamma_*^T\,.
\end{align*}
\end{lemma}
\begin{proof}
For any $u\in\USweak^T$,
\begin{align*}
\sum_{\substack{t\in[T]}}
 \b(\ind{\Bsf{1}_{u,t}} +\ind{\Bsf{3}_{u,t}}\b) \stackrel{(a)}\geq 
 \sum_{\substack{t\in[T]}}
 \ind {d_u^t < \tul, (\BBjil{u}{
a_{u,t}}{t})^c}\stackrel{(b)}= 
 d_u^{T+1}  \stackrel{(c)}\geq \gamma_*^T\,.
\end{align*}
Here (a) follows from \eqref{e:B1B3} in the prior lemma; the LHS of (b) is counting the first $\tul$ times a new item type is recommended to $u$, and the equality holds since any user $u\in\USweak^T$ has rated $ d_u^{T+1} < \tul$ item types; 
(c) is by~\eqref{def:jl-gamma-U}. Summing over the $N \Uweak^T$ users in $\USweak^T$ gives the lemma.
\end{proof}

\paragraph*{Constraint IV: First of an item type to a user type} 
 Let $\minUser$ be the maximum number of users of any type:
\begin{equation}\label{def:jl-minUser}
    \minUser = \max_{\tu\in[\qu]} \b| \big\{u\in [N]: \tau_U(u)=\tu \big\} \b|\,.
\end{equation}
The number of \emph{represented user types}, with at least one user of that type, is at least ${N}/{\minUser}$.

\begin{lemma}
\label{l:BoundBad4}
The  number of bad recommendations is bounded as 
$
\bad(T)
\geq
\gamma_*^T N/\minUser\,.
$
\end{lemma}
\begin{proof}

First, observe that
$
\BBjjl{u}{
a_{u,t}}{t} \supseteq \BBjil{u}{
a_{u,t}}{t} \cup \BBjul{u}{
a_{u,t}}{t}
$, from which it follows that
\begin{equation}    
    \label{e:B1B2B3B4}
    \ind{\Bsf{1}_{u,t}} +\ind{\Bsf{2}_{u,t}} + \ind{\Bsf{3}_{u,t}}+\ind{\Bsf{4}_{u,t}}\geq \ind { (\BBjjl{u}{
a_{u,t}}{t})^c}\,.
\end{equation}
In words, $(\BBjjl{u}{
a_{u,t}}{t})^c$ is the event that $a_{u,t}$ is the first item of its type to be recommended to any user of type $\tau_U(u)$, which always yields a bad recommendation. 
The proof follows directly from definitions~\eqref{def:jl-gamma-U} and~\eqref{def:jl-minUser}: each of the at least $N/\minUser$ represented user types has rated at least $\gamma_*^T$ item types by time $T$.
\end{proof}

\paragraph*{Constraint V: Total number of recommendations}

This step is based on the fact that the total number of recommendations made is $TN$. 
\begin{lemma}
\label{l:BoundBad5}
The  number of bad recommendations $\bad(T)$ satisfies 
\[\bad(T) \geq TN- \Big[\Istrong^T (N-\til) + \Iweak^T (\minUser-1)\Big]\,. \]
\end{lemma}
\begin{proof}
The number of \emph{good} recommendations, i.e., those that are not bad, is $TN-\bad(T)$. 

Now, any good recommendation of an item must be either in $\ISstrong^T$ or $\ISweak^T$.
Any $i\in\ISstrong^T$ can be recommended at most $N$ times and as discussed in the proof of Lemma~\ref{l:BoundBad1}, at least $\til$ of these recommendations contribute to $\bad(T)$. 
So the total number of good recommendations in $\ISstrong^T$ is at most $\Istrong^T (N- \til)$.
Any $i\in\ISweak^T$ is recommended to $c_i^{T+1}$ user types, each of which have at most $\minUser$ users. The first time $i$ is recommended to a user type, it is a bad recommendation 
in scenario $\Bsf{1}_{u,t} $ or  $\Bsf{2}_{u,t} $. So the total number of good recommendations of $\ISweak^T$ is at most $\sum_{i\in\ISweak^T} c_i^{T+1} (\minUser-1) = \Iweak^T (\minUser-1)$. 
\end{proof}


\subsection{Combining Constraints on Bad Recommendations}
\label{sec:CombBad}
Combining Lemmas~\ref{l:BoundBad2} and~\ref{l:BoundBad3} gives
\begin{align*}
\bad(T)
&\geq
\max\b\{
\Uweak^T N\gamma_*^T,\, 
(1-\Uweak^T) N \tul
\b\}
\\
&\geq
\min_{0\leq \Uweak^T\leq 1}\max\b\{
\Uweak^T N\gamma_*^T,\, 
(1-\Uweak^T) N \tul
\b\}
\\
&=
N\,
\frac{\tul \gamma_*^T }{\tul+\gamma_*^T }
\geq
\frac{1}{2}N\,
\min\{\gamma_*^T,\, \tul\}
\,.
\end{align*}
%
Combining Lemmas~\ref{l:BoundBad1} and~\ref{l:BoundBad5} yields
\begin{align*}
\bad(T)
&\geq
\max\b\{\Istrong^{T}\til + \Iweak^{T} ,\, 
TN- \big[\Istrong^T (N-\til) + \Iweak^T (\minUser-1) \big]
\b\}
\\
&\geq
\min_{\Istrong^T,\Iweak^T}
\max\b\{
\Istrong^{T}\til + \Iweak^{T} ,\, 
TN- \big[\Istrong^T (N-\til) + \Iweak^T (\minUser-1) \big]
\b\}
\\
&\geq
T \min\{\til, {N}/{\minUser}\}
\,.
\end{align*}
The last two displayed equations together with Lemmas~\ref{l:BoundBad1} and~\ref{l:BoundBad4} give
\begin{align}
\label{eq:LBcombineSteps}
\bad(T)
&\geq
\max\B\{ \Itotal^T,\,
\frac{N}{\minUser}\gamma_*^T,\, 
\frac{N}{2}
\min\{\gamma_*^T,\, \tul\}
,\,
 T \min\big\{\til, {N}/{ \minUser}\big\}
\B\}\,.
\end{align}

\paragraph*{Probabilistic Bounds}

The bound in \eqref{eq:LBcombineSteps} is in terms of random variables $\minUser$ and $\gamma_*^T$. We next identify an event of probability $1/4$ where these variables can be crudely bounded in terms of the parameters of the model $\qu, \qi, N$, and $T$ as well as the number $\Itotal^T$ of items seen by the algorithm. 

\begin{lemma}
\label{l:BoundminUserSize}
Assume $N>2\qu (\log\qu)^2$ and let $\minUser$ be as in Eq.~\eqref{def:jl-minUser}. Then $\Pr(\mathcal{E}_1)\geq 3/4$ for the event $\mathcal{E}_1$ defined as
\begin{align}
\label{eq:BoundminUserSize}
\minUser <
{3N}/{\qu}\,.
\end{align}
\end{lemma}
The proof (in Appendix~\ref{sec:BoundminUser}) of this lemma is based on standard balls and bins analysis: each user (ball) has type (bin) independently uniformly distributed on $[\qu]$, and $\minUser$ is the max load of any bin. 

We next relate $\gamma_*^T$ and $\Itotal^T$.


\begin{lemma}
\label{l:JL-bdGammaStarP}
Let $\gamma_*^T$ and $\Itotal^T$ be as defined in~\eqref{eq:jl-nT} and~\eqref{def:jl-gamma-U}. Then  $\Pr(\mathcal{E}_2)\geq 1/2$ for the event $\mathcal{E}_2$ defined as
\begin{align}
\label{eq:bdGammaStar}
\Itotal^T \geq 
\begin{cases}
T\,,
  \,\quad&
  \text{if }\,  T/\gamma_*^T\leq 2
 \\
\sqrt{\qi}/3 \,,
  \,\quad&
  \text{if }\,  2<T/\gamma_*^T\leq 8\log\qi
 \\
 T \qi /(4\gamma_*^T)\,, 
 \,\quad&
 \text{if }\, 8\log \qi< T/\gamma_*^T \,.
\end{cases}
\end{align} 
\end{lemma}
In what follows, we will only use the last two regimes in the above bounds on event $\mathcal{E}_2$.
The proof of this lemma is deferred to Appendix~\ref{sec:lbdRjTPproof}, but the intuition is as follows: Each of $\Itotal^T$ items has one of $\qi$ item types, so there are roughly $\Itotal^T/\qi$ items of each type. Each user is recommended $T$ distinct items of at least $\gamma_*^T$ types, so $\gamma_*^T\Itotal^T/\qi \gtrsim  T$.
Some care is necessary 
due to the adaptive nature of recommendations, and these approximations are made formal by defining appropriate martingales. 


\begin{corr} 
\label{corr:LowerBoundBad}
The expected number of bad recommendations is lower bounded as
\begin{align}
\label{eq:LowerBoundBad}
\Ex\big[\bad(T)\big]
\geq 
\frac{1}{16}
\min_{\gamma\geq 1}
\max\big\{
f_1(\gamma), f_2(\gamma), f_{3}(\gamma)\big\}\,.
\end{align}
where
\begin{align}
f_1(\gamma) 
&=
\begin{cases}
T\qi /\gamma & \text{ if } \gamma < T/(8\log \qi)
 \\
\sqrt{\qi}  & \text{ if }  T/(8\log \qi) \leq \gamma < T/2
 \end{cases}
 \quad \quad
f_2(\gamma)= \max\big\{ \qu\gamma , N\min\{\tul,\gamma\}\big\}
\notag
\\
&åf_3(\gamma)=
T \min\big\{\til, \qu\big\}\,.
\label{eq:deffgammaLB}
 \end{align}
\end{corr}
\begin{proof}    
By the union bound and Lemmas~\ref{l:BoundminUserSize} and~\ref{l:JL-bdGammaStarP}, $\Pr(\mathcal{E}_1\cap \mathcal{E}_2)\geq 1/4$. On event $\mathcal{E}_1\cap \mathcal{E}_2$ we can eliminate $\minUser$ and $\Itotal^T$ from 
\eqref{eq:LBcombineSteps}. Hence, 
\begin{align*}
\Ex\big[\bad(T)\big]
&\geq 
\frac{1}{4}\Ex\big[\bad(T) \, \b| \mathcal{E}_1\cap \mathcal{E}_2\, \big]
\\
&\geq
\frac{1}{4}\Ex\left[
\max\B\{ \Itotal^T,\,
\frac{N}{\minUser}\gamma_*^T,\, 
\frac{N}{2}
\min\{\gamma_*^T,\, \tul\}
,\,
 T \min\big\{\til, {N}/{ \minUser}\big\}
\B\}
\, \b| \mathcal{E}_1\cap \mathcal{E}_2\, \right]
\\
&\geq 
\frac{1}{4}\Ex\left[ 
\frac{1}{4}
\max\B\{ f_1(\gamma_*^T), f_2(\gamma_*^T), f_3(\gamma_*^T) 
\B\}
\, \b| \mathcal{E}_1\cap \mathcal{E}_2\, \right]
\\
&\geq
\frac{1}{16}
\min_{\gamma\geq 1}
\max\big\{
f_1(\gamma), f_2(\gamma), f_{3}(\gamma)\big\}\,.\qquad\qquad\qquad\qedhere
 \end{align*}
\end{proof}
Note that the functions $f_1(\gamma), f_2(\gamma)$ and constant function $f_{3}(\gamma)$ are parameterized by the parameters of the system $\qu, \qi, N$ and $T$. So depending on the realization of these parameters, the above minimax takes different forms. 
Computing this, although not particularly insightful, and plugging it into Corollary~\ref{l:jl-reg-lower} immediately gives the lower bound for regret in the following Theorem~\ref{t:joint-L} as done in Appendix~\ref{sec:thm-joint-L}.

\begin{theorem}\label{t:joint-L}
Let $\tul$ and $\til$ be as defined in Eq.~\eqref{eq:jl-thru}.
Any recommendation algorithm must incur regret lower bounded as below with numerical constant $c>0$
\begin{align*}
    N\reg(T) \geq & c
    \max\Big\{N, \min\{NT, N\tul, \sqrt{\qi}\}, \min\{\qu T, \sqrt{\qi}\},
    \\
   & \min\big\{\frac{NT}{\log\qi},\sqrt{T\qi N}, N\tul\big\},
    \min\{\frac{\qu T}{\log\qi}, \sqrt{T\qi\qu}\},\, 
    \min\big\{T\til, T\qu\big\}\Big\}\,.
\end{align*}
\end{theorem}
This lower bound on regret takes various function forms depending on operating regimes. Figure~\ref{fig:LBReg} in Appendix~\ref{sec:thm-joint-L} shows all possible function forms for this lower bound depending on the model parameters $\qi, \qu$, and $N$. 

\subsection{Proof of Proposition~\ref{p:badUncertain}}
\label{sec:BadproofMainBody}

\subsubsection{Regularity of Preference Matrix}
\label{sec:regularity}
Suppose the preference matrix $\Xi$ is known.
The columns of $\Xi$ correspond to different item types. We want to control the uncertainty in a user's type, i.e. row, given the user's preferences $\bx\in\{-1,+1\}^s$ for some number $s$ of item types. Given $\bx$, the possible user types are those that have preferences consistent with $\bx$. 

The following definition captures that the number of these possible types is close to its expectation.
We make use of a bit of notation:  For $m\times n$ matrix $A$, let $A_{\bullet j}$ denote the $j$th column of $A$, and 
	for ordered tuple of distinct (column) indices $\bj=(j_1,\dots,j_s)\in [n]^s$, let $A_{\bullet \bj}$
 be the ${m \times s}$ matrix formed from the columns of $A$ indexed by $\bj$. The $i$-th row of this matrix is denoted by $A_{i,\bj}$. For given row vector $\bx\in\{-1,+1\}^s$, let $$\Lambda_{\bx}(A_{\bullet \bj})=\{i\in [m]:A_{i,\bj}=\bx\}$$ be the  indices of rows in $A_{\bullet \bj}$ that are identical to $\bx$.

\begin{definition}[$(s,\eta)$-column regularity]\label{def:ul-reg}
	Fix $s$ and $\eta$. Let $A\in \{-1,+1\}^{m \times n}$.
Matrix $A$ is said to be \emph{$(s,\eta)$-column regular} if 
	\[\max_{\bx,\bj}\Big| \,|\Lambda_{\bx}(A_{\bullet \bj})|-\frac{m}{2^{s}}\Big| \leq \eta\cdot \frac{m}{2^{s}}\,,\]
	where the maximum is over tuples $\bj$ of $s$ distinct columns and $\bx\in \{\pm1\}^s$.
	We define $\rgl_{s,\eta}$ to be the set of $(s,\eta)$-column regular matrices. We also define $\Omega_{0,\eta}$ to be all $\pm1$ matrices.
\end{definition}

For capturing uncertainty over the type of an item, given preferences of several user types for it, we define the analogous property for the rows of a matrix.  

\begin{definition}[$(s,\eta)$-row regularity]\label{def:row-reg}
The matrix $A\in\{-1,+1\}^{n\times m}$ is said to be $(s,\eta)$-\textit{row regular} if its transpose $A^\top\in \rgl_{s,\eta}$ is $(s,\eta)$-column regular. 
\end{definition}

\begin{claim}\label{cl:ul-smalltregular}If a matrix $A\in\{-1,+1\}^{m\times n}$ is \emph{$(s,\eta)$-column regular}, then it is also \emph{$(s',\eta)$-column regular} for all $s'<s$.\end{claim}

This claim follows from the triangle inequality applied to the definition of regularity.
The next lemma states that uniformly random binary matrices are regular with high probability.
Its proof follows easily from Chernoff and union bounds, and is omitted.

\begin{lemma}\label{l:ul-regularity}
	Let matrix $A\in\{-1,+1\}^{m \times n}$ have independent $\unif(\{-1,+1\})$ entries. 
	If $\eta<1$, then $A$ is $(s,\eta)$-column regular with probability at least $$
 \Pr(A\in\Omega_{s,\eta})\geq 1-2(2n)^s \exp\B( - \frac{\eta^2}{3}\frac{m}{2^{s}}\B)\,.$$
\end{lemma}

\begin{corr} [Regularity of the preference matrix]
\label{cor:regularityprob}
Set
\begin{align}
 \tul & :=\big\lfloor\log \qu - \log\log \qi -\log\log N-12\big\rfloor_{+}\,,\qquad\notag
\\
\til & :=\big\lfloor 0.99 \log \qi - 4\log\log N - 12\big\rfloor_{+}\,,
\qquad
\text{and}\quad 
\eta  :={1}/13\,. \label{eq:jl-thru}
\end{align}
Let $\Omega\in\sigma(\Xi)$ be the event that the preference matrix is both $(\tul,\eta)$-column regular and $(\til,\eta)$-row regular, \textit{i.e.}, 
$\Omega=\{\Xi\in \Omega_{\tul,\eta}\}\cap \{\Xi^T\in \Omega_{\til,\eta}\}$. Then $\Pr(\Omega)\geq 1-1/(4N)$. 
\end{corr} 
The proof is immediate using the model assumptions in Figure~\ref{f:assumptions} $\qi, \qu>100 \log N$,  $N>20\qu\log^2\qu$ and, $N>100$.

\subsubsection{Probability of Liking Item in Terms of Regularity}

There are four scenarios to consider in Prop.~\ref{p:badUncertain}. 
Leaving the others to Appendix~\ref{sec:ProofBadAppendix}, we focus on the third one ($b=3$) here. The proof of some lemmas used for Prop.~\ref{p:badUncertain} with $b=3$ are also deferred to Appendix~\ref{sec:A3propbound}. This proposition for $b=3$ states
\begin{align}
    \Pr(L_{u,a_{u,t}}=-1\, |\, c_{a_{u,t}}^t \geq \til, d_{u}^t < \tul, (\BBjil{u}{a_{u,t}}{t})^c, \Omega)\geq 1/3\,.
    \label{eq:PropLBb3}
\end{align}
Item $a_{u,t}$, which has been rated by at least $\til$ user types, is being recommended to user $u$ who has rated fewer than $\tul$ item types, and user $u$ has not rated an item of type $\tau_I({a_{u,t}})$. 

It is convenient to interpret the item type function $\tI:\mathbb{N}\to [\qi]$ assigning types to the items as a sequence and user type function assigning types to the users $\tU:[N]\to [\qu]$ as a vector. 
 Let $\tUnu:=\{\tau_U(u'): u'\in [N]\setminus\{ u \}\}$ denote the vector of user types for all users except $u$.
Recall that $\H_{t-1}$ denotes the feedback obtained through time $t-1$.  

The next two lemmas express the posterior probability under $(\BBjil{u}{i}{t})^c$
of user $u$'s preference for item $i$, conditional on the user feedback $\H_{t-1}$ obtained so far \emph{as well as the preference matrix $\Xi$, item types $\tau_I$, and all user types $\tUnu$ except for $u's$}, in terms of a quantity that is controlled by the column regularity. We emphasize that conditioning on $\Xi, \tau_I, \tUnu$ yields a simple expression for the posterior in terms of these quantities, but 
does \emph{not} mean giving the information to the algorithm producing recommendation $a_{u,t}$. 

First, in Lemma~\ref{l:posteriortypes}, we will show that posterior distribution of types of users and items are uniform over the set of types that are \emph{consistent} with the history so far. We use the notation introduced in Section~\ref{sec:regularity}.  

\begin{lemma}
\label{l:posteriortypes}
Given $\H_{t-1}=h$ in which user $u$ has not rated item $i$,
let $\bj_u\in \sigma(\H_t,\tIni)$ be the set of item types previously rated by user $u$ and row vector $\bx_u\in \sigma(\H_{t-1},\tIni)$ be the $\pm 1$ feedback. 
Similarly, Let $\bj_i\in \sigma(\H_{t-1},\tUnu)$ be the set of user types that previously rated item $i$ and row vector $\bx_i\in \sigma(\H_{t-1},\tUnu)$ be the $\pm 1$ feedback.  Then, 
     \begin{align*}
\Pr&\big(\tau_U(u)=v, \tau_I(i)=j \,\big|\, \H_{t-1}=h, \Xi, \, \tIni, \tUnu\big)
=
\frac{\ind{v\in \Lambda_{\bx_u}(\Xi_{\bullet \bj_u})}}{|\Lambda_{\bx_u}(\Xi_{\bullet \bj_u})|}
\,\, \frac{\ind{j\in \Lambda_{\bx_i}(\Xi^T_{\bullet \bj_i})}}{|\Lambda_{\bx_i}(\Xi^T_{\bullet \bj_i}))|}\,.
\end{align*}
\end{lemma}
We defer the proof to Appendix~\ref{sec:A3propbound}. 
The next Lemma is a simple consequence of the former. 
Lemma~\ref{l:posteriortypes} shows that the probability of an event (e.g., item is liked) is given by counting the number of consistent types satisfying the relevant property.
Next, Lemma~\ref{l:A3propbound} proved in Appendix~\ref{sec:A3propboundproof} expresses these counts in terms of the preference matrix. 

\begin{lemma}
\label{l:A3propbound}
Let $\bj_u\in \sigma(\H_{t-1},\tI)$ and  $\bx_u\in \sigma(\H_{t-1},\tI)$ be as defined in Lemma~\ref{l:posteriortypes}. 
Let $\bx^+$ be the vector $\bx$ appended by $+1$. Then 
\begin{align*}
\Pr\big(L_{u,i}=+1 \,\big|\, \H_{t-1}, \,\Xi, \, \tI, \tUnu, (\BBjil{u}{i}{t})^c \big)
&=
    \frac{|\Lambda_{\bx_u^+}(\Xi_{\bullet ,(\bj_u,\tau_I(i))})|}
    {|\Lambda_{\bx_u}(\Xi_{\bullet \bj_u})|}\,.
\end{align*}
\end{lemma}

This lemma is useful because the count on the RHS can be bounded using regularity of the matrix $\Xi$, as long as $|\bj_u|$ is not too large. We carry this out next.


\subsubsection{Proof of Proposition~\ref{p:badUncertain} with $b=3$}
We need to show the upper bound in Eq.~\eqref{eq:PropLBb3}
\begin{align}\label{eq:jl-bdA3statement}
 \Pr\big(
L_{u,a_{u,t}}=+1 \,\big|\, \histE{3}(u,a_{u,t})
\big)
\leq 2/3\,.
\end{align} 
where to shorten the notation, given user $u$ and item $i$ we defined the event $\histE{3}(u,i)$ and the set $\hist{3}(u,i)$ as follows:
\begin{align}
\histE{3}(u,i) & = \big\{c_i^t \geq \til, \,d_u^t < \tul,\,
\big(\BBjil{u}{i}{t}\big)^c,\, \Omega\Big\}\\
\hist{3}(u,i)&=
\big\{
 \text{realizations of }
\big(\H_{t-1}, \,\Xi, \, \tI, \tUnu \big)
\text{ such that }
\notag
\\&
 \quad\quad   \histE{3}(u,i) \text{ holds}
\text{ and } a_{u,s}\neq i \text{ for all } s<t
\, \big\} 
\label{eq:defA3Event}
\end{align}
be the set of possible realizations of the model parameters and history up to time $t-1$  consistent with the event conditioned upon in~\eqref{eq:jl-bdA3statement}.

\paragraph*{Set $\hist 3(u,i)$ is well-defined}
We note that the condition in \eqref{eq:defA3Event} is a function of the variables $\big(\H_{t-1}, \,\Xi, \, \tI, \tUnu \big)$, so one can determine whether the latter satisfies the former. We spell this out as follows.
The history $\H_{t-1}$ determines whether item $i$ has been recommended to user $u$ before or not.
The value of $c_i^t$ (as in Definition~\ref{def:jl-cit}) is a function of the previous recommendations, summarized in $\H_{t-1}$, and the type of all users  except user $u$, i.e., $\tUnu$
\footnote{Conditioning on the event $a_{u,t}=i$ implies that 
user $u$ has not been recommended item $i$ by time $t-1$. 
Hence ambiguity in the type of user $u$ does not make the value of $c_i^t$ ambiguous.}. 
Similarly, the value of $d_u^t$ (as in Definition~\ref{def:jl-dut}) is a function of the previous recommendations, summarized in $\H_{t-1}$, and the type of all items  except item $i$, i.e., $\tIni$. 
Furthermore, the row regularity and column regularity of the preference matrix $\Xi$ is a function only of $\Xi$.

\paragraph*{Utilizing the regularity property}
Let $\bj_u$ and $\bx_u$ be as defined in Lemma~\ref{l:posteriortypes}.
Note that $\bj_u$ is  a deterministic function of realization of $\H_{t-1}$ and $\tI$.
Given any realization of $\H_{t-1}$ and $\tI \in \hist{3}(u,i)$, the type of user $u$ is uniform on $\Lambda_{\bx_u}(\Xi_{\bullet \bj_u})$ and $|\bj_u|= d_u^{t}<\tul$. 

 So for any realization of 
 $\big(\H_{t-1}, \,\Xi, \, \tI, \tUnu\big)
\in\hist{3}(u,i)$,
 \begin{align}
\Pr\Big(L_{u,i}=+1 \,\big|\, \H_{t-1}, \,\Xi, \, \tI, \tUnu \Big)
& \overset{(a)}{=}
 \frac{|\Lambda_{\bx_u^+}\b(\Xi_{\bullet (\bj_u,\tau_I(i))}\b)|}
 {|\Lambda_{\bx_u}\b(\Xi_{\bullet \bj_u}\b)|}
\overset{(b)}{\leq }
\frac{(1+\eta)2^{|\bj_u|}}
{(1-\eta) 2^{|\bj_u|+1}} 
 \overset{(c)}{\leq }
2/3\,,
\label{eq:BoundProbH3}
\end{align}
where (a) uses Lemma~\ref{l:A3propbound}. 
(b) uses the column regularity of matrix $\Xi$ ($ \Xi \in \rgl_{\tul,\eta}$ as in Definition~\ref{def:ul-reg}) and $|\bj_{u}|<\tul$. (c) uses
$\eta=1/13$.

\paragraph*{Applying tower property}
Using the total probability lemma on above display, for any $i$ such that $a_{u,s}\neq i$ for all $s<t$, 
\begin{align}
\label{eq:BoundProptowerH3}
   \Pr\big(L_{u,i}=+1 \,\big|\, \H_{t-1}, \histE{3}(u,i) \big) \leq 2/3\,.
\end{align}

Recall that there is a random variable $\zeta_{u,t}$, independent of all other variables, such that $a_{u,t}=f_{u,t}(\H_{t-1},\zeta_{u,t})$, for some deterministic function $f_{u,t}$. 
Also, for all $i$ such that $u$ has not rated $i$ before, $\histE{3}(u,i) \in \sigma(\H_{t-1}, \Xi, \tI, \tUnu)$. This is proved using the the same justification as above showing that the set $\hist{3}(u,i)$ is well-defined. Clearly, $L_{u,i}\in \sigma(\Xi,\tU, \tI)$.
So conditioning on $\H_{t-1}$, $\{a_{u,t}=i\}$ is independent of event $\histE{3}(u,i)$ and $L_{u,i}$. Hence, 
\begin{align}
       \Pr\big(L_{u,i}=+1 \,\big|\, \H_{t-1}, \histE{3}(u,a_{u,t}), a_{u,t}=i \big)
       &=   \Pr\big(L_{u,i}=+1 \,\big|\, \H_{t-1}, \histE{3}(u,i), a_{u,t}=i \big)\notag
       \\&=
      \Pr\big(L_{u,i}=+1 \,\big|\, \H_{t-1}, \histE{3}(u,i) \big)\,.
         \label{eq:NExtItemH3}
\end{align}


We use these properties to get
  \begin{align*}
\Pr&\big(L_{u,a_{u,t}}=+1 \,\big|\,  \H_{t-1}, \histE{3}(u,a_{u,t}) \big)
\\& \overset{(a)}{=}
\sum_{i\in\mathbb{N}}
\Pr\big( a_{u,t}=i \,\big|\, \H_{t-1}, \histE{3}(u,a_{u,t}) \big)
\Pr\big(L_{u,i}=+1 \,\big|\, \H_{t-1}, \histE{3}(u,a_{u,t}), a_{u,t}=i \big)
\\
&\overset{(b)}{=} \sum_{i: a_{u,s} \neq i \text{ for all } s<t}
\Pr\big( a_{u,t}=i \,\big|\, \H_{t-1}, \histE{3}(u,a_{u,t})\big)
 \Pr\big(L_{u,i}=+1 \,\big|\, \H_{t-1}, \histE{3}(u,i) \big)
\overset{(c)}{\leq} 2/3\,.
\end{align*}
Here (a) uses total probability lemma $\{L_{u,a_{u,t}}=+1\} = \bigcup_{i\in\mathbb{N}} \{a_{u,t}=i, L_{u,i}=+1\}$ and Eq.~\eqref{eq:NExtItemH3}. (b) uses the assumption that an item can be recommended at most once to a user to take the sum only over the items never recommended to $u$ before. Using Eq.~\eqref{eq:BoundProptowerH3} in each term of the sum gives (c). 

Using the tower property over $\H_{t-1}$ on the above display gives the statement of proposition for $b=3$ in~\eqref{eq:jl-bdA3statement}.
 \qed


\section{Optimal Recommendation Algorithm and its Approximate Regret: Heuristic Analysis}
\label{sec:AlgHybrid}

Theorem~\ref{thm:MainResult} is somewhat complicated, with a number of different regimes and behaviors. 
The aim of this section is to explain how the regret in Theorem~\ref{thm:MainResult} arises by giving an informal description of our algorithm and a back-of-the-envelope style of heuristic analysis. We defer the formal treatment to Section~\ref{sec:pseudocode} where we provide detailed pseudocode and Appendix~\ref{sec:performance} where we prove the associated performance guarantee.

Our recommendation algorithm optimally uses structure in both the item space and user space.
The algorithm makes explicit exploration and exploitation recommendations. 
Feedback from the exploration steps allows the algorithm to estimate a portion of the underlying model that is sufficient for making good recommendations in most of the exploitation steps. 

There are different types of exploratory recommendations, and each has an associated cost and benefit. Later in this section we will heuristically bound the regret achieved by the algorithm by optimizing the parameters of the algorithm controlling the types of exploration.

\subsection{Comparison of Items and Users}
As explained in Section~\ref{ss:comparison}
two items $i_1$ and $i_2$ are declared to be \emph{similar} if a random subset of users $\Uc$ of appropriate size $\tiu$ rates them identically. Similarly, two users are assumed to be of the same type if they rate  $\tuu$ random items similarly. The choice of $\tuu$ and $\tiu$ will be determined later to give small probability of error in clustering items and users.

\subsection{Overview of Algorithm}
\label{sec:AlgOverview}
Algorithm {\recsys} is given below as Algorithm~\ref{alg:joint-fixed} in Section~\ref{sec:pseudocodedetail}.  The algorithm  has several operating regimes depending on the parameters of the model, $\qu,\qi, N$, and time horizon $T$. 
Our description will focus on the regime 
in which the algorithm utilizes the structure in both item space and user space.
In certain parameter regimes the algorithm makes use of only one of user structure or item structure. 

{\recsys} performs the following procedures (not in this order): 
\paragraph*{User Clustering}
A set of items $\IsetU$ is selected. The items in $\IsetU$ are recommended to every user and feedback is received. The 
users are then partitioned into clusters, with the users in each cluster having identical feedback for items in $\IsetU$.

\paragraph*{Item Clustering} A set $\IsetR$ of item cluster representatives is selected, forming the cluster centers. A set $\IsetE$ of items to be clustered is selected. To form the clusters, exploratory recommendations are used to compare items in $\IsetE$ to the representatives $\IsetR$.

\paragraph*{Find Preferences}
Each item cluster representative in $\IsetR$ is recommended to one user from each user cluster. The feedback determines the learned preference of each user cluster for each item cluster. 


\paragraph*{Exploitation}
Finally, in the exploitation phase, each `like' feedback in the previous step results in all items in the item cluster being recommended to all users in the user cluster.

\subsection{Detailed Algorithm Description and Cost of Exploration}
\label{sec:AlgDesc}

We now explain the main procedures in the algorithm and in parallel determine the cost in regret of each. 
As noted above, our description will focus on the regime 
in which the algorithm utilizes the structure in both item space and user space. 

Recall that regret is the expected number of disliked items recommended per user~\eqref{eq:regdef}. 
Recommendations made by our algorithm occur in one of two phases: (i) \textsc{Explore}, in which case they turn to be highly uncertain and disliked with probability half, or (ii) \textsc{Exploit}, in which case they turn out to be liked with high probability. If there are insufficiently many exploitable items, then the \textsc{Exploit} phase ends prematurely and the remaining recommendations are again highly uncertain. Thus, the regret is roughly given by the number of exploratory recommendations, which we call the \emph{cost} of exploration (and which has contributions from several types of exploration), plus the deficit in exploitable items. 

Whether or not structure in the item space or user space (or both) is used depends on the \emph{cost versus benefit} of making use of this structure.
The cost of making use of the structure is the number of exploratory recommendations needed to learn. The benefit is in reduced cost for subsequent learning or in obtaining items to be later recommended in the  exploit phase.

\paragraph*{Algorithm Description} 
At this point we recommend that the reader glance at the pseudocode in Section~\ref{sec:pseudocode}.
{{\recsys}} (Alg.~\ref{alg:joint-fixed}), 
as indicated just above, selects three sets $\IsetE, \IsetU$ and $\IsetR$, containing $\IE, \IU$ and $\IR$ random distinct items.\footnote{The use of three sets $\IsetE, \IsetU$ and $\IsetR$, as opposed to just one, is to eliminate dependencies and simplify the analysis. In a practical implementation it would make sense to use just one set.}  The numbers $\IE, \IU$ and $\IR$ are algorithm parameters that will be chosen to minimize regret.

 Next, the algorithm calls \textsc{Explore} (Alg.~\ref{alg:joint-itemExplore}) and \textsc{Exploit} (Alg.~\ref{alg:jointExploit}).  \textsc{Explore} forms the set of exploitable items by using the feedback from calling \textsc{UserClustering} (Alg.~\ref{alg:UserClustering}), \textsc{FindPrefs} (Alg.~\ref{alg:FindPref}), and  \textsc{ItemClustering} (Alg.~\ref{alg:ItemClustering}). 
 We next describe  each part of \textsc{Explore} and identify the costs associated to each of them.

\paragraph*{User Clustering}

In \textsc{UserClustering} the items in $\IsetU$ are recommended to every user and feedback is received.
The selected set of items $\IsetU$ is either of size zero, in which case no recommendations are made in this phase, or of size $\IU= \tuu$, where $$\tuu=\lceil2\log (N\qu^2)\rceil\,.$$ This is the value of $\tuu$ discussed in Section~\ref{ss:comparison} with the choice of error probability bound $\epsilon = 1/\qu N$. 
 The 
users are then partitioned into clusters, denoted by $\{\Pc_w\}_w$, with the users in each cluster $\Pc_w$ having identical feedback for items in $\IsetU$. If 
$\IU=0$, then the clustering is the trivial one with each user in its own cluster and $\Pc_w = \{w\}$.
Otherwise, with $\IU= \tuu$ the choice $\epsilon = 1/\qu N$ gives probability of at least $1-1/N$ that the users are clustered correctly according to their types (Lemma~\ref{l:ju-user-mis-pr}). 

\paragraph*{Cost of User Clustering} If user clustering is carried out, then the number of exploratory recommendations is $\IU$ for each user, for a total of 
\begin{equation}\label{e:CostUser}
	N\cdot \IU \approx  N\cdot \tuu\,.
\end{equation}
Here and below the `$\approx$' denotes approximation up to a (small) constant factor.

\paragraph*{Find Preferences}
In \textsc{FindPrefs} each item cluster representative in $\IsetR$ is recommended to one user from each user cluster. The feedback determines the learned preference of each user cluster for each item cluster. The number of time steps needed is larger for small user clusters than for large ones, and clusters that finish this task are recommended random items until the slowest (smallest) cluster finishes.

\paragraph*{Cost of Finding Preferences} 
The number of users in the smallest cluster is $\min_w |\Pc_w|$, hence it takes time $\IR/\min_w |\Pc_w|$ to obtain feedback for all the items in $\IsetR$ from the smallest user cluster. 
Note that if there is no user clustering, then $ |\Pc_w|=1$ for each $w$, and if the users are clustered according to their types, then $ |\Pc_w|\approx N/\qu$.
It follows that total number of recommendations made in this exploratory phase is
\begin{equation}\label{e:findPrefCost}
\IR\cdot \frac{	N  }{\min_w |\Pc_w|} \approx
\IR \cdot \big(\qu\cdot \ind{\IU\geq \tuu} + N\cdot \ind{\IU< \tuu}\big)\,.
\end{equation}

User clustering enables collaboration to reduce the cost of finding item preferences: to learn the preference of users for an item cluster, one only needs to recommend the representative from the item cluster to one user from each user cluster. This reduces the cost by a factor of $\frac{\qu}{N}$. 

\paragraph*{Item Clustering}
The set of items $\IsetR$ serve as cluster representatives and the items in $\IsetE$ are clustered by comparing them to the items in $\IsetR$.
The preferences of all user clusters (and hence all users, assuming the user clustering is correct) for the items in $\IsetR$ are known from the previous step. In \textsc{ItemClustering} the items in $\IsetE$ are recommended to $\tiu$ random users, where 
$$\tiu=\lceil 2\log N\qi\rceil\,,$$ and the feedback is used to cluster the items relative to items in $\IsetR$. The number of comparisons $\tiu$ ensures that with probability $1-1/N$ each item in $\IsetE$ is clustered correctly according to its type (Lemma~\ref{l:PrErrItemClass}). Conceptually, it may be useful to think of recommending the items in $\IsetE$ to random user clusters, in order to compare the items to $\IsetR$ which already have user cluster preferences. We recommend to random users (rather than random user clusters) to remove dependence and simplify the analysis. 

When $\IE=0$, the item clustering is the trivial one with each cluster consisting only of its item representative.

\paragraph*{Cost of Item Clustering}
The number of exploratory recommendations is $\tiu$ for each item in $\IsetE$, for a total of 
\begin{equation}\label{e:itemClusterCost}
\IE\cdot \tiu\,.
\end{equation}

\subsection{Cost of Insufficient Exploration}

\label{sec:OpRegimes}

Here we assume that the clustering of users and items by types is correct, a high probability event due to choice of parameters $\tuu$ and $\tiu$.

\subsubsection*{Exploiting only a portion of the item space}  
Recall that the preference of every user is learned for items in $\IsetR$ and that the algorithm clusters items in $\IsetE$ relative to $\IsetR$ such that items with types not present in $\IsetR$ are discarded in the clustering step. 
Thus the size of $\IsetR$ (and the corresponding number of item types in $\IsetR$) controls both the cost of clustering items and the proportion of items in $\IsetE$ that can be clustered and subsequently exploited. 
 
 Let $\ell=\k(\IR)$ denote the number of distinct item types in the set $\IsetR$ of item representatives. 
This is a random function of $\IR=|\IsetR|$, but it turns out that the random fluctuations do not play a significant role in the analysis, so in what follows we treat $\ell=\k(\IR)$ as deterministic.
A fraction $\ell/\qi$ of the item space is learned, and in order to balance the cost versus benefit described in the prior paragraph it is sometimes necessary to take $\ell$ strictly less than $\qi$. 
 
 We next determine how $\ell$ impacts the number of recommendable items from $\IsetE$.

\subsubsection*{Obtaining sufficiently many exploitable items} The outcome of all the exploration steps is a set of items to be recommended that the algorithm is confident will be liked. If for a particular user there are not enough items to recommend for the given time horizon $T$, then the algorithm recommends random items to the user, and these items are disliked with probability half.

For each user, the set of exploitable items includes the items in $\IsetE$ with types that (1) are among the learned types and (2) are liked by the user. A fraction of $\ell/\qi$ of item types are learned and each item type is liked with probability $1/2$. Hence, the number of exploitable items from $\IsetE$ is roughly 
$$
    \IE\cdot \ell/2\qi\,.
$$ In addition, the item representatives $\IsetR$ have been explored by one user per cluster and can be recommended to the other users in the cluster (if any), yielding an additional number of exploitable items per user of
\begin{align*}
\f\IR2 \B( 1- \frac1{\min_w|\Pc_w|}\B) &= \f\IR2  \Big( 1-\frac1N \Big({\qu\cdot \ind{\IU\geq \tuu}\atop +\; N\cdot \ind{\IU< \tuu}}\Big) \Big) \\&=  \f\IR2 \Big( 1-\frac \qu N \B)\ind{\IU\geq \tuu}\,.
\end{align*}
All together, the number of time steps without exploitable items is roughly
\begin{equation}\label{e:NumExpItems}
   \Big(T-\frac{\IE\cdot \ell}{2\qi}- \f\IR2 \Big( 1-\frac \qu N \B)\ind{\IU\geq \tuu}\Big)_+ \,.
\end{equation}
Note that for simplicity we have ignored the time steps consumed by exploration, as this does not significantly change the outcome.
Due to the assumption $2\qu \log\qu <N$,  the factor of $(1-\qu/N)$ in the last term of~\eqref{e:NumExpItems} is a constant and is removed in~\eqref{e:totReg}.

\subsection{Regret in Terms of Costs}

Fix algorithm parameters $\IU, \IE,\IR$. Recall that $\tuu$ and $\tiu$ are functions of the number of types $\qu$ and $\qi$ and represent the number of comparisons needed to determine similarity of a pair of users or items, respectively. 
The total cost of exploration plus the cost due to insufficiently many exploitable items as determined in
Equations~\eqref{e:CostUser},~\eqref{e:findPrefCost},~\eqref{e:itemClusterCost}, and~\eqref{e:NumExpItems}
leads to a regret of
\begin{align}
\label{e:totReg}
    2N&\cdot\reg(T)\; \approx\;
    \\
     N\cdot \IU  + & \IR   \Big({\qu \ind{\IU\geq \tuu}\atop + N \ind{\IU< \tuu}}\Big)+ \IE \tiu +   N\Big(T-\frac{\IE \ell}{2\qi}
    -\f\IR2 \ind{\IU\geq \tuu}\Big)_+\,.\notag
\end{align}
Here $\ell = \k(\IR)$, hence the total regret is a function of the algorithm parameters $\IU, \IE,\IR$.   The factor of $2$ in front of the regret on the left-hand side is due to random recommendations being liked with probability $1/2$.

\subsubsection*{Understanding Cost Versus Benefit}

Before optimizing \eqref{e:totReg} over $\IU, \IE,\IR$, we informally examine the cost versus benefit associated to varying each of these parameters. 

To begin, we see that increasing $\IU$ to $\tuu$ 
allows to cluster the users and reduces the second term, the cost of learning the user preferences for the item representatives, because only a single user in each cluster needs to give feedback for each item in $\IsetR$. The fourth term also decreases, since items in $\IsetR$ are recommended to other users from clusters where the selected user liked the item.
The cost of increasing $\IU$ is reflected in the first term, since all users are recommended all items in $\IsetU$. 

Next, increasing $\IE$ grows the set of items that are clustered to the representatives $\IsetR$. This yields, for each of the $\ell$ item types represented in $\IsetR$, on the order of $\IE/\qi$ items that can be recommended in the exploit phase. This is reflected in the fourth term. The third term is the cost of clustering these items, $\tiu$ for each item in $\IsetE$. 

Finally, increasing
$\IR$ improves the fourth term in two ways: first by increasing the number of items $\IR$ to be exploited in the case that the users are clustered, and second by increasing the number $\ell$ of item types in $\IsetR$ which increases the proportion of items in $\IsetE$ which can be usefully clustered to item representatives and which in turn results in more exploitable items. 
The cost of increasing $\IR$ appears in the second term in \eqref{e:totReg}, and is based on each user cluster rating each item in $\IsetR$.

The effects of changing the algorithm parameters are evidently inter-related, and we optimize them in the next subsection.

\subsection{Optimizing the Costs}
\label{sec:OptCostHeuristic}
The algorithm parameters $\IU, \IE,\IR$ are chosen to minimize the total regret in \eqref{e:totReg}. The optimal choice depends on the system parameters $N,T,\qu,\qi$ and there are different algorithm operating regimes depending on these parameters. We return to the regimes momentarily, but first make a simplifying observation.  

The factor $N$ in front of the last term in \eqref{e:totReg} is so large that 
it is optimal to choose $\IE$ and $\IR$ to make the last term zero, 
i.e., to ensure that the set of exploitable items is large enough that no random recommendations are made in the \textsc{Exploit} phase. 
Intuitively, this is because if random recommendations are going to be made at the end of the \textsc{Exploit} phase, one might as well instead increase the \textsc{Explore} phase with no penalty. 
Hence, to minimize regret, the parameters $\IU, \IE,\IR$ are chosen to minimize the total cost of exploration subject to the constraint $$T\leq \f{\IE \ell}{2\qi} + \f\IR2 \ind{\IU\geq \tuu}$$ 
and taking into consideration 
the third term of \eqref{e:totReg} leads to the optimal choice
\begin{equation}
    \label{e:IEstar}
\IE= \f{2\qi}{\ell}\B(T -  \f\IR2 \ind{\IU\geq \tuu}\B)_+\,.
\end{equation}
Plugging  into \eqref{e:totReg} we get
\begin{align}\notag
   & 2N\cdot\reg(T)\; \approx\;
     \\&\qquad 
     \label{e:totReg2}
     N\cdot \IU + \IR \cdot  \Big({\qu\cdot \ind{\IU\geq \tuu}\atop +\; N\cdot \ind{\IU< \tuu}}\Big)+ \f{2\qi\tiu}{\ell}\B(T -  \f\IR2 \ind{\IU\geq \tuu}\B)_+ \,.
\end{align}

It remains to optimize over $\IU$ and $\IR$.
Starting with $\IU$, note that it is always optimal to set $\IU$ to one of $\tuu$ or $0$. The former corresponds to using the user structure while the latter corresponds to ignoring the user structure. Explicitly,
\begin{equation}
\label{e:IUstar}
\IU = \tuu\cdot \indB{N\tuu+\IR\qu 
+\f{2\qi\tiu}{\ell}\B(T -  \f\IR2 \B)_+< \IR N + \f{2\qi\tiu}{\ell}T }\,.
\end{equation}
Plugging into \eqref{e:totReg2} yields
\begin{align}
\label{e:totReg3}
    2N\cdot\reg(T)\; &\approx
     \min\B\{N\tuu+\IR\qu 
+\f{2\qi\tiu}{\ell}\B(T -  \f\IR2 \B)_+,\, \IR N + \f{2\qi\tiu}{\ell}T\B\}\,.
\end{align}

For the purpose of the slightly heuristic exposition of this section we make the approximation (accurate within a factor of $\log \qi$)
\begin{equation}\label{e:k}
    \ell = \k(\IR) \approx \min\{\IR/2 , \qi\}
    \,.
\end{equation}
(See Lemma~\ref{l:boundsizeLu} for a more accurate bound.)
Recall that $\IR$ is an algorithm design parameter that can be optimized over. 
Plugging~\eqref{e:k} into~\eqref{e:totReg3} and
minimizing over $\IR$ (after some simplification, see Appendix~\ref{s:OptAppCost}) yields 
\begin{equation}\label{e:regA}
\reg(N,T,\qu,\qi) \approx \min\{ f(N,T,\qu,\qi), g(N,T,\qu,\qi)\}\,,
\end{equation}
where 
\begin{align*}
  &f(N,T,\qu,\qi)  = 
  \begin{cases}
     \tuu/2 + \qu N\inv T, \quad 
     &\text{if } T \leq \qi  \frac{\tiu}{\qu} \text{ or } \qu\leq\tiu
     \\\noalign{\vskip2pt}
       \tuu/2 +2 N\inv\sqrt{\qu \qi \tiu T} - \qi \tiu N\inv, \quad 
     &\text{if } \qi \frac{\tiu}\qu< T\leq \qi \frac{\qu}{\tiu}
     \\\noalign{\vskip2pt}
       \tuu/2 + \qi N\inv (\qu-\tiu) + \tiu N\inv T, \quad 
     &\text{if }  \qi \frac{\qu}{\tiu} < T \text{ and } \qu>\tiu
     \end{cases}
     \\
 & g(N,T,\qu,\qi) 
     = 
 \begin{cases}
    2 \sqrt{N\inv \qi \tiu T}, \quad &\text{if } T \leq \qi \frac{N}{\tiu}
    \\\noalign{\vskip2pt}
    \qi + \tiu N\inv T, \quad &\text{if } T >\qi \frac{N}{\tiu}\,.
    \end{cases}
\end{align*}
In the next subsection we refer to the five functions appearing in $f$ and $g$ as
$\rg 1,\dots, \rg5$ (in the order they appear).

\subsection{Algorithm Operating Regimes}
\label{s:algOper}
 The regret as given in
\eqref{e:regA}
is a piece-wise smooth function of the system parameters $N,T,\qu,\qi$, with each of the five pieces corresponding to a region of the parameter space:
$$
S_i = \{(N,T,\qu,\qi)\in \mathbb{N}^4: \reg(N,T,\qu,\qi)= \rg i(N,T,\qu,\qi)\}\,.
$$
The optimizer $\IRs$ of~\eqref{e:totReg3} as shown in  Appendix~\ref{s:OptAppCost} is also a piecewise smooth function of the system parameters $N,T,\qu,\qi$ (corresponding to the same partition of the parameter space) and hence so too are the optimal values $\IUs$ and $\IEs$ obtained from the preceding Equations~\eqref{e:IEstar} and~\eqref{e:IUstar}. 

Recall that the model for user preferences consists of latent 
types for the users, latent types for the items, and an unknown preference matrix specifying for each user type and item type whether a user of that type likes an item of that type. 
The algorithm learns various parts of the model (and to varying degrees) depending on the values of the algorithm parameters $\IUs, \IEs,$ and $\IRs$. 
We now  explain how the algorithm behavior corresponds to each choice of algorithm parameter (deferring the derivations leading to these choices to  Appendix~\ref{s:OptAppCost} and Equations~\eqref{e:IEstar} and~\eqref{e:IUstar}). 

\paragraph*{User-user regime}
When $(N,T,\qu,\qi)\in S_1$, i.e., regret  in~\eqref{e:regA} is given by $\rg1$, 
$$\IRs = 2T,\quad \IUs=\tuu,   \quad \text{and } \quad \IEs  =0\,.$$ 
The users are clustered  based on the $\IUs=\tuu$ recommendations made to each for that purpose, i.e., the algorithm learns the structure in the user space.
After clustering the users, the algorithm learns the preference of each user cluster for the items in $\IsetR$ by recommending each to one user per cluster. 
If a user likes an item, then the item is recommended to the rest of the user's cluster in the exploit phase. Note that $\IsetE$ is empty in this regime, meaning that the clustering of items is trivial with the representatives $\IsetR$ each in its own cluster. 

\paragraph*{Hybrid partially-learned regime}
When $(N,T,\qu,\qi)\in S_2$, i.e.,
regret  in~\eqref{e:regA} is given by $\rg2$, it turns out that
$$\IRs = 2\sqrt{{T\qi\tiu}/{\qu}},\quad \IUs=\tuu,   \quad \text{and } \quad \IEs  
= 2\sqrt{T\qi\qu/\tiu}-2\qi\,.$$ 
After clustering the users, the algorithm learns the preference of each user cluster for the items in $\IsetR$ by recommending each to one user per cluster. 
The algorithm then learns structure in the item space,
clustering items in $\IsetE$ to the representatives $\IsetR$. 
\emph{Only a portion of the item space is learned}, 
because $\k(\IRs)$, the number of learned item types, is smaller than $\qi$:
when  regret  in~\eqref{e:regA} is given by $\rg2$, $T< \qi \qu/\tiu$, which amounts to $\IRs =2\sqrt{{T\qi\tiu}/{\qu}}< 2\qi$. Plugging into \eqref{e:k} gives $\k(\IRs)< \qi$.

\paragraph*{Hybrid fully-learned regime}
When $(N,T,\qu,\qi)\in S_3$, i.e., regret  in~\eqref{e:regA} is given by $\rg3$,
$$\IRs = 2\qi,\quad \IUs=\tuu,   \quad \text{and } \quad \IEs  =2T - 2\qi\,.$$
The algorithm again clusters the users but now learns the preferences for \textit{all} item types. The fact that all of the item space is learned is due to $\k(\IRs) = \qi$, i.e., with high probability every item type is represented in $\IsetR$. 

\paragraph*{Item-item partially-learned regime}
When $(N,T,\qu,\qi)\in S_4$, i.e.,
regret  in~\eqref{e:regA} is given by $\rg4$,
$$\IRs = 2\sqrt{{\qi\tiu T}/{N}},\quad \IUs=0,   \quad \text{and } \quad \IEs  =2\sqrt{\qi TN/\tiu} \,.$$ 
Because $\IUs=0$, no user clustering occurs and structure in the user space is ignored. 
A portion of the item space is learned as in the regime $S_1$ above, except that now this learning occurs without the benefit of collaboration between the users. 

\paragraph*{Item-item fully-learned regime}
When $(N,T,\qu,\qi)\in S_5$, i.e.,
regret  in~\eqref{e:regA} is given by $\rg5$,
$$\IRs = 2\qi,\quad \IUs=0,   \quad \text{and } \quad \IEs  =2T \,.$$ 
The users are not clustered and all of the item space is learned by recommending the items in $\IsetR$ to all users. 


\section{Algorithm Pseudocode}
\label{sec:pseudocode}

\subsection{Notation Used in the Algorithm}
For an item $i$ and time $t>0$,  define
\[\rated_t(i) = \{u\in [N]: a_{u,s}= i \text{ for some } s<t\}\] 
is the set of users that have rated  item $i$ before time $t$. We use the notation $\rated(i)$ in the algorithm for $\rated_t(i)$ at the time it is used. 
The partitioning (clustering) over the users is denoted by $\Pc$. 
The notation $L_{\Pc,\IsetR}$ denotes the preference of  user clusters for items in $\IsetR$.
The sets $\S_{\ti}$ are the clusters of items in $\IsetE$ 
which are declared to be of the same type as $\ti\in\IsetR$, and also include item $\ti$.
Note that several clusters of items $\S_{\ti}$'s, can  overlap, i.e., a given item in $\IsetE$ can belong to multiple clusters if there are two or more representative items of the same type. 
For each user $u$ there is a set of exploitable items $\R_u\,$ formed in Lines~\ref{L:defRu} of \textsc{Explore}, which consists of the items in the groups $\S_\ti$ whose representative $\ti$ were liked by $u$ or a user from the same user cluster as $u$.


\subsection{Algorithm Pseudocode}
\label{sec:pseudocodedetail}
We now give the pseudocode for $\recsys$ and its subroutines.

\begin{algorithm}[H] 
	\caption{$\recsys$}
	\label{alg:joint-fixed}
\KwIn{$(N, T, \qi, \qu)$}
   \BlankLine
Set global algorithm parameters $\tuu, \tiu, \IU, \IR$, and $ \IE$ as in Section~\ref{sss:params}.

\tcp{Form sets of items to be used for exploration}
 $\IsetU \gets \IU$  new random items \tcp{items used to cluster users}
 
  $\IsetR \gets \IR$  new random items \tcp{item cluster representatives}
  
 $\IsetE \gets \IE$  new random items \tcp{items to be clustered}

\tcp{Explore in order to form sets of exploitable items}
 $\{\R_u\}_{u\in [N]}\gets \textsc{Explore}(\IsetU,\IsetR, \IsetE)$

\tcp{Recommend exploitable items}
 $\textsc{Exploit}(\{\R_u\}_{u\in[N]})$
\end{algorithm}

\begin{algorithm}[H]
	\caption{$\textsc{Explore}$}\label{alg:joint-itemExplore}

\KwIn{$(\IsetU,\IsetR, \IsetE)$}
\KwOut{A set $\{\R_u\}_{u\in [N]}$ to be recommended to each user}

		\label{line:self-item-explore}
\tcp{Cluster the users}
 $\Pc \gets \textsc{UserClustering}(\IsetU)$
 
\tcp{Learn preference of user clusters $\Pc_\tu$ for cluster centers in $\IsetR$}
 $L_{\Pc,\IsetR}\gets \textsc{FindPrefs}(\IsetR,\Pc)$	
 
\tcp{Cluster the items in $\IsetE$ to cluster centers $\IsetR$}	
 $\{\S_{\ti}\}_{\ti \in \IsetR} \gets\textsc{ItemClustering}(\IsetR,\IsetE,\Pc, L_{\Pc,\IsetR})$
	
\tcp{Form sets of good items for each user}			
\For{$\tu =1,\dots, |\Pc| $}{
		$\R_u \gets
		 \bigcup_{j\in\IsetR:L_{\Pc_w,\ti}=+1} \S_j$ for each $u$ 
		such that $u\in \Pc_\tu$ \label{L:defRu}}
		\Return $\{\R_u\}_{u\in [N]}$
\end{algorithm}


\begin{algorithm}[H]
	\caption{$\textsc{UserClustering}(\IsetU)$}\label{alg:UserClustering}

\KwIn{Items $\IsetU$ to be used for user clustering}
\KwOut{A partition $\Pc$ of the users}

 Recommend all items in $\IsetU$ to all users \label{Line:RecommM2}

\If{$\IU\geq\tuu$ }{\label{Line:largeellbegin}

 Let $\{\Pc_{\tu}\}_{\tu}$ be the partition of the users 
 into fewest possible groups (clusters) such that users in each group agree on all items in $\IsetU$ \label{line:user-partiotion}}
\Else{
$\Pc_{\tu}\gets \{\tu\}$ for $\tu\in[N]$ \tcp{trivial user clustering}
}

\Return $\Pc=\{\Pc_{\tu}\}_{\tu}$
\end{algorithm}

\begin{algorithm}[H]
	\caption{$\textsc{FindPrefs}(\IsetR,\Pc)$}\label{alg:FindPref}

\KwIn{Item representatives $\IsetR$ and user clusters $\Pc$}
\KwOut{Preferences $L_{\Pc,\IsetR}$ of user clusters for item representatives}

\tcp{Determine rating $L_{\Pc_\tu,\ti}$ of each user cluster $\Pc_w$ for each item cluster representative $\ti\in \IsetR$}
	\SetKwFor{DoParallel}{do in parallel for}{}{end}
	\DoParallel{$\tu=1,\dots,|\Pc|$}{
	\For{\emph{$\ti\in \IsetR$}} 
		{
 $a_{u,t}\gets \ti$
for any available user $u$ in $\Pc_\tu$ and 
let $L_{\Pc_\tu,\ti} = L_{u,\ti}$  \tcp{Rmk.~\ref{r:FindPrefs}} 
}
Until all user clusters are finished exploring $\IsetR$, recommend random items to users in $\Pc_\tu$
}
\Return $L_{\Pc,\IsetR}$
\end{algorithm}


\begin{algorithm}[H]
	\caption{$\textsc{ItemClustering}(\IsetR,\IsetE,\Pc,L_{\Pc,\IsetR})$}\label{alg:ItemClustering}	
	\KwIn{$(\IsetR,\IsetE, \Pc, L_{\Pc,\IsetR})$}
\KwOut{$\{\S_{\ti}\}_{\ti}$}
\tcp{Preliminary exploration of exploitable items in $\IsetE$ }
Recommend each item $i\in\IsetE$ to $\tiu$ random users 
\tcp{see Rmk.~\ref{r:usersToItems}}
\label{Line:exploreM1}
\tcp{Record prefs $(L_{\Pc_\tu,i})$ of user clusters for items $i\in\IsetE$}

For each obtained $L_{u,i}$, record pref of $u$'s cluster $\Pc_w$ as $L_{\Pc_\tu,i} = L_{u,i}$ \tcp{Rmk.~\ref{r:clustPref}} \label{Line:clustPref}

\tcp{Cluster $\IsetE$ around centers $\IsetR$}
	\For{\emph{$\ti\in \IsetR$}} 
		{
		 \label{line:joint-item-explore} 
				
		$\S_{\ti}\gets \{ \ti \}\cup \{i\in \IsetE: 
		L_{\Pc_\tu,i} = L_{\Pc_\tu,\ti}$ for all $\tu$ such that $\rated(i)\cap\Pc_\tu\neq \varnothing$\}
		\label{Line:defSj2}
}
		\Return $\{\S_{\ti}\}_{\ti}$
\end{algorithm}

\begin{algorithm}[H]
	\caption{$\textsc{Exploit}$}\label{alg:jointExploit}

	\KwIn{$\{\R_u\}_{u\in [N]}$}

\For{remaining $ t \leq T $}{
\For{$u\in [N]$}{
\If{there is an item $i \in\R_u$ such that $u\notin \rated(i)$,} {

 $a_{u,t}\gets i$ }
\Else { { } $a_{u,t} \gets$ a random item not yet rated by $u$ ~\label{Line:ExploitRand}
}
}
}
\end{algorithm}

\begin{remark}[Find preference of user clusters for item representatives] \label{r:FindPrefs}
At each time $t$, for each user $u$, assign an item $j\in\IsetR$ such that no other user in the same partition as $u$ has been recommended $j$ before time $t$, or already assigned to $j$ at time $t$. If no such item in $\IsetR$ exists, assign a random new item to $u$.
\end{remark}

\begin{remark}[Assignment of users to items in Line~\ref{Line:exploreM1} of {\sc ItemClustering}] \label{r:usersToItems}
 Recommending each item in $\IsetE$ to $\tiu$ random users is done over $\lceil{\IE\tiu}/{N}\rceil$ time-steps, with additional $N \lceil{\IE\tiu}/{N}\rceil -\IE\tiu$ recommendations being random new items. 
Concretely, randomly permute the users, form a list by repeating the same permutation $\lceil{\IE \tiu}/{N}\rceil$ times, and then assign each item in $\IsetE$ to successive contiguous blocks of $\tiu$ users from the list.
\end{remark}

\begin{remark}[User cluster preferences in Line~\ref{Line:clustPref} of {\sc ItemClustering}]
 \label{r:clustPref}
In case there is an error in the user clustering and there are contradictory preferences from two different users from the same cluster for a given item, then the recorded cluster preference $L_{\Pc_\tu,i}$ can be arbitrary. Note that the preferences of some user clusters for some items remains unknown.
\end{remark}

\subsection{Choice of Algorithm Parameters}\label{sss:params}
\label{s:ChoiceParam}

We now specify the values of parameters  used in the algorithm pseudocode. Let 
\begin{align}
\label{eq:defthrjuu}
\tuu &=
\lceil2\log (N\qu^2)\rceil,
\quad\text{ and } \quad
\tiu = 
\big\lceil 2\log(N\qi)\big\rceil\,.
\end{align}
It is convenient to separate the algorithm operation into two broad operating regimes depending on whether or not structure in the item space is used.

\subsubsection{ System parameters $I^{\mathsf{NoItemClust}}$} 
\label{sec:UUparam} 

Let functions $\Rtt_U(T)$ and $\Rtt_I(T)$ be defined in~\eqref{eq:RegUSer} and~\eqref{eq:RegItem}.
In the regime of parameters in which $\Rtt_U(T)\leq \Rtt_I(T)$ (in which item clustering does not occur)
the algorithm $\recsys$ uses the following values for $\IE, \IU$ and $\IR$ referred to collectively as $ I^{\mathsf{NoItemClust}}$:

\begin{align}
\label{eq:ParamChoice1}
\IE  = 0,\quad \IU =\tuu, \quad
 \IR=6T\,.
\end{align}

\subsubsection{System parameters $I^{\mathsf{ItemClust}}$}
\label{sec:IJparam}
Let functions $\Rtt_U(T)$ and $\Rtt_I(T)$ be defined in~\eqref{eq:RegUSer} and~\eqref{eq:RegItem}.
In the regime of parameters in which $\Rtt_U(T)> \Rtt_I(T)$ (in which item clustering occurs)
the algorithm $\recsys$ uses the following values for $\IE, \IU$ and $\IR$ referred to collectively as $ I^{\mathsf{ItemClust}}$:

\begin{align} 
\IE & = \left\lceil 16\frac{\qi}{\ell}T\right\rceil
,\,
\IU =
 \tuu \ind{ \ell>\tuu } ,\,
 \IR
=
\begin{cases}
\lceil 3\ell \rceil, 
&\text{if } \ell\leq \frac{\qi}{3}
\\
\lceil \qi \log(N\qi)\rceil,
& \text{if } \ell >\frac{\qi}{3}
 \end{cases}\,.
\label{eq:ParamChoice2}
\end{align}
where the parameter $\ell$ is defined as
\begin{align}
\label{eq:ParamChoice2ell}
\ell&=\begin{cases} 
\kI
\quad  &\text{ if }\,\,
\kI
\leq \tuu
\,,
\\ 
\kH, 
& \text{ if }
\kI
>\tuu
\quad \text{and }
\kH \leq \frac{\qi}{3}
\\
\kA, 
& \text{ if } 
\kI>\tuu
\quad \text{and }
\kH >\frac{\qi}{3}
\,,
\end{cases}
\end{align}
where
\begin{align}
\label{eq:defKrep}
	\kI &= 16\log T+ 2\sqrt{\frac{\qi\tiu}{N}T}, \quad
	\kH = 8\tuu + 2\sqrt{\frac{\tiu \qi T}{\qu}}\,, 
    \text{ and}\quad
	\kA = \qi\,.
\end{align}

The above discontiuities in the definition of parameter $\ell$ as a function of time horizon $T$ can be interpreted as follows: 
The parameter $\ell$ is a nondecreasing function of $T$ where for three different regimes $\ell\leq \tuu$, $\tuu< \ell \leq \qi/3$ and $\ell>\qi/3$ the parameter $\ell$ has different dependencies on $T$.

\subsection{Upper Bound on Regret for $\recsys$}
The following theorem gives the upper bound on the regret of the proposed algorithm \recsys. 

\begin{theorem}\label{th:Item-upper}
Define the parameters $\tuu$  and $\tiu$ as in Equation~\eqref{eq:defthrjuu}. Then the following bounds for regret of the $\recsys$ algorithm holds:
\[\reg(T)
\leq
 \min\Big\{ \frac{T}{2}, \,\,C \Rtt_U(T), C\Rtt_I(T)\Big\}\]
for a constant $C$ where the functions $\Rtt_U(T)$ and  $\Rtt_I(T)$ are defined as follows:

\begin{align}
\label{eq:RegUSer}
\Rtt_U(T) =
\tuu + \frac{\qu}{N} T
\end{align}
and
\begin{align}
\label{eq:RegItem}
\Rtt_I(T) &=
\begin{cases}
\vspace{.1in}
 \log T + \sqrt{\frac{\qi \tiu}{N}T},
& \text{ if }
T< \TupI
\\ 
\vspace{.1in}
 \tuu +   \frac{1}{N} \sqrt{\qi\qu \tiu T},&
\text{ if }
\TupI \leq T <\TupU
\\  
\tuu  + \frac{\qu}{N} \qi \log(N\qi) + \frac{\tiu}{N} T,&
\text{ if }
\TupU\leq T
\,.
\end{cases}
\end{align}
where
\begin{align*}
  \TupI &:= \max\{T: \kI  \leq \tuu\}  \qquad \text{and}\qquad
  \TupU := \max\{T: \kH  \leq \qi/3\}\,
\end{align*}
for $ \kI$ and $\kH$ defined in~\eqref{eq:defKrep}.
\end{theorem}

\bibliographystyle{acm} 

\bibliography{RSbib.bib}

 \newpage
 \appendix

 \section{Deferred Proofs from Section~\ref{sec:BadRecommScenario}}
\label{sec:proofLBApp}

In this section, we prove the lemmas in Section~\ref{sec:BadRecommScenario} used to identify the bad recommendations and their associated scenarios in Proposition~\ref{p:badUncertain} for $b=1, 2, 3 $ and $4$ which characterizes the bad recommendations. Next, we tie these scenarios to the definition of regret (Cor.~\ref{l:jl-reg-lower}) proved in Section~\ref{sec:ReginTermsofBadProof}.

\subsection{Notation}
\label{s:ProofUncertain}
In this section it is convenient to interpret the item type function $\tI:\mathbb{N}\to [\qi]$ assigning types to the items as a sequence and user type function assigning types to the users $\tU:[N]\to [\qu]$ as a vector. 
Let $\tIni:=\{\tau_I(i'): i'\in\mathbb{N}, i'\neq i \}$ denote the sequence of item types for all items except $i$, with $\tUnu$ defined analogously.

We use 
$\Xi^{\sim(\tu,\ti)}=\{\xi_{\tu',\ti'}:(\tu',\ti')\in[\qu]\times[\qi], (\tu',\ti')\neq (\tu,\ti)\}$ to denote the set of all elements in matrix $\Xi$ except $\xi_{\tu,\ti}$.
Recall that $\H_{t-1}$ denotes the feedback obtained through time $t-1$.  


When user $u$ has rated an item  with type $\ti$, for the sake of brevity, we say `user $u$ has rated item type $\ti$'. 
Similarly, when item $i$ has been rated by a user with user type $\tu$, we say `item $i$ has been rated by user type $\tu$'.

 \subsection{Missing Lemma Proofs from Section~\ref{sec:BadproofMainBody}}
\label{sec:A3propbound}


The proof of Lemma~\ref{l:posteriortypes} rests on a simple probabilistic statement about the posterior over histories. We start with the latter.

\subsubsection{Probabilities of System Trajectories}
\begin{lemma}
\label{l:PHist2}
Given $\H_t=h_t$ in which user $u$ has not rated item $i$,
let $\bj_u\in \sigma(\H_t,\tIni)$ be the set of item types previously rated by user $u$ and row vector $\bx_u\in \sigma(\H_t,\tIni)$ be the $\pm 1$ feedback. 
Similarly, Let $\bj_i\in \sigma(\H_t,\tUnu)$ be the set of user types that previously rated by item $i$ and row vector $\bx_i\in \sigma(\H_t,\tUnu)$ be the $\pm 1$ feedback. 
Then, 
\begin{align*}
    \Pr\big(\H_{t}= h_{t},
\,\big|\,& \tau_U(u)=v, \tau_I(i)=j, \,\Xi, \, \tIni, \tUnu \big)\\
&=c\cdot \ind{v\in \Lambda_{\bx_u}\b(\Xi_{\bullet \bj_u}\b) (h_{t},\Xi,\tIni),  j\in \Lambda_{\bx_i}\b(\Xi^T_{\bullet \bj_i}\b)(h_{t},\Xi,\tUnu)}
\end{align*}
for some $c$ that does not depend on $v$ or $j$.
\end{lemma}

\begin{proof}
    The history $\H_{t}$ includes all the items recommended and the associated $\pm 1$ feedback up to and including time $t$; its distribution is complicated, but we do not need to calculate the posterior probability of $\H_{t}$ in the lemma statement.

    Fix $\Xi, \, \tIni, \tUnu$, as well as a specific possible trajectory for the history $h_{t}$ in which user $u$ has not rated item $i$. 
    We consider two copies of the history random process $\H_{t}$: the first $\H_{t}^1$, where $\tau_U(u)=v_1, \tau_I(i)=j_1$, and the second $\H_{t}^{2}$ where $\tau_U(u)=v_2, , \tau_I(i)=j_2$, where both $v_1,v_2\in \Lambda_{\bx_u}\b(\Xi_{\bullet \bj_u}\b) (h_{t},\Xi,\tIni)$ and both  $j_1, j_2\in \Lambda_{\bx_i}\b(\Xi^T_{\bullet \bj_i}\b) (h_{t},\Xi,\tUnu)$. I.e., both user types and item types are consistent with the revealed preferences of user $u$ and item $i$ in the history $h_t$. For all users $u'$ (including $u$) and times $s\leq t$, the recommended item is described by a function $a_{u',s}=f_{u',s}(\H_{s-1},\zeta_{u',s})$ for some auxiliary independent random variable $\zeta_{u',s}$. We thus have two copies of all these variables, $a_{u',s}^1$ and  $a_{u',s}^2$, and so forth. This shows that all consistent histories have the same probabilities, proving the lemma.
    
    We will give an inductive argument. Suppose that one of these two cases holds for some $s,0\leq s \leq t-1$:
    \begin{enumerate}
        \item $\H_{s}^1 = \H_{s}^2=h_{s}$, or
        \item Both $\H_{s}^1 \neq  h_{s}$ and  $\H_{s}^2\neq h_{s}$.
    \end{enumerate} 
    We will show below that at time $s+1$ one of these two cases continues to hold.  
Note that the first case is tautologically true for $s=0$, since then the histories are empty, which serves as our base case. By induction, one of the two cases holds at time $t$, which implies that for any $h_t$ and any $ v_1, v_2 \in \Lambda_{\bx_u}\b(\Xi_{\bullet \bj_u}\b) (h_{t},\Xi,\tIni)$ and  $j_1, j_2\in \Lambda_{\bx_i}\b(\Xi^T_{\bullet \bj_i}\b) (h_{t},\Xi,\tUnu)$, 

\begin{align*}
    \Pr\big(\H_{t}= h_{t}
\,\big|\,&\tau_U(u)=v_1,  \tau_I(i)=j_1, \,\Xi, \, \tIni, \tUnu \big)\\
&= \Pr\big(\H_{t}= h_{t} 
\,\big|\,\tau_U(u)=v_2,  \tau_I(i)=j_2, \,\Xi, \, \tIni, \tUnu \big).
\end{align*}
    We now prove the inductive step. 
    First, observe that if case 2 holds above, i.e., the histories both differ from $h_s$, then case 2 continues to hold  for all future times. 

    Now suppose that case 1 holds at time $s$.  
    We couple the $\zeta_{u',s+1}^1$ and $\zeta_{u',s+1}^2$ random variables to be equal for all $u'$, which results in $a_{u',s+1}^1 = a_{u',s+1}^2$, i.e., all recommendations made at step $s+1$ in both processes are the same. These both agree with $h_{s+1}$, or both disagree with $h_{s+1}$, and the latter case puts us into case 2 at time $s+1$.

    In the former case where $a_{u',s+1}^1 = a_{u',s+1}^2$ agree with $h_{s+1}$, the feedback $L_{u',a_{u',s+1}}^1$ and $L_{u',a_{u',s+1}}^2$ from these recommendations is the same in both $\H_{s}^1$ and $\H_{s}^2$ copies: 
    \begin{enumerate}
        \item  For users $u'\neq u$  such that $a_{u',s+1}\neq i$ the types are the same in both copies and hence preferences are the same, as both $\Xi$ and the types of these users and the recommended items are fixed.
          \item  For users $u' \neq u$ such that $a_{u',s+1} = i$, the user types are the same in both copies and item $i$'s type $j_1,j_2\in \Lambda_{\bx_i}\b(\Xi^T_{\bullet \bj_i}\b) $ by definition of $\Lambda$ guarantees that the actual preferences of these users are the same for $i$ in $h_{s+1}$.
        \item  For user $u$, its type is different in the two copies, however, user $u$'s type $v_1,v_2\in \Lambda_{\bx_u}\b(\Xi_{\bullet \bj_u}\b)$ by definition of $\Lambda$ guarantees that the actual preferences are the same for items $a_{u,s+1}$ recommended in $h_{s+1}$.
    \end{enumerate}
  This completes the induction argument. 
\end{proof}

\subsubsection{Proof of Lemma~\ref{l:posteriortypes}: Uniform Posterior Over Types}

\begin{lemman}[\ref{l:posteriortypes}]
 Given $\H_{t-1}=h$ in which user $u$ has not rated item $i$,
let $\bj_u\in \sigma(\H_t,\tIni)$ be the set of item types previously rated by user $u$ and row vector $\bx_u\in \sigma(\H_{t-1},\tIni)$ be the $\pm 1$ feedback. 
Similarly, Let $\bj_i\in \sigma(\H_{t-1},\tUnu)$ be the set of user types that previously rated item $i$ and row vector $\bx_i\in \sigma(\H_{t-1},\tUnu)$ be the $\pm 1$ feedback.  Then, 
     \begin{align*}
\Pr&\big(\tau_U(u)=v, \tau_I(i)=j \,\big|\, \H_{t-1}=h, \Xi, \, \tIni, \tUnu\big)
=
\frac{\ind{v\in \Lambda_{\bx_u}(\Xi_{\bullet \bj_u})}}{|\Lambda_{\bx_u}(\Xi_{\bullet \bj_u})|}
\,\, \frac{\ind{j\in \Lambda_{\bx_i}(\Xi^T_{\bullet \bj_i})}}{|\Lambda_{\bx_i}(\Xi^T_{\bullet \bj_i}))|}\,.
\end{align*}
\end{lemman}
\begin{proof}
Bayes' rule, the fact that $\tau_U(u),\tau_I(i)$ and  $\Xi, \, \tIni, \tUnu $ are independent (and $\tau_U(u)\sim \mathsf{unif}([\qu])$, $\tau_I(i)\sim \mathsf{unif}([\qi])$), and then Lemma~\ref{l:PHist2} yields 
 \begin{align*}
\Pr&\big(\tau_U(u)=v, \tau_I(i)=j \,\big|\, \H_t=h, \Xi, \, \tIni, \tUnu\big)
\\&=
\frac{ \Pr\big(\H_t=h \big|\,\tau_U(u)=v, \tau_I(i)=j, \,\Xi, \, \tIni, \tUnu \big)\,\, \Pr\big(\tau_U(u)=v, \tau_I(i)=j \big)}
{\sum_{v'\in [\qu], j'\in [\qi]}\Pr\big(\H_t=h \big|\,\tau_U(u)=v', \tau_I(i)=j', \,\Xi, \, \tIni, \tUnu \big) \Pr\big(\tau_U(u)=v', \tau_I(i)=j' \big)}
\\&= \frac{c\cdot \ind{v\in \Lambda_{\bx_u}(\Xi_{\bullet \bj_u}),  j\in \Lambda_{\bx_i}(\Xi^T_{\bullet \bj_i}) }}
{\sum_{v'\in [\qu], j'\in [\qi]} c\cdot \ind{v'\in \Lambda_{\bx_u}(\Xi_{\bullet \bj_u}),  j'\in \Lambda_{\bx_i}(\Xi^T_{\bullet \bj_i})}}
\\&=
\frac{\ind{v\in \Lambda_{\bx_u}(\Xi_{\bullet \bj_u}) }}{|\Lambda_{\bx_u}(\Xi_{\bullet \bj_u})|}
\,\, \frac{\ind{j\in \Lambda_{\bx_i}(\Xi^T_{\bullet \bj_i}) }}{|\Lambda_{\bx_i}(\Xi^T_{\bullet \bj_i})|}\,,
\end{align*}   
where in the last line we used $\Lambda_{\bx_i}(\Xi^T_{\bullet \bj_i}) \in \sigma (h, \Xi, \tUnu)$ and $\Lambda_{\bx_u}(\Xi_{\bullet \bj_u}) \in \sigma (h, \Xi, \tIni)$ and hence are independent of $j$ and $v$.

\end{proof}

\subsubsection{Proof of Lemma~\ref{l:A3propbound}: Probability of Liked Item In Terms of Regularity}
\label{sec:A3propboundproof}
\begin{lemman}[\ref{l:A3propbound}]
    Let $\bj_u\in \sigma(\H_{t-1},\tau_I)$ and  $\bx_u\in \sigma(\H_{t-1},\tau_I)$ be as defined in Lemma~\ref{l:posteriortypes}. 
Let $\bx^+$ be the vector $\bx$ appended by $+1$. Then 
\begin{align*}
\Pr\big(L_{u,i}=+1 \,\big|\, \H_{t-1}, \,\Xi, \, \tI, \tUnu, (\BBjil{u}{i}{t})^c \big)
&=
    \frac{|\Lambda_{\bx_u^+}(\Xi_{\bullet ,(\bj_u,\tau_I(i))})|}
    {|\Lambda_{\bx_u}(\Xi_{\bullet \bj_u})|}\,.
\end{align*}
\end{lemman}
\begin{proof}
Let $\bj=\bj_u$ and $\bx=\bx_u$. 
Using $\BBjil{u}{i}{t}\in \sigma(\H_{t-1},\tI)$ and Lemma~\ref{l:posteriortypes}
 \begin{align*}
\Pr&\big(\tau_U(u)=v \,\big|\, \H_{t-1}=h, (\BBjil{u}{i}{t})^c \,,\Xi, \, \tI, \tUnu\big)
=
\Pr\big(\tau_U(u)=v \,\big|\, \H_{t-1}=h, ,\Xi, \, \tI, \tUnu\big)
\\&=
\frac{\ind{v\in \Lambda_{\bx}(\Xi_{\bullet \bj})}}{|\Lambda_{\bx}(\Xi_{\bullet \bj})|}\,.
\end{align*}

Given the items upon which we are conditioning in the lemma statement the only uncertainty is in the type of user $u$, and moreover
 $L_{u,i}=+1$ if and only if $\tau_U(u)\in \Lambda_{+}(\Xi_{\bullet \tau_I(i)})=\{v\in [\qu]: \xi_{v,\tau_I(i)}=+1\}$. 
 Hence
\begin{align*}
\Pr\big(L_{u,i}=+1 \,\big|\, \H_{t-1}, \,\Xi, \, \tI, \tUnu, (\BBjil{u}{i}{t})^c \big)
&= \frac{|\Lambda_{+}(\Xi_{\bullet \tau_I(i)}) \cap \Lambda_{\bx}(\Xi_{\bullet \bj})|}{|\Lambda_{\bx}(\Xi_{\bullet \bj})|}
\,.
\end{align*}
Using Definition~\ref{def:jl-B-u-tauI}, on the event $\big(\BBjil{u}{i}{t}\big)^c$,
$\tau_I(i)$ is not among the item types rated by $u$, i.e., 
 $\tau_I(i)\notin \bj$. It follows that, $|\Lambda_{+}(\Xi_{\bullet \tau_I(i)})  \cap \Lambda_{\bx}(\Xi_{\bullet \bj})| = |\Lambda_{\bx^+}\b(\Xi_{\bullet (\bj,\tau_I(i))}\b)|$.
\end{proof}



\subsection{Proof of Prop~\ref{p:badUncertain} for $b=1, 2, 4$}
\label{sec:ProofBadAppendix}

\subsubsection{Proof of Proposition~\ref{p:badUncertain} with $b=1$}

\begin{propn}[\ref{p:badUncertain} with $b=1$]
Denote $\Omega\in \sigma(\Xi)$ the event that the preference matrix is regular defined in Corollary~\ref{cor:regularityprob}. Then 
\[\Pr(L_{u,a_{u,t}}=-1\, |\, c_{a_{u,t}}^t < \til, d_{u}^t < \tul,  \Omega) \geq 1/3\,.\]
\end{propn}

\paragraph*{Probability of Liking Item Type in Terms of Regularity}
We use the same notation introduced in Section~\ref{s:ProofUncertain}.
The next lemma expresses the posterior probability  of user $u$'s preference for item $i$, conditional on the user feedback $\H_{t-1}$ obtained so far \emph{as well as the preference matrix $\Xi$, all item types $\tIni$ except for item $i$'s, and all user types $\tUnu$ except for $u's$}, in terms of a quantity that is controlled by the column and row regularity. We emphasize that conditioning on $\Xi, \tIni, \tUnu$ yields a simple expression for the posterior in terms of these quantities, but 
does \emph{not} mean giving the information to the algorithm producing recommendation $a_{u,t}$.

\begin{lemma}
\label{l:A1propbound}
Consider a realization of $\H_{t-1}=h$ in which use $u$ has not rated item $i$.
Let $\bj_u\in \sigma(\H_{t-1},\tIni)$ be the set of item types previously rated by user $u$ and row vector $x_u\in \sigma(\H_{t-1},\tIni)$ be the $\pm 1$ feedback. Similarly, Let $\bj_i\in \sigma(\H_{t-1},\tUnu)$ be the set of user types that previously rated item $i$ and row vector $\bx_i\in \sigma(\H_{t-1},\tUnu)$ be the $\pm 1$ feedback.
Let $\bx^+$ be the vector $\bx$ appended by $+1$. Then
\begin{align*}
\Pr\big(L_{u,i}=+1 \,\big|\, \H_{t-1}, \,\Xi, \, \tIni, \tUnu \big)
&\leq
\frac{|\bj_u | }{|\Lambda_{\bx_i}(\Xi_{\bullet \bj_i}^T)|}
+
\max_{j\in [\qu]\setminus \bj_u} 
\frac{ |\Lambda_{\bx_u^+}(\Xi_{\bullet (\bj_u,j)})| }{|\Lambda_{\bx_u}(\Xi_{\bullet \bj_u})|}\,.
\end{align*}
\end{lemma}
\begin{proof}
Lemma~\ref{l:posteriortypes} states that conditioning on a realization of $\big(\H_{t-1}, \,\Xi, \, \tIni, \tUnu \big)$, 
the variables $\tau_I(i)$ and $\tau_U(u)$ are independently uniformly distributed on $\Lambda_{\bx_i}(\Xi_{\bullet \bj_i}^T)$ and $\Lambda_{\bx_u}(\Xi_{\bullet \bj_u})$.
Note that user $u$ likes item $i$ only if $\tau_U(u)\in \Lambda_{+,\tau_I(i)}(\Xi)$.
    Hence,
\begin{align}
\label{eq:targetLBterm1}
&\Pr
\b(
L_{u,i}=+1 
\,\big|\, 
\H_{t-1}, \,\Xi, \, \tIni, \tUnu
\b) 
\\&= \sum_{j,v}
\Big\{\Pr
\b(L_{u,i}=+1 
\,\big|\,  \tau_U(u)=j,\tau_I(i)=k,
\H_{t-1}, \,\Xi, \, \tIni, \tUnu \b)\cdot\\
&\qquad\qquad\qquad\qquad\qquad
\Pr\b(\tau_U(u)=j,\tau_I(i)=k
\,\big|\, 
\H_{t-1}, \,\Xi, \, \tIni, \tUnu
\b)\Big\}
\\&\leq \sum_{j,v}
\ident\{v \in \Lambda_{+}(\Xi_{\bullet j}) \}
\frac{\ident\{ v\in \Lambda_{\bx_u}(\Xi_{\bullet \bj_u})\}}{|\Lambda_{\bx_u}(\Xi_{\bullet \bj_u})|}\,\,
\frac{\ident\{j\in \Lambda_{\bx_i}(\Xi_{\bullet \bj_i}^T)\}}{|\Lambda_{\bx_i}(\Xi_{\bullet \bj_i}^T)|}\,\,
\end{align}

To upper bound this, in terms of regularity parameters of $\Xi$, using simple algebra, we get
\begin{align*}
\sum_{j,v}&\ident\{j\in \Lambda_{\bx_i}(\Xi_{\bullet \bj_i}^T), 
v\in \Lambda_{\bx_u}(\Xi_{\bullet \bj_u})\cap \Lambda_{+}(\Xi_{\bullet j})\}
\\&\stackrel{(a)}\leq 
\sum_{j,v}\ident\{j\in \Lambda_{\bx_i}(\Xi_{\bullet \bj_i}^T)\cap \bj_u,\,\, v\in \Lambda_{\bx_u}(\Xi_{\bullet \bj_u})\}
\\&\,+ 
\sum_{j,v}\ident\{j\in \Lambda_{\bx_i}(\Xi_{\bullet \bj_i}^T)\setminus \bj_u, \,\,v\in \Lambda_{\bx_u^+}(\Xi_{\bullet (\bj_u,j)})\}
\\&\leq 
\sum_{j,v}\ident\{j\in \bj_u, v\in \Lambda_{\bx_u}(\Xi_{\bullet \bj_u})\} + 
|\Lambda_{\bx_i}(\Xi_{\bullet \bj_i}^T) \setminus \bj_u|\,\,
\max_{j\in \Lambda_{\bx_i}(\Xi_{\bullet \bj_i}^T) \setminus \bj_u} |\Lambda_{\bx_u^+}(\Xi_{\bullet (\bj_u,j)})|
\\&\leq |\bj_u|\,\, |\Lambda_{\bx_u}(\Xi_{\bullet \bj_u})| + |\Lambda_{\bx_i}(\Xi_{\bullet \bj_i}^T) |\,\, \max_{j\notin \bj_u} |\Lambda_{\bx_u^+}(\Xi_{\bullet (\bj_u,j)})|
\end{align*} 

where (a) uses the definition of $\Lambda$
$$\Lambda_{\bx_u}(\Xi_{\bullet \bj_u})\cap \Lambda_{+}(\Xi_{\bullet j}) = \Lambda_{\bx_u^+}(\Xi_{\bullet (\bj_u,j)}) \qquad \text{for } j\notin \bj_u\,.$$

Plugging this into above display gives the statement of lemma.
\end{proof}

\subsubsection*{Proof of Prop.~\ref{p:badUncertain} with $b=1$}
We need to show the upper bound
\begin{align}
\label{eq:jl-bdA1statement}
\Pr
\big( &
L_{u,a_{u,t}}=+1 \,\big|\, c_{a_{u,t}}^t < \til, \,d_u^t < \tul,\,
 \Omega
\big)
= \Pr\big(
L_{u,a_{u,t}}=+1 \,\big|\, \histE{1}(u,a_{u,t})
\big)
\leq 2/3\,.
\end{align} 
where to shorten the notation, given user $u$ and item $i$ we defined the event $\histE{1}(u,i)$ and the set $\hist{1}(u,i)$ as follows:
\begin{align}
\histE{1}(u,i) & = \big\{c_i^t < \til, \,d_u^t < \tul,\,
 \Omega\Big\}\\
\hist{1}(u,i)&=
\big\{
 \text{realizations of }
\big(\H_{t-1}, \,\Xi, \, \tIni, \tUnu \big)
\text{ such that }
\notag
\\&\qquad\qquad
\histE{1}(u,i) \text{ holds}
\text{ and } a_{u,s}\neq i \text{ for all } s<t
\, \big\} 
\label{eq:defA1Event}
\end{align}
be the set of possible realizations of the model parameters and history up to time $t-1$  consistent with the event conditioned upon in~\eqref{eq:jl-bdA1statement}.

\paragraph*{Set $\hist 1(u,i)$ is well-defined.}
We note that the condition in \eqref{eq:defA1Event} is a function of the variables $\big(\H_{t-1}, \,\Xi, \, \tIni, \tUnu \big)$, so one can determine whether the latter satisfies the former. We spell this out as follows.
The history $\H_{t-1}$ determines whether item $i$ has been recommended to user $u$ before or not.
The value of $c_i^t$ (as in Definition~\ref{def:jl-cit}) is a function of the previous recommendations, summarized in $\H_{t-1}$, and the type of all users  except user $u$, i.e., $\tUnu$
\footnote{Conditioning on the event $a_{u,t}=i$ implies that 
user $u$ has not been recommended item $i$ by time $t-1$. 
Hence ambiguity in the type of user $u$ does not make the value of $c_i^t$ ambiguous.}. 
Similarly, the value of $d_u^t$ (as in Definition~\ref{def:jl-dut}) is a function of the previous recommendations, summarized in $\H_{t-1}$, and the type of all items  except item $i$, i.e., $\tIni$. 
Furthermore, the row regularity and column regularity of the preference matrix $\Xi$ is a function only of $\Xi$.

Let $\bj_u, \bj_i, \bx_u$ and $\bx_i$ be as defined in Lemma~\ref{l:A3propbound}.
Note that $\bj_u$ is  a deterministic function of realization of $\H_{t-1}$ and $\tIni$ and $\bj_i$ is  a deterministic function of realization of $\H_{t-1}$ and $\tUnu$.
Given any realization of $\H_{t-1}$ and $\tIni, \tUnu \in \hist{1}(u,i)$, $|\bj_u|= d_u^{t}<\tul$ and $|\bj_i|= c_i^{t}<\til$. 
 
 So for any realization of 
 $\big(\H_{t-1}, \,\Xi, \, \tIni, \tUnu) \big)
\in\hist{1}(u,i)$,
 \begin{align*}
\Pr\Big[L_{u,i}=+1 \,\big|\, \H_{t-1}, \,\Xi, \, \tI, \tUnu \Big]
& \overset{(a)}{\leq }
 \frac{|\bj_u | }{|\Lambda_{\bx_i}(\Xi_{\bullet \bj_i}^T)|}
+
\max_{j\notin \bj_u} 
\frac{ |\Lambda_{\bx_u^+}(\Xi_{\bullet (\bj_u,j)})| }{|\Lambda_{\bx_u}(\Xi_{\bullet \bj_u})|}
\\& 
\overset{(b)}{\leq }
\frac{\tul 2^{\til}}{(1-\eta)\qi} 
+
\frac{ (1+\eta) \frac{\qu}{2^{|\bj_u|+1}} }
{(1-\eta) \frac{\qu}{2^{|\bj_u|}} }
\overset{(c)}{\leq }
\frac{1}{36} + \frac{1}{2} \frac{1+\eta}{1-\eta}  
\overset{(d)}{\leq } 2/3
\end{align*}
where (a) uses Lemma~\ref{l:A1propbound}. 
(b) uses  the row and column regularity of matrix $\Xi$ ($\Xi^T \in \rgl_{\til,\eta}, \Xi \in \rgl_{\tul,\eta}, $ as in Definition~\ref{def:ul-reg}) and $|\bj_i|<\til$ and $|\bj_{u}|<\tul$. (c) uses  the choice of $\tul \leq \log \qu, \til\leq 0.99 \log \qi - 5\log\log N$, and $\eta$ in~\eqref{eq:jl-thru} and the model assumption $\qi> 100\log N, \qu<n$ in Section~\ref{ss:pref} to get
\[\frac{\tul 2^{\til}}{(1-\eta)\qi} 
< \frac{(\log N)  \qi^{0.99} }{(1-\eta)\qi (\log N)^5}  
< \frac{2}{\qi^{0.01}  (\log N)^4 } < \frac{1}{\log^2 N} = 1/36\]
(d) uses  $\eta=1/13$. 


\paragraph*{Applying tower property}
Using the total probability lemma on above display, for any $i$ such that $a_{u,s}\neq i$ for all $s<t$, 
\begin{align}
\label{eq:BoundProptowerH1}
   \Pr\big(L_{u,i}=+1 \,\big|\, \H_{t-1}, \histE{1}(u,i) \big) \leq 2/3\,.
\end{align}
Recall that there is a random variable $\zeta_{u,t}$, independent of all other variables, such that $a_{u,t}=f_{u,t}(\H_{t-1},\zeta_{u,t})$, for some deterministic function $f_{u,t}$. 
Also, for all $i$ such that $u$ has not rated $i$ before, $\histE{1}(u,i) \in \sigma(\H_{t-1}, \Xi, \tIni, \tUnu)$. This is proved using the the same justification as above showing that the set $\hist{1}(u,i)$ is well-defined. Clearly, $L_{u,i}\in \sigma(\Xi,\tU, \tI)$.
So conditioning on $\H_{t-1}$, $\{a_{u,t}=i\}$ is independent of event $\histE{1}(u,i)$ and $L_{u,i}$. Hence, 

\begin{align}
       \Pr\big(L_{u,i}=+1 \,\big|\, \H_{t-1}, \histE{1}(u,a_{u,t}), a_{u,t}=i \big)
       &=   \Pr\big(L_{u,i}=+1 \,\big|\, \H_{t-1}, \histE{1}(u,i), a_{u,t}=i \big)\notag
       \\&=
      \Pr\big(L_{u,i}=+1 \,\big|\, \H_{t-1}, \histE{1}(u,i) \big)\,.
         \label{eq:NExtItemH1}
\end{align}

We use these properties to get
  \begin{align*}
\Pr&\big(L_{u,a_{u,t}}=+1 \,\big|\,  \H_{t-1}, \histE{1}(u,a_{u,t}) \big)
\\& \overset{(a)}{=}
\sum_{i\in\mathbb{N}}
\Pr\big( a_{u,t}=i \,\big|\, \H_{t-1}, \histE{1}(u,a_{u,t}) \big)
\Pr\big(L_{u,i}=+1 \,\big|\, \H_{t-1}, \histE{1}(u,a_{u,t}), a_{u,t}=i \big)
\\
&\overset{(b)}{=} \sum_{i: a_{u,s} \neq i \text{ for all } s<t}
\Pr\big( a_{u,t}=i \,\big|\, \H_{t-1}, \histE{1}(u,a_{u,t})\big)
 \Pr\big(L_{u,i}=+1 \,\big|\, \H_{t-1}, \histE{1}(u,i) \big)
\overset{(c)}{\leq} 2/3\,.
\end{align*}
Here (a) uses total probability lemma $\{L_{u,a_{u,t}}=+1\} = \bigcup_{i\in\mathbb{N}} \{a_{u,t}=i, L_{u,i}=+1\}$ and Eq.~\eqref{eq:NExtItemH1}. (b) uses the assumption that an item can be recommended at most once to a user to take the sum only over the items never recommended to $u$ before. Using Eq.~\eqref{eq:BoundProptowerH1} in each term of the sum gives (c). 

Using the tower property over $\H_{t-1}$ on the above display gives the statement of proposition for $b=1$.
 \qed


\subsubsection{Proof of Proposition~\ref{p:badUncertain} with $b=2$}

\begin{propn}[\ref{p:badUncertain} with $b=2$]
Denote $\Omega\in \sigma(\Xi)$ the event that the preference matrix is regular defined in Corollary~\ref{cor:regularityprob}. Then,
\[\Pr(L_{u,a_{u,t}}=-1\, |\, c_{a_{u,t}}^t < \til, d_{u}^t \geq \tul, (\BBjul{u}{
a_{u,t}}{t})^c,  \Omega) \geq 1/3\,.\]
\end{propn}

 The proof is very similar to the case $b=3$ presented in Sec.~\ref{p:badUncertain}, switching the role of items and users. First, parallel to Lemma~\ref{l:A3propbound}, the next lemma expresses the posterior probability under $(\BBjul{u}{i}{t})^c$ of user $u$'s preference for item $i$ conditional on the feedback $\H_{t-1}$ obtained so far as well as the preference matrix, user types $\tU$ and all item types except for item $i$'s, in terms of a quantity that is controlled by the row regularity of $\Xi$. 
\begin{lemma}
\label{l:A2propbound}
Let $\bj_i\in \sigma(\H_{t-1},\tU)$ and  $\bx_i\in \sigma(\H_{t-1},\tU)$ be as defined in Lemma~\ref{l:posteriortypes}. 
Let $\bx^+$ be the vector $\bx$ appended by $+1$. Then 
\begin{align*}
\Pr\big(L_{u,i}=+1 \,\big|\, \H_{t-1}, \,\Xi, \, \tIni, \tU, (\BBjul{u}{i}{t})^c \big)
&=
    \frac{|\Lambda_{\bx_i^+}(\Xi^T_{\bullet ,(\bj_i,\tau_U(u))})|}
    {|\Lambda_{\bx_i}(\Xi^T_{\bullet \bj_i})|}\,.
\end{align*}
\end{lemma}
\begin{proof}
Let $\bj=\bj_i$ and $\bx=\bx_i$. 
Using $\BBjul{u}{i}{t}\in \sigma(\H_{t-1},\tI)$ and Lemma~\ref{l:posteriortypes}
 \begin{align*}
\Pr&\big(\tau_I(i)=j \,\big|\, \H_{t-1}=h, (\BBjul{u}{i}{t})^c \,,\Xi, \, \tIni, \tU\big)
=
\Pr\big(\tau_I(i)=j \,\big|\, \H_{t-1}=h, ,\Xi, \, \tIni, \tU\big)
\\&=
\frac{\ind{j\in \Lambda_{x}(\Xi^T_{\bullet \bj})}}{|\Lambda_{x}(\Xi^T_{\bullet \bj})|}\,.
\end{align*}

Given the items upon which we are conditioning in the lemma statement the only uncertainty is in the type of item $i$, and moreover
 $L_{u,i}=+1$ if and only if $\tau_I(i)\in \Lambda_{+}(\Xi^T_{\bullet \tau_U(u)})=\{j\in [\qu]: \xi_{\tau_U(u), j}=+1\}$. 
 Hence
\begin{align*}
\Pr\big(L_{u,i}=+1 \,\big|\, \H_{t-1}, \,\Xi, \, \tIni, \tU, (\BBjul{u}{i}{t})^c \big)
&= \frac{|\Lambda_{+}(\Xi^T_{\bullet \tau_U(u)}) \cap \Lambda_{x}(\Xi^T_{\bullet \bj})|}{|\Lambda_{x}(\Xi^T_{\bullet \bj})|}
\,.
\end{align*}
Using Definition~\eqref{def:jl-B-u-tauU}, on the event $\big(\BBjul{u}{i}{t}\big)^c$,
$\tau_U(u)$ is not among the user types that has rated item $i$, i.e., 
 $\tau_U(u)\notin \bj$. It follows that, $|\Lambda_{+}(\Xi^T_{\bullet \tau_U(u)})  \cap \Lambda_{\bx}(\Xi^T_{\bullet \bj})| = |\Lambda_{\bx^+}\b(\Xi^T_{\bullet (\bj,\tau_U(u))}\b)|$.
 \end{proof}


\subsubsection*{Proof of Proposition~\ref{p:badUncertain} with $b=2$}
We need to show the upper bound
\begin{align}\label{eq:jl-bdA2statement}
 \Pr\big(
L_{u,a_{u,t}}=+1 \,\big|\, \histE{2}(u,a_{u,t})
\big)
\leq 2/3\,.
\end{align} 
where to shorten the notation, given user $u$ and item $i$ we defined the event $\histE{2}(u,i)$ and the set $\hist{2}(u,i)$ as follows:
\begin{align}
\histE{2}(u,i) & = \big\{c_i^t < \til, \,d_u^t \geq \tul,\,
\big(\BBjul{u}{i}{t}\big)^c,\, \Omega\Big\}\\
\hist{2}(u,i)&=
\big\{
 \text{realizations of }
\big(\H_{t-1}, \,\Xi, \, \tIni, \tU \big)
\text{ such that }
\notag
\\&\qquad\qquad
\histE{2}(u,i) \text{ holds}
\text{ and } a_{u,s}\neq i \text{ for all } s<t
\, \big\} 
\label{eq:defA2Event}
\end{align}
be the set of possible realizations of the model parameters and history up to time $t-1$  consistent with the event conditioned upon in~\eqref{eq:jl-bdA2statement}.
\paragraph*{Set $\hist 2(u,i)$ is well-defined.}
We note that the condition in \eqref{eq:defA3Event} is a function of the variables $\big(\H_{t-1}, \,\Xi, \, \tIni, \tU \big)$, so one can determine whether the latter satisfies the former. We spell this out as follows.
The history $\H_{t-1}$ determines whether item $i$ has been recommended to user $u$ before or not.
The value of $c_i^t$ (as in Definition~\ref{def:jl-cit}) is a function of the previous recommendations, summarized in $\H_{t-1}$, and the type of all users  except user $u$, i.e., $\tUnu$
\footnote{Conditioning on the event $a_{u,t}=i$ implies that 
user $u$ has not been recommended item $i$ by time $t-1$. 
Hence ambiguity in the type of user $u$ does not make the value of $c_i^t$ ambiguous.}. 
Similarly, the value of $d_u^t$ (as in Definition~\ref{def:jl-dut}) is a function of the previous recommendations, summarized in $\H_{t-1}$, and the type of all items  except item $i$, i.e., $\tIni$. 
Furthermore, the row regularity and column regularity of the preference matrix $\Xi$ is a function only of $\Xi$.

Let $\bj_i$ and $\bx_i$ be as defined in Lemma~\ref{l:A2propbound}.
Note that $\bj_i$ is  a deterministic function of realization of $\H_{t-1}$ and $\tU$.
Given any realization of $\H_{t-1}$ and $\tU \in \hist{2}(u,i)$, the type of item $i$ is uniform on $\Lambda_{\bx_i}(\Xi^T_{\bullet \bj_i})$ and $|\bj_i|= c_i^{t}<\til$. 

 So for any realization of 
 $\big(\H_{t-1}, \,\Xi, \, \tIni, \tU) \big)
\in\hist{2}(u,i)$,
 \begin{align*}
\Pr\Big[L_{u,i}=+1 \,\big|\, \H_{t-1}, \,\Xi, \, \tIni, \tU \Big]
& \overset{(a)}{=}
 \frac{|\Lambda_{\bx_i^+}\b(\Xi^T_{\bullet (\bj_i,\tau_U(u))}\b)|}
 {|\Lambda_{\bx_i}\b(\Xi^T_{\bullet \bj_i}\b)|}
\overset{(b)}{\leq }
\frac{(1+\eta)2^{|\bj_i|}}
{(1-\eta) 2^{|\bj_i|+1}} 
 \overset{(c)}{\leq }
2/3\,,
\end{align*}
where (a) uses Lemma~\ref{l:A2propbound}. 
(b) uses the row regularity of matrix $\Xi$ ($ \Xi^T \in \rgl_{\til,\eta}$ as in Definition~\ref{def:ul-reg}) and $|\bj_{i}|<\til$. (c) uses
$\eta=1/13$.

\paragraph*{Applying tower property}
Using the total probability lemma on above display, for any $i$ such that $a_{u,s}\neq i$ for all $s<t$, 
\begin{align}
\label{eq:BoundProptowerH2}
   \Pr\big(L_{u,i}=+1 \,\big|\, \H_{t-1}, \histE{2}(u,i) \big) \leq 2/3\,.
\end{align}
Recall that there is a random variable $\zeta_{u,t}$, independent of all other variables, such that $a_{u,t}=f_{u,t}(\H_{t-1},\zeta_{u,t})$, for some deterministic function $f_{u,t}$. 
Also, for all $i$ such that $u$ has not rated $i$ before, $\histE{2}(u,i) \in \sigma(\H_{t-1}, \Xi, \tIni, \tU)$. This is proved using the the same justification as above showing that the set $\hist{2}(u,i)$ is well-defined. Clearly, $L_{u,i}\in \sigma(\Xi,\tU, \tI)$.
So conditioning on $\H_{t-1}$, $\{a_{u,t}=i\}$ is independent of event $\histE{2}(u,i)$ and $L_{u,i}$. Hence, 
\begin{align}
       \Pr\big(L_{u,i}=+1 \,\big|\, \H_{t-1}, \histE{2}(u,a_{u,t}), a_{u,t}=i \big)
       &=   \Pr\big(L_{u,i}=+1 \,\big|\, \H_{t-1}, \histE{2}(u,i), a_{u,t}=i \big)\notag
       \\&=
      \Pr\big(L_{u,i}=+1 \,\big|\, \H_{t-1}, \histE{2}(u,i) \big)\,.
         \label{eq:NExtItemH2}
\end{align}

We use these properties to get
  \begin{align*}
\Pr&\big(L_{u,a_{u,t}}=+1 \,\big|\,  \H_{t-1}, \histE{2}(u,a_{u,t}) \big)
\\& \overset{(a)}{=}
\sum_{i\in\mathbb{N}}
\Pr\big( a_{u,t}=i \,\big|\, \H_{t-1}, \histE{2}(u,a_{u,t}) \big)
\Pr\big(L_{u,i}=+1 \,\big|\, \H_{t-1}, \histE{2}(u,a_{u,t}), a_{u,t}=i \big)
\\
&\overset{(b)}{=} \sum_{i: a_{u,s} \neq i \text{ for all } s<t}
\Pr\big( a_{u,t}=i \,\big|\, \H_{t-1}, \histE{2}(u,a_{u,t})\big)
 \Pr\big(L_{u,i}=+1 \,\big|\, \H_{t-1}, \histE{2}(u,i) \big)
\overset{(c)}{\leq} 2/3\,.
\end{align*}
Here (a) uses total probability lemma $\{L_{u,a_{u,t}}=+1\} = \bigcup_{i\in\mathbb{N}} \{a_{u,t}=i, L_{u,i}=+1\}$ and Eq.~\eqref{eq:NExtItemH2}. (b) uses the assumption that an item can be recommended at most once to a user to take the sum only over the items never recommended to $u$ before. Using Eq.~\eqref{eq:BoundProptowerH2} in each term of the sum gives (c). 

Using the tower property over $\H_{t-1}$ on the above display gives the statement of proposition for $b=2$.
 \qed


\subsubsection{Proof of Proposition~\ref{p:badUncertain} with $b=4$}
\begin{propn}[\ref{p:badUncertain} with $b=4$]
\[\Pr(L_{u,a_{u,t}}=-1\, |\, c_{a_{u,t}}^t \geq \til, d_{u}^t \geq \tul, (\BBjjl{u}{
a_{u,t}}{t})^c) = 1/2\,.\]
\end{propn}

 The proof is less similar to the cases $b=1, 2$ and 3$3$. The major difference comes from the observation that the uncertainty in the preference of user $u$ for item $i$ is not due to the uncertainty in the type of either one, but rather on the value of the relevant element of preference matrix:
 on event $(\BBjjl{u}{i}{t})^c)$, even revealing the type of user $u$ and the type of item $i$, the preference of $u$ for $i$ is uncertain. This is proved later in Lemma~\ref{l:A4propbound} whose proof requires the statement of next lemma.

\begin{lemma}
\label{l:PHist4}
Given $h_t$ in which no user from type  $\tau_U(u)$ has rated any item with type  $\tau_I(i)$ by time $t-1$,
\begin{align*}
    \Pr\big(\H_{t}= h_{t},&
\,\big|\, \xi_{v,j}=+1, \tau_I(i)=j, \tau_U(u)=v, \Xi^{\sim(v,j)},\,  \tIni, \tUnu \big)
\\
&=    \Pr\big(\H_{t}= h_{t},
\,\big|\, \xi_{v,j}=-1, \tau_I(i)=j, \tau_U(u)=v, \Xi^{\sim(v,j)},\, \tIni, \tUnu \big)\,.
\end{align*}
\end{lemma}
\begin{proof}
 The history $\H_{t}$ includes all the items recommended and the associated $\pm 1$ feedback up to and including time $t$; its distribution is complicated, but we do not need to calculate the posterior probability of $\H_{t}$ in the lemma statement. 

    Fix $\tIni, \tUnu$ such that $\tau_I(i)=j$ and $\tau_U(u)=v$. Also, fix $\Xi^{\sim(v,j)}$, as well as a specific possible trajectory for the history $h_{t}$ in which no user $u'$  with type $v$ has  rated any item $i'$ with type $j$. 
    
    We consider two copies of the history random process $\H_{t}$: the first $\H_{t}^1$, where $\xi_{v,j}=+1$, and the second $\H_{t}^{2}$ where $\xi_{v,j}=-1$.
    For all users $u'$ (including $u$) and times $s\leq t$, the recommended item is described by a function $a_{u',s}=f_{u',s}(\H_{s-1},\zeta_{u',s})$ for some auxiliary independent random variable $\zeta_{u',s}$. We thus have two copies of all these variables, $a_{u',s}^1$ and  $a_{u',s}^2$, and so forth. 
    
    We will give an inductive argument. Suppose that one of these two cases holds for some $s,0\leq s \leq t-1$:
    \begin{enumerate}
        \item $\H_{s}^1 = \H_{s}^2=h_{s}$, or
        \item Both $\H_{s}^1 \neq  h_{s}$ and  $\H_{s}^2\neq h_{s}$.
    \end{enumerate} 
    We will show below that at time $s+1$ one of these two cases holds.  
Note that the first case is tautologically true for $s=0$, since then the histories are empty, which serves as our base case. By induction, one of the two cases holds at time $t$, which implies that for any $h_t$ 
\begin{align*}
   & \Pr\big(\H_{t}= h_{t},
\,\big|\, \xi_{v,j}=+1, \tau_I(i)=j, \tau_U(u)=v, \Xi^{\sim(v,j)},\,  \tIni, \tUnu \big)
\\
&=    \Pr\big(\H_{t}= h_{t},
\,\big|\, \xi_{v,j}=-1, \tau_I(i)=j, \tau_U(u)=v, \Xi^{\sim(v,j)},\, \tIni, \tUnu \big)
\end{align*}
    
    We now prove the inductive step. 
    First, observe that if case 2 holds above, i.e., the histories both differ from $h_s$, then case 2 continues to hold  for all future times. 

    Now suppose that case 1 holds at time $s$.  
    We couple the $\zeta_{u',s+1}^1$ and $\zeta_{u',s+1}^2$ random variables to be equal for all $u'$, which results in $a_{u',s+1}^1 = a_{u',s+1}^2$, i.e., all recommendations made at step $s+1$ in both processes are the same. These both agree with $h_{s+1}$, or both disagree with $h_{s+1}$, and the latter case puts us into case 2 at time $s+1$.

    In the former case where $a_{u',s+1}^1 = a_{u',s+1}^2$ agree with $h_{s+1}$, the feedback $L_{u',a_{u',s+1}}^1$ and $L_{u',a_{u',s+1}}^2$ from these recommendations is the same in both $\H_{s}^1$ and $\H_{s}^2$ copies: 
  \begin{enumerate}
        \item  For users $u'$  such that $\tau_U(u')\neq v$ 
        the types and the elements of the preference matrix for these user types are the same in both copies and hence preferences are the same.
        \item  For users $u'$  such that $\tau_U(u')= v$, the types of users and items are the same in both copies; however, the preference of user type $v$ for item type $j$ is different in two copies.
        
        However, since in the case we are analyzing, $a_{u',s+1}^1 = a_{u',s+1}^2$ agree with $h_{s+1}$, they also agree with $h_t$.
        It is  also assumed that according to $h_t$, no user of type $v$ rates any item of type $j$ by time $t>s$. Hence the feedback for these users will also be the same for both copies.
    \end{enumerate}
    This completes the argument.
\end{proof}

\begin{lemma}
\label{l:A4propbound}
Fix $\tI,\tU$, and $\Xi^{\sim(\tau_U(u),\tau_I(i))}$. Consider a copy of $\H_{t-1}$ such that $(\BBjjl{u}{i}{t})^c \in \sigma(\H_{t-1}, \tI, \tU)$ holds, i.e., according to $\H_{t-1}$ no user of type $\tau_U(u)$ has rated any item of type $\tau_I(i)$. 
Then
\begin{align*}
\Pr\big(L_{u,i}=+1 \,\big|\, \H_{t-1}, (\BBjjl{u}{i}{t})^c, \, \Xi^{\sim(\tau_U(u),\tau_I(i))}, \, \tI, \tU \big)
= 
\frac{1}{2}\,.
\end{align*}
\end{lemma}
\begin{proof}
    Lemma~\ref{l:PHist4} shows that for any $\H_{t-1}\in (\BBjjl{u}{i}{t})^c$, 
    \begin{align*}
    \Pr\big(\H_{t}= h_{t},&
\,\big|\, \xi_{v,j}=+1, \tau_I(i)=j, \tau_U(u)=v, \Xi^{\sim(v,j)},\,  \tIni, \tUnu \big)
\\
&=    \Pr\big(\H_{t}= h_{t},
\,\big|\, \xi_{v,j}=-1, \tau_I(i)=j, \tau_U(u)=v, \Xi^{\sim(v,j)},\, \tIni, \tUnu \big)
\end{align*}
Hence, define $A=\Pr\big(\H_{t-1}, (\BBjjl{u}{i}{t})^c, \,   \,\big|\,  \xi_{\tau_U(u), \tau_I(i)}=+1, \, \Xi^{\sim(\tau_U(u),\tau_I(i))}, \, \tI, \tU \big)$ and note that $\Pr\big(\xi_{\tau_U(u), \tau_I(i)}=+1 \,\big|\, \Xi^{\sim(\tau_U(u),\tau_I(i))}, \, \tI, \tU \big) =\frac{1}{2}$, an application of Bayes theorem gives
\begin{align*}
\Pr&\big(L_{u,i}=+1 \,\big|\, \H_{t-1}, (\BBjjl{u}{i}{t})^c, \, \Xi^{\sim(\tau_U(u),\tau_I(i))}, \, \tI, \tU \big)
\\&= 
\Pr\big(\xi_{\tau_U(u), \tau_I(i)}=+1 \,\big|\, \H_{t-1}, (\BBjjl{u}{i}{t})^c, \, \Xi^{\sim(\tau_U(u),\tau_I(i))}, \, \tI, \tU \big)
\\&= 
\frac{A\frac{1}{2}}{A\frac{1}{2} + 2\frac{1}{2}} = \frac{1}{2}\,.\qedhere
\end{align*}
\end{proof}

\subsubsection*{Proof of Prop.~\ref{p:badUncertain} with $b=4$.}
We need to show the upper bound
\begin{align}\label{eq:jl-bdA4statement}
 \Pr\big(
L_{u,a_{u,t}}=+1 \,\big|\, \histE{4}(u,a_{u,t})
\big)
=1/2\,.
\end{align} 
where to shorten the notation, given user $u$ and item $i$ we defined the event $\histE{4}(u,i)$ and the set $\hist{4}(u,i)$ as follows:
\begin{align}
\histE{4}(u,i) & = \big\{c_i^t \geq \til, \,d_u^t \geq \tul,\,
\big(\BBjjl{u}{i}{t}\big)^c\Big\}\\
\hist{4}(u,i)&=
\big\{
 \text{realizations of }
\big(\H_{t-1}, \,\Xi^{\sim (\tau_U(u), \tau_I(i))}, \, \tI, \tU \big)
\text{ such that }
\notag
\\&\qquad\qquad
\histE{4}(u,i) \text{ holds}
\text{ and } a_{u,s}\neq i \text{ for all } s<t
\, \big\} 
\label{eq:defA4Event}
\end{align}
be the set of possible realizations of the model parameters and history up to time $t-1$  consistent with the event conditioned upon in~\eqref{eq:jl-bdA4statement}.

\paragraph*{Set $\hist 4(u,i)$ is well-defined.}
We note that the condition in \eqref{eq:defA4Event} is a function of the variables $\big(\H_{t-1}, \,\Xi^{\sim (\tau_U(u), \tau_I(i))}, \, \tI, \tU \big)$, so one can determine whether the latter satisfies the former. We spell this out as follows.
The history $\H_{t-1}$ determines whether item $i$ has been recommended to user $u$ before or not.
The value of $c_i^t$ (as in Definition~\ref{def:jl-cit}) is a function of the previous recommendations, summarized in $\H_{t-1}$, and the type of all users.
Similarly, the value of $d_u^t$ (as in Definition~\ref{def:jl-dut}) is a function of the previous recommendations, summarized in $\H_{t-1}$, and the type of all items.
The previous recommendations summarized in $\H_{t-1}$ and the types $\tI, \tU$ determine whether event $\BBjjl{u}{i}{t}$ has occurred.

The tower property of expectation over the statement of lemma~\ref{l:A4propbound} on realization of $\big(\H_{t-1}, \,\Xi^{\sim (\tau_U(u), \tau_I(i))}, \, \tI, \tU \big)$ in $\hist{4}(u,i)$ gives
$$
\Pr
\big[ 
L_{u,i}=+1 \,\big|\, \H_{t-1}, \,\Xi^{\sim (\tau_U(u), \tau_I(i))}, \, \tI, \tU 
\big] =\frac{1}{2}\,.
$$
\paragraph*{Applying tower property}
Using the total probability lemma on above display, for any $i$ such that $a_{u,s}\neq i$ for all $s<t$, 
\begin{align}
\label{eq:BoundProptowerH4}
   \Pr\big(L_{u,i}=+1 \,\big|\, \H_{t-1}, \histE{4}(u,i) \big) =1/2\,.
\end{align}

Recall that there is a random variable $\zeta_{u,t}$, independent of all other variables, such that $a_{u,t}=f_{u,t}(\H_{t-1},\zeta_{u,t})$, for some deterministic function $f_{u,t}$. So conditioning on $\H_{t-1}$, $\{a_{u,t}=i\}$ is independent of event $\histE{2}(u,i)$ and $L_{u,i}$:
\begin{align}
       \Pr\big(L_{u,i}=+1 \,\big|\, \H_{t-1}, \histE{4}(u,a_{u,t}), a_{u,t}=i \big)
       &=   \Pr\big(L_{u,i}=+1 \,\big|\, \H_{t-1}, \histE{4}(u,i), a_{u,t}=i \big)\notag
       \\&=
      \Pr\big(L_{u,i}=+1 \,\big|\, \H_{t-1}, \histE{4}(u,i) \big)\,.
         \label{eq:NExtItemH4}
\end{align}

We use these properties to get
  \begin{align*}
\Pr&\big(L_{u,a_{u,t}}=+1 \,\big|\,  \H_{t-1}, \histE{4}(u,a_{u,t}) \big)
\\& \overset{(a)}{=}
\sum_{i\in\mathbb{N}}
\Pr\big( a_{u,t}=i \,\big|\, \H_{t-1}, \histE{4}(u,a_{u,t}) \big)
\Pr\big(L_{u,i}=+1 \,\big|\, \H_{t-1}, \histE{4}(u,a_{u,t}), a_{u,t}=i \big)
\\
&\overset{(b)}{=} \sum_{i: a_{u,s} \neq i \text{ for all } s<t}
\Pr\big( a_{u,t}=i \,\big|\, \H_{t-1}, \histE{4}(u,a_{u,t})\big)
 \Pr\big(L_{u,i}=+1 \,\big|\, \H_{t-1}, \histE{4}(u,i) \big)
\overset{(c)}{=} 1/2\,.
\end{align*}
Here (a) uses total probability lemma $\{L_{u,a_{u,t}}=+1\} = \bigcup_{i\in\mathbb{N}} \{a_{u,t}=i, L_{u,i}=+1\}$ and Eq.~\eqref{eq:NExtItemH4}. (b) uses the assumption that an item can be recommended at most once to a user to take the sum only over the items never recommended to $u$ before. Using Eq.~\eqref{eq:BoundProptowerH4} in each term of the sum gives (c). 

Using the tower property over $\H_{t-1}$ on the above display gives the statement of proposition for $b=4$.
 \qed


\subsection{Regret in terms of Bad Recommendations: Cor.~\ref{l:jl-reg-lower}}
\label{sec:ReginTermsofBadProof}

\begin{corn}[\ref{l:jl-reg-lower}]
The regret is lower bounded as
\[\reg(T)\geq  \frac{1}{3N}\,\Exp{\bad(T)}- \frac{1}{3N}T\,.\]
where using~\eqref{eq:jl-def-bad}
\begin{align*}
\bad(T)& := \sumTN \ident\b\{\Bsf{1}_{u,t}\cup\Bsf{2}_{u,t} \cup\Bsf{3}_{u,t} \cup\Bsf{4}_{u,t} \b\}\,.
\end{align*}
\end{corn}
\begin{proof}
Define $B= \Bsf{1}_{u,t}\cup \Bsf{2}_{u,t}\cup \Bsf{3}_{u,t}$.
Since $\Bsf{1}_{u,t}, \Bsf{2}_{u,t}, \Bsf{3}_{u,t} $ and $ \Bsf{4}_{u,t}$ are disjoint events, we have
\begin{align*}
    \Pr\b(L_{u,a_{u,t}} = -1\b) & \geq 
      \Pr\b(L_{u,a_{u,t}} = -1, B\b)  
      +   \Pr\b(L_{u,a_{u,t}} = -1, \Bsf{4}_{u,t} \b)
      \\
     \Pr\b(L_{u,a_{u,t}} = -1, B\b)  & \geq     \Pr\b(L_{u,a_{u,t}} = -1, B, \Omega\b) 
    \\ & = \Pr\b(L_{u,a_{u,t}} = -1 |  B, \Omega\b) \Pr\b(B, \Omega\b) 
     \\& \overset{(a)}{\geq} 1/3 \,\,\Pr\b(B, \Omega\b) 
  \geq 1/3 \,\big[\Pr\b(B\b) -\Pr\b( \Omega^c\b)\big] 
        \\& = 1/3[\Pr\b(\Bsf{1}_{u,t}\b)+  \Pr\b(\Bsf{2}_{u,t}\b) + \Pr\b(\Bsf{3}_{u,t}\b)]
     - 1/3 \Pr\b( \Omega^c\b)
     \\
     \Pr\b(L_{u,a_{u,t}} = -1, \Bsf{4}_{u,t} \b) & \overset{(b)}{=} 1/2 \Pr\b(\Bsf{4}_{u,t}\b)
\end{align*}
where in (a), and (b) we used the statement of Proposition~\ref{p:badUncertain}. Applying Corollary~\ref{cor:regularityprob} which states $\Pr(\Omega^c)\leq 1/(4N)$ gives
\begin{align*}
    N\reg(T)& = \sumTN  \Pr\b(L_{u,a_{u,t}} = -1\b)
    \\
    &\geq 1/3 \sumTN\Pr[\cup_{b=1}^4  \Bsf{b}_{u,t} \b)] - 1/3 \Pr\b( \Omega^c\b)\\
    &\geq   \frac{1}{3}\,\Exp{\bad(T)}  - \frac{TN}{12N}\,.
\end{align*}
\end{proof}

\section{Combining Lower Bounds for Regret}

\subsection{Proof of Lemma~\ref{l:BoundminUserSize}}
\label{sec:BoundminUser}
$\minUser$ is the maximum number of users of any type as defined in Eq.~\eqref{def:jl-minUser}
\begin{equation*}
    \minUser = \max_{\tu\in[\qu]} \b| \big\{u\in [N]: \tau_U(u)=\tu \big\} \b|\,.
\end{equation*}
A priori, the users have independent uniform types over $[\qu]$. So $\minUser$ is the maximum number of balls in a bin and bag argument where $N$ balls are thrown into $\qu$ bins. Applying Chernoff Bound, Lemma~\ref{l:Chernoff} and a union bound guarantees
\[\Pr(\minUser>3N/\qu) \leq \qu \exp(-2N/(3\qu)) \leq 3/4\,,\]
where we used $N>20\qu\log^2\qu$ in the last inequality.

\subsection{Proof of Lemma~\ref{l:JL-bdGammaStarP}}
\label{sec:lbdRjTPproof}
 To prove Lemma~\ref{l:JL-bdGammaStarP} rigorously, we use the following definition: 

\begin{definition} \label{def:JL-rjT}
Let $r_{\ti}^T$ be the number of items with item type $j$ that have been rated by time $T$:
\[r_{j}^T=|\{i: \tau_I(i)=j, a_{u,t}=i, u\in[N], t\leq T\}|.\]
\end{definition}

For any user $u$, $T$ items from $\IStotal^T=\ISstrong^T\cup\ISweak^T$ are recommended by time $T$.
The number of item types recommended to user $u$ by time $T$ is denoted by $\gamma_u^T = d_u^{T+1}$. So $T\leq \gamma_u^T \max_{j}r_{j}^T $ for all $u$ and 
\begin{equation}\label{eq:maxrjT}
\max_{j}r_{j}^T \geq 
\frac{T}{\min_{u} \gamma_u^T}
= \frac{T}{\gamma_*^T}  \,.
\end{equation}
Lemma~\ref{l:JL-bdGammaStarP} is a direct consequence of the above property and the following lemma. \footnote{
A tighter bound in the statement of Lemma~\ref{l:JL-bdGammaStarP}  can be constructed as follows.  
 Lemma~\ref{l:JL-bdRjTP}  shows that with probability at least half, $\Itotal^T >  \Theta_{\qi}\left(\frac{T}{\gamma_*^T}\right)$ where the function $\Theta_{q}(k)$ is defined as 
\begin{align*}
\Theta_{q}(k) := \begin{cases} 
1\,,
 \,\quad&
 \text{if }\, k\leq 2
\\
\sqrt{q}/3\,, 
\,\quad&
\text{if }\, 2< k\leq  \frac{3}{\log 2}\log q
\\
q/2\,, 
\,\quad&
\text{if }\, \frac{3}{\log 2}\log q< k\leq 8\log q
\\
k q /4\,, 
\,\quad&
\text{if }\, 8\log q< k \,. \end{cases}
\end{align*}
But this tighter bound results in more complex description of statement of Therorem~\ref{t:joint-L} and the improvement in the statement of the Theorem is over the multiplicative logarithmic terms. Hence, we decided to use the form in Equation~\eqref{eq:bdGammaStar}.
}
\begin{lemma}
\label{l:JL-bdRjTP}
For $r_j^T$ in Definition~\eqref{def:JL-rjT}, with probability at least half,
\begin{align*}
\max_{\ti\in[\qi]}r_j^T 
<
 \begin{cases} 
2\,,
 \,\quad&
 \text{if }\, \Itotal^T< k_0:=\sqrt{\qi}/3\\
  3\, \frac{\log q}{\log \qi- \log \Itotal^T}\,,
 \,\quad&
 \text{if }\, k_0\leq \Itotal^T< k_1:=\qi/2\\
8\log \qi\,, 
\,\quad&
\text{if }\, k_1\leq \Itotal^T< k_2:=2\qi\log \qi
\\
{4\Itotal^T}/{\qi}\,, 
\,\quad&
\text{if }\, k_2\leq \Itotal^T\,. \end{cases}
\end{align*}

\end{lemma}
The proof of this lemma is exactly the same as the proof  of Claim 7.9. in~\cite{RS1}.
We repeat the proof here for the sake of completeness

\begin{proof}
First, we define a useful martinagle. 
Let $r^t =(r_1^t,\dots, r_{\qi}^t)$ where $r_{\ti}^t$ is defined in~\eqref{def:JL-rjT}. Note that $\Itotal^t=\sum_{\ti\in[\qi]} r_{\ti}^t$ is the total number of recommended items at the end of time $t$. Any new item has type uniformly distributed on $[\qi]$; as a consequence, the sequence $r_{j}^t - \Itotal^t/\qi$ is a martingale with respect to filtration $\Fc_t = \sigma(r^0,r^1,\dots, r^t)$, because $\Itotal^t$ is incremented whenever a new item is recommended and  each new item increases $r^t_j$ by one  with probability $1/\qi$.

It turns out to be easier to work with a different martingale that considers recommendations to each user separately, so that the item counts are incremented by at most one at each step. Consider the lexicographical ordering on pairs $(t,u)$, where $(s,v)\leq(t,u)$ if either $s< t$ or $s=t$ and $v\leq u$ (such that the recommendation to user $v$ at time $s$ occurred before that of user $u$ at time $t$). For $j\in [\qi]$, let
$$
r_{j}^{t,u} 
=
\Big|
\big\{i: \tau_I(i)=j, a_{v,s}=i \text{ for some } (s,v)\leq (t,u)\big\}
\Big|\,.
$$
Let $r^{t,u} = (r_1^{t,u} ,\dots, r_{\qi}^{t,u} )$ and define $\rho^{t,u} = \sum_j	r_{j}^{t,u}$ to be the total number of items recommended by $(t,u)$, \textit{e.g.}, $\rho^{T,N}=\Itotal^T$. 
	We now define a sequence of stopping times $Z_k \in \mathbb{N}\times [N]$, 
	$$
	Z_k = \min\big\{(t,u)> Z_{k-1}: \rho^{t,u} > \rho^{t,u-1}\big\}\,,
	$$
	where  $(t,0)$ is interpreted as $(t-1,N)$ and $Z_0=(0,N)$. $Z_k$ is the first $(t,u)$ such that a new item is recommended by the algorithm for the $k$-th time, so $\rho^{Z_k}=k$. The $Z_k$ are stopping times with respect to $(\rho^{t,u})$, and observe that $k^* = \max\{k: Z_k \leq (T,N)\} = \rho^{T,N} = \Itotal^T$ since $\Itotal^T$ is the total number of items recommended by the algorithm by the end of time $T$. Also, $\rho^{Z_{k^*}}=k^*=\Itotal^T$ and $r_j^{Z_{k^*}}=r_{\ti}^{(T,N)}$ for all $\ti\in[\qi]$.
	
Fix item type $j\in[\qi]$. The sequence $M^{t,u}_j=r_j^{t,u} - \rho^{t,u}/\qi$ is a martingale with respect to the filtration $\Fc^{t,u}=\sigma(r^{1,1},\dots, r^{t,u})$ \footnote{To see that, define the event $\mathcal{E}^{\mathrm{new}}_{u,t}=\{\text{the item } a_{u,t} \text{ has not been recommended before to anybody}\}$ where the order is based on the lexicographic order we define in the proof of Lemma~\ref{l:JL-bdRjTP}. Then,
\begin{align*}
\Ex\Big[r_j^{t,u} - \tfrac{\rho^{t,u}}\qi \, \big|\, \Fc^{t,u-1}\Big]
&-
\Big[r_j^{t,u-1} - \tfrac{\rho^{t,u-1}}\qi \, \Big]
\overset{(a)}{=} 
\Ex\Big[
\ident[\tau_I(a_{u,t})=j] - \frac 1\qi 
\, \big|\,
\Fc^{t,u-1},\, \mathcal{E}^{\mathrm{new}}_{u,t}\Big] \,\,
\Pr\Big[\mathcal{E}^{\mathrm{new}}_{u,t}
 \, \big|\,
\Fc^{t,u-1}\Big]\overset{(b)}{=}0\,,
\end{align*} 
(a) uses the fact that
 condition on event $\big(\mathcal{E}^{\mathrm{new}}_{u,t}\big)^c$,  we have $\rho^{t,u}=\rho^{t,u-1}$ and $r_j^{t,u}=r_j^{t,u-1}$.  
Equality (b) uses the assumption in the model which states that the prior distribution of type of  an item which has not been recommended before is uniform over $[\qi]$. Hence, $\Pr\Big[
\tau_I(a_{u,t})=j
\, \Big|\,
\Fc^{t,u-1}, \mathcal{E}^{\mathrm{new}}_{u,t}\Big] = 1/\qi$.}.
It follows that $\widetilde M^{k}_j := M^{(T,N)\wedge Z_k}_j$ is martingale as well, this time with respect to $\widetilde\Fc^k:=\Fc^{(T,N)\wedge Z_k}$. Since $Z_{k^*}\leq (T,N)$, we have $\widetilde M^{k^*}_j=M^{Z_{k^*}}_j$. We will use this notation to prove statement of the claim in three different regimes.
First, we would like to apply martingale concentration (Lemma~\ref{l:martingaleBound}) to $\widetilde M^{k}_j$, and to this end observe that $\mathrm{Var}(\widetilde M^{k}_j|\widetilde\Fc^{k-1}) \leq 1/\qi$ and $|\widetilde M^{k}_j - \widetilde M^{k-1}|\leq 1$ almost surely. 	 
	
\paragraph*{Step~1}	
	For any $k\geq k_2:= 2\qi \log \qi$, Lemma~\ref{l:martingaleBound} gives
	\begin{align*}
\Pr\Big[\widetilde M_{j}^{k} \geq  \frac{3k}{\qi}\Big]
 \leq 
 \exp\bigg(\frac{-9k^2/\qi^2}{2(k/\qi + k/\qi)}\bigg) 
 =
  \exp\Big(-\frac{2k}{\qi}\Big)\,.
\end{align*}

This gives
\begin{align}
\Pr\Big[\max_{j\in[\qi]}&r_j^{Z_{k^*}} \geq \Theta_{\qi}(k^*), k^*\geq k_2\Big]
\leq 
 \Pr\Big[\exists k\geq k_2 
 \text{ s.t. } 
 \max_{j\in[\qi]}r_j^{Z_{k}} \geq \tfrac{4k}{\qi}\Big]
 \nonumber
\\
&
 \overset{(a)}{\leq } 
\Pr\Big[ 
\exists k\geq k_2
\text{ s.t. }
\widetilde M_j^{k} 
\geq  \tfrac{3k}{\qi}\Big] 
\nonumber \\&
 \overset{(b)}\leq  
\sum_{k\geq k_2} \exp\big(-\tfrac{2 k}{\qi}\big) 
= 
\frac{\exp\big(-\tfrac{2 k_2}{\qi}\big)}{1-\exp(-2/\qi)}
  \overset{(c)} \leq 
  \frac{1}{\qi^2}\,.
\label{eq:IL-mart-k-large}
\end{align}
where (a) uses $\rho^{Z_k}=k$. (b) uses a union bound and the inequality in the last display.  (c)  uses definition of $k_2$ and ${1-\exp(-2/\qi)}> \qi^{-2}$ (which is derived using $e^{-a}\leq 1-a+a^2/2$ and $\qi>1$). 
\paragraph*{Step~2} For any $k_2> k$ we get 
\begin{align*}
	\Pr\Big[\widetilde M_j^{k} \geq  6\log \qi\Big]
	 & \leq
	 \exp\bigg(\frac{-36\log^2\qi}{2(k/\qi +  2\log \qi)}\bigg) 
\\
& \leq
	\exp\bigg(\frac{-36\log\qi}{2k_2/(\qi\log\qi) +  4}\bigg) 
	\leq 
	\frac{1}{\qi^4}\,.
\end{align*}

This gives
\begin{align}
\Pr\Big[\max_{j\in[\qi]}r_j^{Z_{k^*}} \geq \Theta_{\qi}(k^*), k^*<k_2\Big]
& \leq
 \Pr\Big[\exists k<k_2 \text{ s.t. } \max_{j\in[\qi]}r_j^{Z_{k}} \geq 8 \log \qi\Big]
 \nonumber
\\& 
\overset{(a)}{\leq} \, 
\Pr\Big[ 
\exists k< k_2\text{ s.t. }\widetilde M_j^{k} 
\geq  6\log \qi\Big] 
\nonumber \\&
\overset{(b)}\leq 
\sum_{k< k_2} \frac{1}{\qi^4}
\leq 
\frac{k_2}{\qi^4}
\leq 
\frac{1}{\qi^2}\,.
\label{eq:IL-mart-k-medium}
\end{align}
(a) uses $\rho^{Z_k}=k$. (b) uses the inequality in the above display. 

\paragraph*{Step~3}
This step, $ k< k_1 := \qi/2 $, corresponds to bounding the number of balls in the fullest bin when the number of balls, $k$, is sublinear in the number of bins, $\qi$ (since $k=\qi^{1-3\delta}$ with $\delta>\frac{1}{8  \log \qi}$).  We will show that in this regime, the number of balls in the fullest bin is bounded by $1/\delta$. 
For given  $ k< k_1$, define $\delta=\frac{1}{3}\,\frac{\log\qi - \log k}{\log\qi }$ (such that $k=\qi^{1-3\delta}$). Then,
\begin{align}
\Pr\Big[  \max_{j\in[\qi]} r^{Z_{k}}_j \geq  1/\delta\Big]
 \overset{(a)}\leq
\qi\Pr\Big[ r^{Z_{k}}_1 \geq  1/\delta\Big]
\overset{(b)}\leq
\qi {{k}\choose{1/\delta}} \frac{1}{\qi^{1/\delta}}
 \overset{(c)}\leq \frac{5}{\qi ^2} 
 \label{eq:IL-mart-k-K_1_1}
\end{align}
(a) is a union bound over $\ti\in[\qi]$. (b) uses the fact that $r_1^{Z_{k+1}}=r_1^{Z_{k}}+1$ with probability $1/\qi$ and $r_1^{Z_{k+1}}=r_1^{Z_{k}}$ with probability $1-1/\qi$ independently of $r_1^{Z_{k}}$.  (c) holds for every $\delta>\frac{1}{8\log \qi}$ ( which is due to $k<k_1$) using ${k \choose 1/\delta}\leq (k e \delta)^{1/\delta}$.

\begin{align}
	\Pr\Big[\max_{j\in[\qi]} r^{T,N}_j \geq   3\frac{\log\qi }{\log\qi - \log k^*} , \, k^*< k_1\Big] 
&	\overset{(a)}=
	\Pr\Big[\max_{j\in[\qi]} r^{Z_{k^*}}_j \geq  3\frac{\log\qi }{\log\qi - \log k^*} ,\, k^*< k_1\Big]
	\nonumber\\
&\leq
	\Pr\Big[ \exists k< k_1 \text{ s.t. }
	 \max_{j\in[\qi]} r^{Z_{k}}_j \geq  3\frac{\log\qi }{\log\qi - \log k} \Big]
		\nonumber\\&
		\overset{(b)}\leq
k_1 \frac{5}{\qi^2} \leq \frac{5}{\qi} 
 \label{eq:IL-mart-k-K_1}
\end{align}
(a) uses $ r^{Z_{k^*}}_j  =   r^{(T,N)}_j $.  (b) uses a union bound and~\eqref{eq:IL-mart-k-K_1_1}. Last inequality uses $k_1<\qi$.

\paragraph*{Step~4}
This step uses a variation of the Birthday Paradox to bound $\max_{j\in[\qi]} r^{T,N}_j$.
\begin{align}
	\Pr\Big[\max_{j\in[\qi]} r^{T,N}_j \geq  2,\, k^*< k_0\Big] 
&	=
	\Pr\Big[\max_{j\in[\qi]} r^{Z_{k^*}}_j \geq  2,\, k^*< k_0\Big]
	\nonumber\\
&\leq
	\Pr\Big[ \exists k< k_0 \text{ s.t. } \max_{j\in[\qi]} r^{Z_{k}}_j \geq  2\Big]
\nonumber\\
&\overset{(a)}\leq
\Pr\Big[  \max_{j\in[\qi]} r^{Z_{k_0}}_j \geq  2\Big]
\nonumber\\&
=
1- 	\Pr\Big[ r^{Z_{k_0}}_j \leq 1 \text{ for all }{j\in[\qi]} \Big]
\nonumber\\
&\overset{(b)}=
1- 	\Pr\Big[ r^{Z_{k}}_j \leq 1 \text{ for all }{j\in[\qi]} \text{ and }  k\leq k_0\Big]
\nonumber\\
& = 
1- 	
\prod_{m=1}^{k_0}\Pr\Big[ r^{Z_{m}}_j \leq 1 \text{ for all }{j\in[\qi]} \,\big|\, r^{Z_{m-1}}_j \leq 1 \text{ for all }{j\in[\qi]} \Big]
\nonumber\\
&\overset{(c)}{=}
1- \prod_{m=1}^{k_0}
\left(1-\frac{m-1}{\qi}\right)
	 \leq
	  1- \big(1-\frac{k_0}{\qi}\big)^{k_0}
	  \overset{(d)}{\leq}
	\frac{2 k_0^2}{\qi}
	\leq 
	\frac{2}{9}
	\,.\label{eq:IL-mart-k-small}
\end{align}
(a) and (b) use the fact that $r^{Z_{k}}_j$ is a nondecreasing function of $k$. 
We define $r_j^{Z_0}=0$. The type of the $(k+1)$-th drawn item is independent of the type of the previous $k$ drawn items. Hence, conditional on $r^{Z_k}=\big(r_1^{Z_k}, \cdots,r_{\qi}^{Z_k}\big)$, the random variable $r^{Z_{k+1}}$ is independent of  $r^{Z_{k-1}}$. This gives equality (c). 
(d) uses $\exp\big(k\log(1-k/\qi)\big)\geq \exp\big(-k^2/(\qi-k)\big)\geq \exp\big(-2k^2/\qi\big) \geq 1-2k^2/\qi$ for $k\leq k_0\leq \sqrt{\qi}\leq \qi/2$. 

We put it all together,
\begin{align*}
\Pr\Big[ \max_{j\in[\qi]} r_j^T \geq \Theta_{\qi}(\Itotal^T)\Big]
& \overset{(a)}{=}
\Pr\Big[\max_{j\in[\qi]}r_j^{Z_{k^*}} \geq \Theta_{\qi}(k^*)\Big]\\
 &
 \overset{}{\leq}
\Pr\Big[\max_{j\in[\qi]}r_j^{Z_{k^*}} \geq \Theta_{\qi}(k^*), k^*\geq k_2\Big]
\\ & \quad
 + 
 \Pr\Big[\max_{j\in[\qi]}r_j^{Z_{k^*}} \geq \Theta_{\qi}(k^*),\,k_0\leq k^*<k_2\Big]
 \\
& \quad
 + 
 \Pr\Big[\max_{j\in[\qi]}r_j^{Z_{k^*}} \geq \Theta_{\qi}(k^*),\,\leq k^*<k_1\Big]
 \\ & \quad + 
 \Pr\Big[\max_{j\in[\qi]}r_j^{Z_{k^*}} \geq \Theta_{\qi}(k^*),\,k_0\leq k^*<k_0\Big]
\\
& \overset{(b)}\leq
\frac{2}{\qi^2}+ \frac{5}{\qi}+ \frac{2}{9}\leq \frac{1}{2}
\,.
	\end{align*}
(a) uses the definition of $Z_k$ and $k^*$. 
(b) 
uses~\eqref{eq:IL-mart-k-large},~\eqref{eq:IL-mart-k-medium},~\eqref{eq:IL-mart-k-K_1} and ~\eqref{eq:IL-mart-k-small}. Last inequality uses $\qi>10$. 
\end{proof}


\subsection{Information-Theoretic Lower Bound} \label{sec:thm-joint-L}

\begin{thmn}[\ref{t:joint-L}]
Let $\tul$ and $\til$ be as defined in Eq.~\eqref{eq:jl-thru}.
Any recommendation algorithm must incur regret lower bounded as below with numerical constant $c>0$
\begin{align*}
    N\reg(T) \geq & c
    \max\Big\{N, \min\{NT, N\tul, \sqrt{\qi}\}, \min\{\qu T, \sqrt{\qi}\},
    \\
   & \min\big\{\frac{NT}{\log\qi},\sqrt{T\qi N}, N\tul\big\},
    \min\{\frac{\qu T}{\log\qi}, \sqrt{T\qi\qu}\},\, 
    \min\big\{T\til, T\qu\big\}\Big\}\,.
\end{align*}
\end{thmn}

\begin{proof}

Corollary~\ref{l:jl-reg-lower} states that 
\[\reg(T)\geq  \frac{1}{3N}\,\Exp{\bad(T)}- \frac{1}{12N}T\,.\]
First, we get rid of the term in the above display. In Lemma~\ref{l:BoundBad1} it is proved that $\bad(T)\geq \Itotal^{T} $ with probability one. Also, since users get new recommendations at each time, by time $T$, at least $T$ items have been used by any algorithm: $\Itotal^{T}\geq T$ with probability one. 
So, 
\[\bad(T)- \frac{T}{4} \geq \frac{3}{4}\bad(T)\,.\]
Plugging this into above gives
\[N\reg(T)\geq  \frac{1}{4}\,\Exp{\bad(T)}\,.\]
Next, using Corollary~\ref{corr:LowerBoundBad} gives the lower bound
\begin{align}
\label{eq:Regminimax}
N\reg(T)
\geq 
\frac{1}{64}
\min_{\gamma\geq 1}
\max\big\{
f_1(\gamma), f_2(\gamma), f_{3}(\gamma)\big\} \,.
\end{align}
where
\begin{align}
f_1(\gamma) 
&=
\begin{cases}
T\qi /\gamma & \text{ if }  \gamma \leq T/(8\log \qi) 
 \\
 \max\{T,\sqrt{\qi}\}  & \text{ if }  T/(8\log \qi) <\gamma \leq T/2
 \\
 T  & \text{ if }  T/2 <\gamma
 \end{cases}
 \quad \quad \quad
f_2(\gamma)= \max\big\{ \qu\gamma , N\min\{\tul,\gamma\}\big\}
\notag
\\
f_3(\gamma)&=
T \min\big\{\til, \qu\big\}\,.
\label{eq:deffgammaLB}
 \end{align}

The remaining proof of the theorem is just based on finding the minimax in Eq.~\eqref{eq:Regminimax} in different regimes and does not contain any major insight. We will include these calculations below for the sake of completeness.

Note that the functions $f_1(\gamma), f_2(\gamma)$ and constant function $f_{3}(\gamma)$ are parameterized by the parameters of the system $\qu, \qi, N$ and $T$. So depending on the realization of these parameters, the above minimax takes different forms. 
To lower bound the minimax term, we observe that $f_{1}(\gamma)$ is a nonincreasing function of $\gamma$; $f_{2}(\gamma)$ is a nondecreasing function of $\gamma$; and $f_{3}(\gamma)$ is a constant function of $\gamma$ (Figure~\ref{fig:LBfig}).

\begin{figure}[t]
\includegraphics[width=15cm]{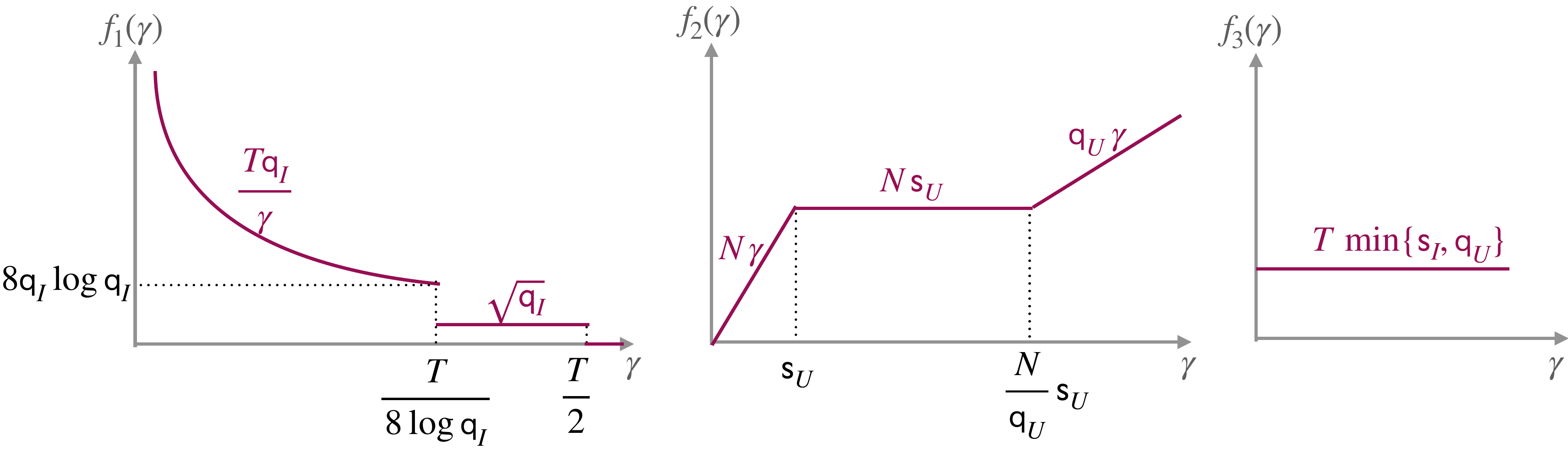}
\caption{Schematics of functions $f_1(\gamma), f_2(\gamma)$ and $f_3(\gamma)$.}
\label{fig:LBfig}
\centering
\end{figure}

Define $$M(T) := \min_{\gamma\geq 1}
\max\big\{
f_1(\gamma), f_2(\gamma)\big\}\,.$$

Since $f_3(\gamma)$ is constant function of $\gamma$,
\[
N\reg(T)
\geq 
\frac{1}{64}
\max\left\{M(T),\, T \min\big\{\til, \qu\big\}\right\}\,.
\]

\paragraph*{Computing $M(T)$}
Define $\gamma_1, \gamma_2$ and $\gamma_3$ to be the intersection of $T\qi/\gamma$ (the decreasing part of $f_1(\gamma)$) and $N\gamma$, $N\tul$, and $\qu\gamma$ respectively. So, 
\[\gamma_1=\sqrt{T\qi/N},\quad \gamma_2= T\qi/(N\tul), \text{ and } \gamma_3 = \sqrt{T\qi/\qu}\,.\]

Note that $f_1(\gamma)$ is equal to $T\qi/\gamma$ before the discontinuity at $T/(8\log \qi)$. 
Also, $f_2(\gamma)$ is a piecewise linear function of $\gamma$. 

We look at three different cases separately. Each case corresponds to a relative position of pieces of $f_1(\gamma)$ and $f_2(\gamma)$.

\paragraph*{Case I. $\sqrt{\qi} < N\tul$ and $8\qi\log\qi< N\tul$}
\begin{itemize}
    \item $M(T) = NT$ when $NT/2<\sqrt{\qi}$.
    \item $M(T) = \sqrt{\qi}$ when $NT/(8\log\qi) <\sqrt{\qi}\leq NT$.
    \item $M(T) = NT/(8\log\qi)$ when   $\sqrt{\qi}\leq NT/(8\log\qi)< \sqrt{T\qi N}$.
    \item $M(T) = \sqrt{T\qi N}$ when $\sqrt{T\qi N} \leq \min\{NT/(8\log\qi),N\tul\}$.
    \item $M(T)=N\tul$ when  $\sqrt{T\qi\qu}<N\tul <\sqrt{T\qi N}$.
    \item $M(T) = \sqrt{T\qi\qu}$ when $N\tul \leq \sqrt{T\qi\qu}$.
\end{itemize}
In This case, 
\begin{align*}
    N\reg(T)
\geq 
C
\max\Big\{ &N, \min\b\{NT,\sqrt{\qi}\b\}, \min\Big\{\frac{NT}{\log\qi}, \sqrt{T\qi N}, N\tul\Big\},
\\&\sqrt{T\qi\qu},\,  \min\big\{T\til, T\qu\big\}\Big\}\,.
\end{align*}

\afterpage{\clearpage}

\begin{figure}[h]
\centering     
\subfigure[$8\qi\log\qi\leq N\tul$]{\label{fig:a}\includegraphics[width=0.85\textwidth]{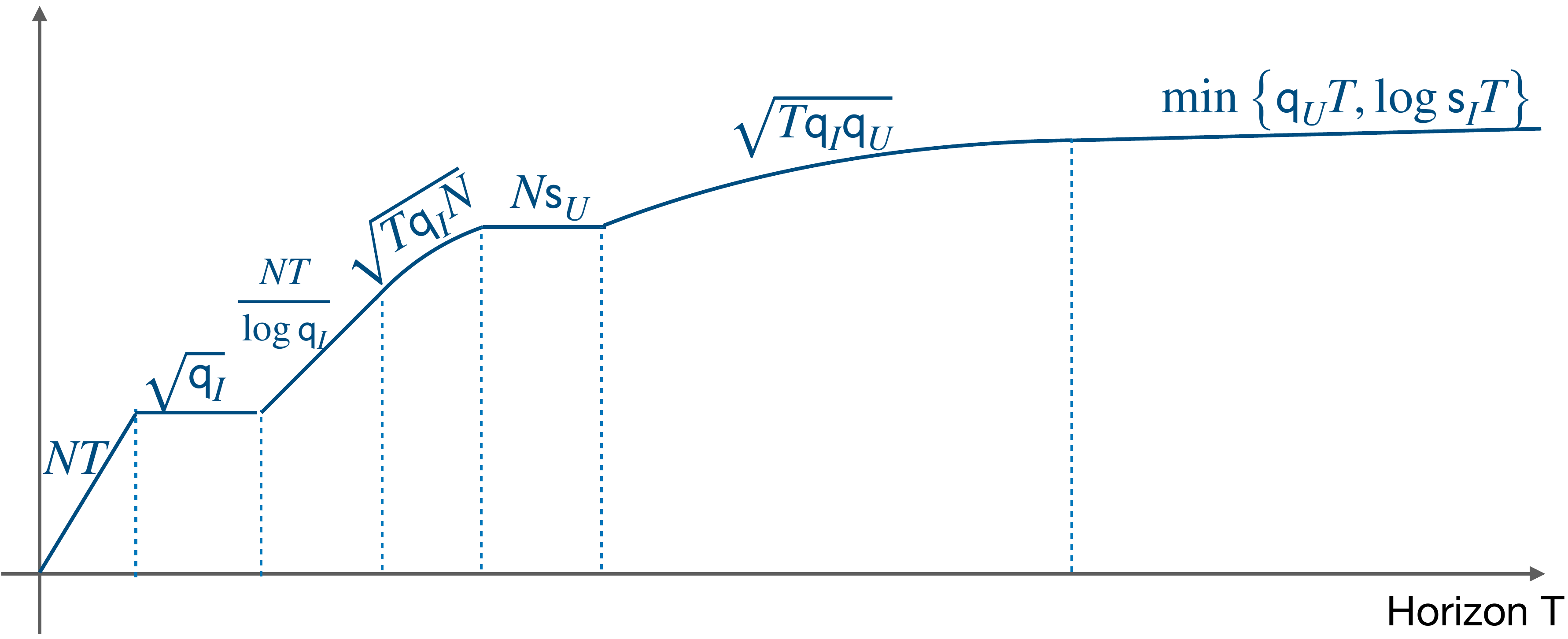}}
\hspace{4mm}\subfigure[$\sqrt{\qi}\leq N\tul< 8\qi\log\qi$]{\label{fig:b}\includegraphics[width=0.85\textwidth]{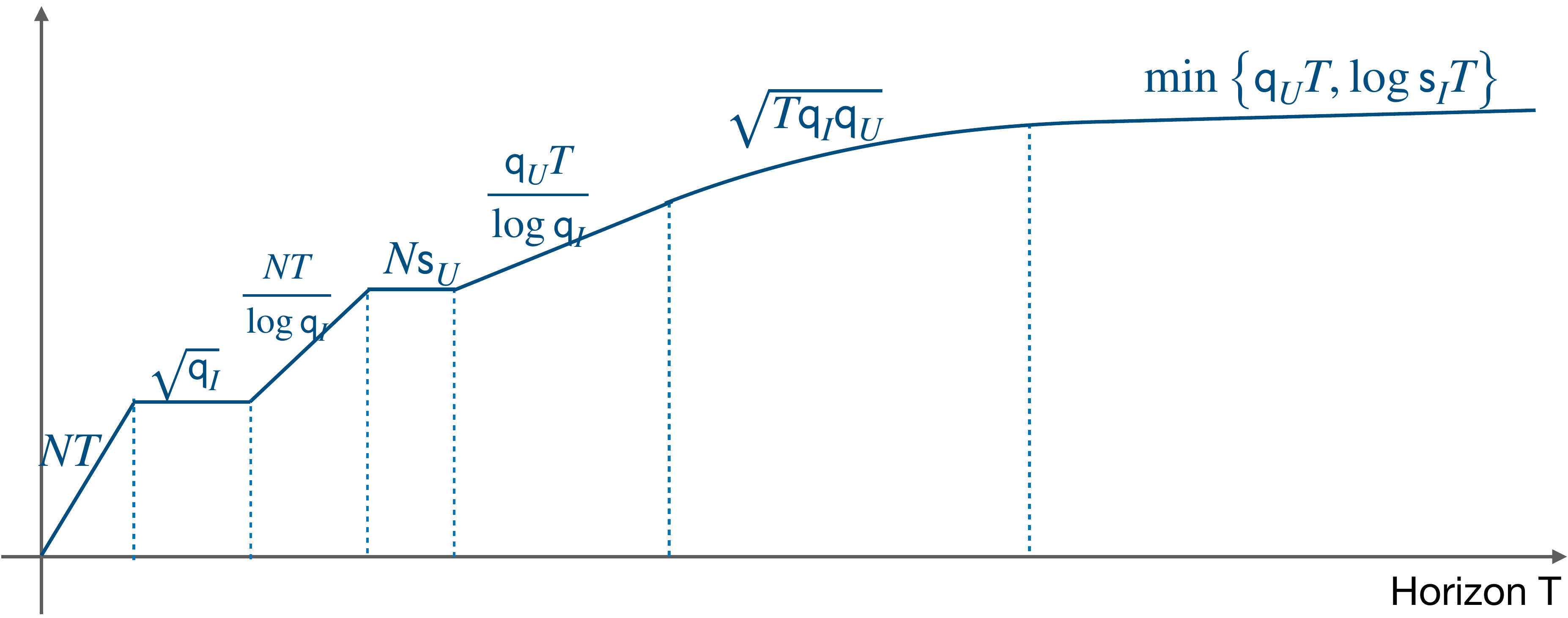}}
\subfigure[$N\tul<\sqrt{\qi}$]{\label{fig:b}\includegraphics[width=0.85\textwidth]{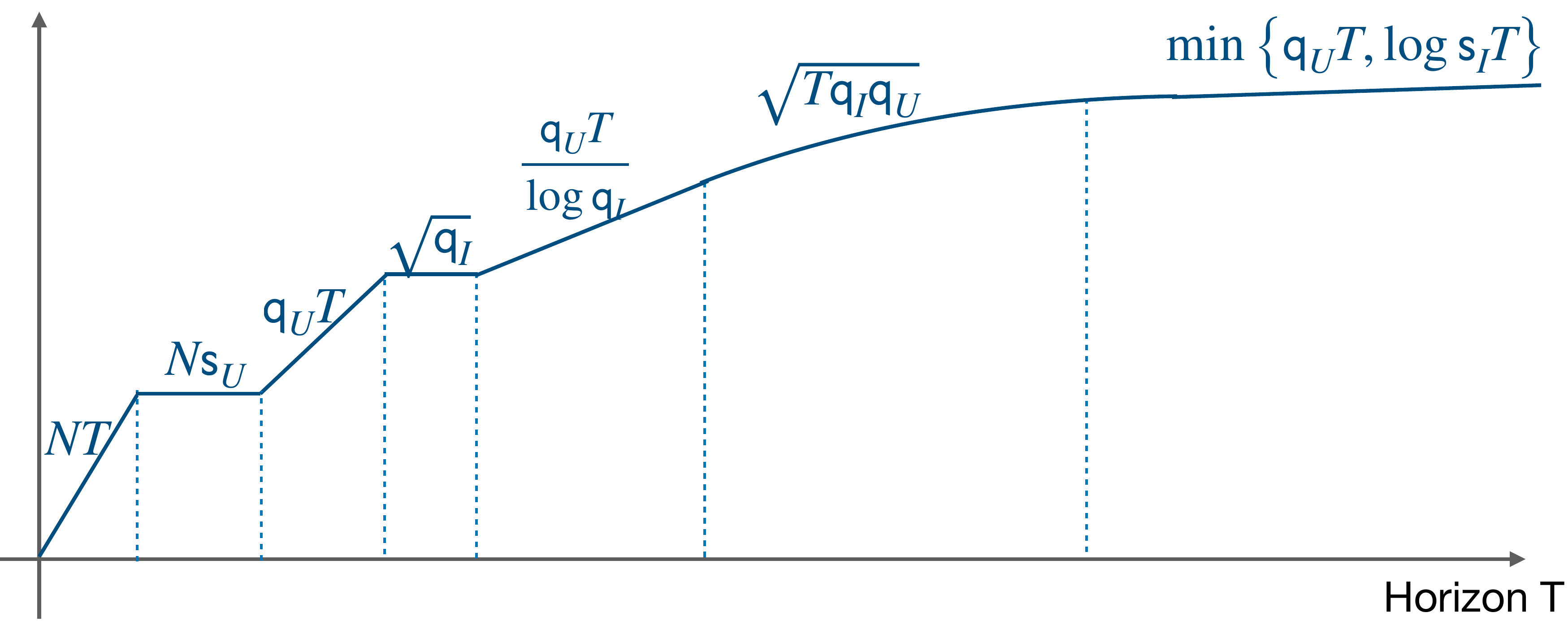}}
\caption{Various regimes of lower bound for regret corresponding to three different cases (a) $8\qi\log\qi\leq N\tul$, (b) $\sqrt{\qi}\leq N\tul< 8\qi\log\qi$; and (c) $N\tul<\sqrt{\qi}$.}
\label{fig:LBReg}
\end{figure}

\paragraph*{Case II. $\sqrt{\qi} < N\tul \leq 8\qi\log\qi$}
\begin{itemize}
    \item $M(T) = NT$ when $NT<\sqrt{\qi}$.
    \item $M(T) = \sqrt{\qi}$ when $ NT/(8\log\qi) <\sqrt{\qi}\leq NT$.
    \item $M(T) = NT/(8\log\qi)$ when $\sqrt{\qi}\leq NT/(8\log\qi)<  N\tul$.
    \item $M(T) = N\tul$ when $ \tul \leq T/(8\log\qi ) < N\tul/\qu$.
    \item $M(T) = \qu T/(8\log\qi)$ when $N\tul \leq \qu T/(8\log\qi ) < \qu\gamma_3 =\sqrt{T\qi\qu}$.
    \item $M(T) = \sqrt{T\qi\qu}$ when $\sqrt{T\qi\qu}\leq \qu T/(8\log\qi)$
\end{itemize}

In This case, 
\begin{align*}
    N\reg(T)
\geq 
C
\max\Big\{ &N, \min\{NT,\sqrt{\qi}\}, \min\{\frac{NT}{\log\qi}, N\tul\}, 
\\&\min\{\frac{\qu T}{\log\qi}, \sqrt{T\qi\qu}\},\,  \min\big\{T\til, T\qu\big\}\Big\}\,.
\end{align*}

\paragraph*{Case III. $N\tul \leq \sqrt{\qi}$}
\begin{itemize}
    \item $M(T) = NT$ when $T/2<\tul$.
    \item $M(T) = N\tul$ when $\tul\leq T/2 < N\tul/\qu$.
    \item $M(T)= \qu T$ when $ N\tul \leq \qu T/2 < \sqrt{\qi}$.
    \item $M(T) = \sqrt{\qi}$ when $\qu T/(8\log\qi) < \sqrt{\qi} \leq \qu T/2$.
    \item $M(T) = \qu T/(8\log\qi)$ when $\sqrt{\qi} \leq \qu T/(8\log\qi)<\qu \gamma_3= \sqrt{T\qi\qu}$.
    \item $M(T) = \sqrt{T\qi\qu}$ when $\gamma_3\leq T/(8\log\qi)$. 
\end{itemize}

In This case, 
\begin{align*}
    N\reg(T)
\geq 
C
\max\Big\{&N, \min\{NT, N\tul\}, \min\{\qu T, \sqrt{\qi}\}, \min\{\frac{\qu T}{\log\qi}, \\&\sqrt{T\qi\qu}\},\,  \min\big\{T\til, T\qu\big\}\Big\}\,.
\end{align*}

Overall, we can combine all three cases to get the following lower bound for regret:
\begin{align*}
    N\reg(T) \geq C
    \max\Big\{&N, \min\big\{NT, N\tul, \sqrt{\qi}\big\},\, \min\big\{\qu T, \sqrt{\qi}\big\},
    \\
   & \min\Big\{\frac{NT}{\log\qi},\sqrt{T\qi N}, N\tul\Big\},\,
    \min\big\{\frac{\qu T}{\log\qi}, \sqrt{T\qi\qu}\big\},\, 
    \\&\min\big\{T\til, T\qu\big\}\Big\}\,.
\end{align*}
\end{proof}

 \section{Optimizing Approximate Cost in Heuristic Analysis}
\label{s:OptAppCost}

In this section, we go over the computations to optimize regret in terms of the parameters of the high level description of the algorithm in Section~\ref{sec:AlgHybrid}. 
In Eq.~\eqref{e:totReg3}, we derive an approximate bound for regret according to the high level cost-benefit analysis of the sketch of the algorithm.
Plugging in Eq.~\eqref{e:k}  into~\eqref{e:totReg3} gives a one-dimensional optimization in terms of the algorithm parameter $\IR$. 
We detail this minimization process, though these calculations primarily serve as a demonstration and aren't critical for understanding the core algorithm. 
The significance of this step is showing how the minimization over parameter $\IR$ can result in various bound in different parameter regimes. The choice of optimal $\IR$ will dictate the choice of optimal parameters $\IE$ and $\IU$ as in Equations.~\eqref{e:IEstar} and~\eqref{e:IUstar} in the heuristic algorithm.

The detailed description of the pseudocode of algorithm in Sec.~\ref{sec:pseudocode} will be analyzed rigorously in Appendix~\ref{sec:performance} and the parameters in the pesudocode are chosen to minimize the exact bound on regret. 

Equation~\eqref{e:totReg3} as below gives the approximate regret in terms in $\IR$ and $\ell$. 
\begin{align*}
   2N\cdot\reg(T)\; &\approx\;
     \min\B\{N\tuu+\IR\qu 
+\f{2\qi\tiu}{\ell}\B(T -  \f\IR2 \B)_+,\, \IR N + \f{2\qi\tiu}{\ell}T\B\}
\\& := \min\B\{\tilde{f}(\IR), \tilde{g} (\IR)\B\}
\end{align*}
Hence, by proper choice of $\IRs$,
\begin{align*} 
2N\cdot\reg(T) &\approx  \min\B\{f(N, T, \qu, \qi), g(N, T, \qu, \qi)\B\}, 
      \\
      \text{where }
      f(N, T, \qu, \qi) &:= \min_{\IR} \tilde{f}(\IR) \quad \text{ and }  \quad
      g(N, T, \qu, \qi): =\min_{\IR} \tilde{g}(\IR)
\end{align*}

We use the approximation $\ell \approx \min\{\IR/2,\qi\}$ below to compute $f$ and $g$ separately:

\paragraph*{Computing $f(N, T, \qu, \qi)$}
\begin{align*}
    &\IR\qu  +\f{2\qi\tiu}{\ell}\B(T -  \f\IR2 \B)_+
    = \IR \qu + 2\tiu \max\B\{ \f{2\qi}{\IR},\, 1 \B\}  \max\B\{T -  \f\IR2,\, 0 \B\} 
    \\
      & = \max\B\{ \IR \qu + \f{4\qi \tiu}{\IR}T - 2\tiu \qi,\,\,\IR (\qu-\tiu) + 2\tiu T,\,\, \IR \qu \B\} 
\end{align*}
To find the minimum of the above display as a function of $\IR$, note that the first term is a convex function of $\IR$ with minimizer $2\sqrt{\qi\frac{\tiu}{\qu}T}$. The second term is linear in $\IR$ with slope $\qu-\tiu$, and the third term is increasing and linear in $\IR$. Also note that the three terms are equal to each other at $\IR=2T$. Also, the first and second terms intersect at $\IR=2\qi$.

Hence, when $\qu > \tiu$, the minimizer of the above display is $\IRs=2\min\{\sqrt{\qi\frac{\tiu}{\qu}T}, \qi, T\}$ and 
\begin{align*}
  & f(N, T, \qu, \qi)
    = N\tuu+
    \\
    &\quad
    \begin{cases}
    2\qu T, 
     \quad \quad \quad \quad \quad\quad \quad \text{ and } \quad \IRs= 2T \quad 
    &\text{if } T \leq \qi  \frac{\tiu}{\qu} 
    \\\noalign{\vskip2pt}
    4 \sqrt{\qu \qi \tiu T} - 2\qi \tiu, 
     \quad\text{ and } \quad \IRs= 2\sqrt{\qi\frac{\tiu}{\qu}T} \quad 
    &\text{if } \qi \frac{\tiu}\qu< T\leq \qi \frac{\qu}{\tiu}
    \\\noalign{\vskip2pt}
    2\qi (\qu-\tiu) + 2\tiu T,
     \quad\text{ and } \quad \IRs= 2\qi  \quad 
    &\text{if }  \qi \frac{\qu}{\tiu} < T 
    \end{cases}
\end{align*}

Hence, when $\qu \leq \tiu$, we have $ \min\{\qi, T\}<\sqrt{\qi\frac{\tiu}{\qu}T}$ and the minimizer of the above display occurs at
\begin{align*}
    \IRs= 2T 
    \quad  \quad \text{ and } \quad \quad 
    f(N, T, \qu, \qi)
      & = 
    N\tuu+2\qu T\,.
\end{align*}

So overall, 
\begin{align*}
   &2N f(N,T,\qu,\qi):=    N\tuu +
   \\ 
    & \begin{cases}
    2\qu T, 
     \quad \quad \quad \quad \quad\quad \quad \text{ and } \quad \IRs= 2T \quad 
    &\text{if } T \leq \qi  \frac{\tiu}{\qu} \text{ or } \qu\leq\tiu
    \\\noalign{\vskip2pt}
    4 \sqrt{\qu \qi \tiu T} - 2\qi \tiu, 
     \quad\text{ and } \quad \IRs= 2\sqrt{\qi\frac{\tiu}{\qu}T} \quad 
    &\text{if } \qi \frac{\tiu}\qu< T\leq \qi \frac{\qu}{\tiu}
    \\\noalign{\vskip2pt}
    2\qi (\qu-\tiu) + 2\tiu T,
     \quad\text{ and } \quad \IRs= 2\qi  \quad 
    &\text{if }  \qi \frac{\qu}{\tiu} < T \text{ and } \qu>\tiu
    \end{cases} 
\end{align*}

\paragraph*{Computing $g(N, T, \qu, \qi)$}
\begin{align*}
  \IR N + \f{2\qi\tiu}{\ell}T 
    & =\IR N + \max\B\{ \f{2\qi}{\IR},\, 1 \B\}  2\tiu T  \,.
\end{align*}
Notice that the right hand side is the maximum of an increasing linear function of $\IR$ and a convex function of $\IR$. The minimizer of this term is either the same as the minimizer of the convex term or the intersection of the linear and convex function, whichever occurs first. Hence,
\begin{align*}
     2N g(N,T,\qu,\qi)
    &= \begin{cases}
    4 \sqrt{N \qi \tiu T}, \quad \text{ and }\quad \IRs= 2\sqrt{\qi\f{\tiu}{N}T} \quad \quad &\text{if } T \leq \qi \frac{N}{\tiu}
    \\\noalign{\vskip2pt}
    2\qi N + 2\tiu T,  \quad\text{ and } \quad \IRs= 2\qi \quad \quad &\text{if } T >\qi \frac{N}{\tiu}\,.
    \end{cases}
\end{align*}

\section{\textsc{RecommendationSystem} performance analysis}
\label{sec:performance}

We now carry out the proof of Theorem~\ref{th:Item-upper} in three steps. First, in Section~\ref{s:errEvents} we will define events capturing error in clustering items or users. Then, in Section~\ref{s:dissecting} we will categorize the disliked recommendations according to three contributions:
\begin{enumerate}
    \item The total number of exploration recommendations made by the algorithm; 
    \item The number of bad recommendations in the exploitation phase due to an error in the exploration phase (including potential errors in item clustering or user clustering); 
    \item Recommendations made in case the set of exploitable items is too small. 
\end{enumerate}
Finally, in Section~\ref{sec:A1-A5} we will bound each of these. 


\subsection{Error Events}
\label{s:errEvents}

\paragraph*{User misclassification}
 Let the event that each user cluster contains users of only one type be denoted by
 \begin{align}
 \Bcal^c
	= 
\big\{\tau_U(u) = \tau_U(v)\text{ for all }u,v\in \Pc_w\text{ for each } \Pc_w\big\}\,.
\label{e:ju-bcal}
%
\end{align}

Note that if $\IU< \tuu$, then the user clustering is the trivial one consisting of one user per cluster and in this case \eqref{e:ju-bcal} automatically holds.

Because user feedback is noiseless and all users of a given type have identical preferences, user clustering errors are one-sided: assuming clustering is carried out at all, users of the same type always end up in the same cluster. Thus, $\Bcal$ fully captures the event of an error in user clustering. 

\begin{lemma}
\label{l:ju-user-mis-pr}
The probability of error in user clustering is $\Pr[\Bcal]\leq 1/N$.
 \end{lemma}

	The proof (in Appendix~\ref{sec:ProofErrUser}) shows that $\tuu$ is large enough so that in the case where nontrivial user clustering occurs, then for each pair of distinct user types with high probability there is at least one item in $\IsetU$ that distinguishes them. 
	

\paragraph*{Item misclassification}
 Algorithm \textsc{Explore} (Alg.~\ref{alg:joint-itemExplore}) makes use of sets of items $\S_\ti$, defined in Line~\ref{Line:defSj2} of \textsc{ItemClustering}.
 The set $\S_\ti$ includes the $\ti$-th representative item in $\IsetR$ and the set of items in $\IsetE$ that appear to be of the same type as this representative based on feedback obtained so far. 
Let $\EE_i$ be the event of misclassification of item $i\in\IsetE$:
		\begin{equation}\label{e:ju-EE}
		\EE_i = \big\{\text{there exists  }\, \ti \in \IsetR: i\in\S_\ti, \tau_I(i)\neq \tau_I(\ti)\big\}\,.
		\end{equation}
		For any item $i\in \IsetR$ we define $\EE_i^c$ to always hold, so $\EE_i$ is defined for all items in the set of items to be recommended in the exploitation phase, $\R_u\subseteq \IsetR\cup \IsetE$. 
Recommendations will be made based on the hope that $L_{u,i}=L_{u,{\ti}}$ for all $i\in\S_{\ti}$ and all $u\in[N]$,
so incorrectly clustering item $i$ can result in bad recommendations in \textsc{Exploit} (Alg.~\ref{alg:jointExploit}).
	 
The following lemma is a consequence of choosing $\tiu$, the number of positive comparisons needed to declare two items to be similar, to be sufficiently large.
	 \begin{lemma}\label{l:PrErrItemClass}
For any $i$ in $\IsetE$, the probability of item clustering error is
$\Pr[ \EE_i]
\leq {2}/{N}$.
\end{lemma}


\subsection{Sets $\R_u$ of Exploitable Items}
\label{s:exploitable}
Recall from Line~\ref{L:defRu} in \textsc{Explore} that \begin{equation}
    \R_u =
		 \bigcup_{j\in\IsetR:L_{\Pc_w,\ti}=+1} \S_j\text{ for each } u 
		\text{ such that } u\in \Pc_\tu\,,
\end{equation}
where in Line~\ref{Line:defSj2} of \textsc{ItemClustering} we defined for each $\ti\in \IsetR$
\begin{equation}
    \S_{\ti}= \{ \ti \}\cup \{i\in \IsetE: 
		L_{\Pc_\tu,i} = L_{\Pc_\tu,\ti} \text{ for all } \tu \text{ such that } \rated(i)\cap\Pc_\tu\neq \varnothing\}\,.
\end{equation}
In words, $\R_u$ is the union of item clusters with representatives $j\in\IsetR$ such that $u$'s user cluster likes $j$ as per \textsc{FindPrefs}.

\begin{lemma}\label{l:NoExploitErrors}
If users are clustered correctly and item $i\in \R_u$ was also clustered correctly, then item $i$ is liked by user $u$, i.e.
$$\Pr(L_{u,i}=-1| i\in\R_u, \EE^c_i,\BB^c) =0\,.$$
\end{lemma}
\begin{proof}
On $\BB^c$ any users in the same cluster have the same type and therefore have identical preferences. 
On
$\EE^c_i$ item $i$ is clustered correctly only to representatives $\ti$ of the same type as $i$. Next, $\R_u$ is the union of item clusters liked by users in $u$'s user cluster. 
Thus, $i\in \R_u$ only if a user of the same type as $u$ (which could be $u$ itself) has liked an item of the same type as $i$ (or $i$ itself).
\end{proof}

\begin{lemma}\label{l:RuLarge}
The sets of exploitable items are large enough, i.e., for each $u$ we have
$$\Pr\big( 
|\R_u|<  T
\big)
\leq
 {4}/{T}+ {2}/{N}\,.$$
\end{lemma}
The proof is deferred to Appendix~\ref{s:ManyExpItems}, and we now briefly convey the intuition. 
First, the set of item types in $\IsetR$ liked by user $u$, called $\Lc_u$, has Binomial size with expectation equal to half the number of item types appearing in $\IsetR$. Next, when $\BB^c$ holds, $\R_u$ includes those items in $\IsetR$ and $\IsetE$ with types in $\Lc_u$, and hence $\R_u$ is approximately $\mathsf{Bin}(\IR+\IE,|\Lc_u|/\qu)$. Thus, large enough $\IR$ and $\IE$ guarantee that $\R_u$ is large with high probability. 


\subsection{Dissecting the Regret}
\label{s:dissecting}

To bound regret (the expected number of bad recommendations), we partition the set of all bad recommendations made by the algorithm into several categories.
Denote by $\Tu$ the number of time-steps users spend in \textsc{Explore} and recall that $\R_u$ is the set of items produced by \textsc{Explore} to be recommended to user $u$. 
Since $\R_u\subseteq \IsetU \cup \IsetR$, the event $\EE_i$ is well-defined for all items in $\R_u$ and the total regret can then be bounded as
\begin{align}
N\reg(T) &\leq   \Exp{ \sum_{u=1}^N\sum_{t=1}^{\Tu} \ident\{L_{u,a_{u,t}}=-1\}}\notag\\
& \quad +   \Exp{ \sum_{u=1}^N\sum_{t=\Tu+1}^T  \ident\{L_{u,a_{u,t}}=-1, a_{u,t}\notin\R_u\}}\notag\\
& \quad +  \Exp{ \sum_{u=1}^N\sum_{t=\Tu+1}^T  \ident\{L_{u,a_{u,t}}=-1, a_{u,t}\in\R_u, \Bcal\}}\notag\\
& \quad +  \Exp{ \sum_{u=1}^N\sum_{t=\Tu+1}^T  \ident\{L_{u,a_{u,t}}=-1, a_{u,t}\in\R_u, \EE_{a_{u,t}}\}}\notag\\
& \quad +  \Exp{ \sum_{u=1}^N\sum_{t=\Tu+1}^T \ident\{L_{u,a_{u,t}}=-1, a_{u,t}\in\R_u, \EE_{a_{u,t}}^c, \Bcal^c\}} \notag\\
& =: \Asf{1} + \Asf{2} + \Asf{3} +\Asf{4} +\Asf{5} \,. \label{eq:ju-reg-decomp}
\end{align}
Here $\Asf{1}$ is the regret from the early time-steps up to $\Tu$ and account for the bad recommendations made by the algorithm during the exploration phase;
$\Asf{2}$ is the regret due to not having enough items available for the exploitation phase, which is proved to be small with high probability for a good choice of $\IU, \IR$ and $\IE$; 
$\Asf{3}$ is the regret due to error in user clustering ($\Bcal$ is defined in~\eqref{e:ju-bcal}), which is small due to Lemma~\ref{l:ju-user-mis-pr};
$\Asf{4}$ is the regret due to exploiting the misclassified items ($\EE_i$ is defined in~\eqref{e:ju-EE}), small due to Lemma~\ref{l:PrErrItemClass}; 
finally, $\Asf{5}$  is the regret due to exploiting correctly classified items under the correct user clustering, for which it is intuitively clear and will be verified later that $\Asf{5}=0$.


\subsection{Bounding $\Asf{1}$-$\Asf{5}$}
\label{sec:A1-A5}


\paragraph*{Bounding $\Asf{1}$}
The random variable $\Tu$ is the number of time-steps the users spend in the exploration phase consisting of \textsc{UserClustering}, \textsc{FindPrefs} and \textsc{ItemClustering}. 

In \textsc{UserClustering}, all of the $\IU$ items in $\IsetU$ are recommended to all users, taking $\IU$ time-steps.
In \textsc{FindPrefs} each representative item is recommended to one user per cluster, and user clusters that complete this task are recommended random items until the slowest (smallest cluster) finish. This takes $\lceil \IR (\min_w |\Pc_w|)^{-1}\rceil$ time-steps. 
\textsc{ItemClustering} takes  $\lceil{\tiu\IE}{N}^{-1}\rceil$ time-steps to rate items in $\IsetE$ by $\tiu$ users each, since $N$ users give feedback per time-step. This is further explained in Remark~\ref{r:usersToItems}.

Adding these quantities yields 
\begin{align}
\label{eq:Tzero}
    \Tu = \IU + \Big\lceil\frac{\IR}{\min |\Pc_w|}\Big\rceil  + \Big\lceil\frac{\tiu \IE}{N}\Big\rceil
    \end{align}
and
\begin{align}
\Asf{1}  \leq N \Exp{\Tu} & \leq N\IU + \Exp{\frac{N}{\min |\Pc_w|}}\, \IR +\tiu\IE +2N\,.\notag
\\
& \leq N\IU  +\tiu\IE + 2N + 
\begin{cases}
(2\qu+1) \IR, &\text{ if } \IU \geq \tuu
\\
N\IR, &\text{ if } \IU < \tuu
\end{cases}
\label{eq:ju-bound-A1}
\end{align}
In the last step we used Lemma~\ref{l:minPw}, which bounds $\Exp{{N}/{\min |\Pc_w|}}$.


\paragraph*{Bounding $\Asf{2}$}

Time-steps $t>\Tu$ are allocated to \textsc{Exploit}. The recommendation $a_{u,t}$ to user $u$ at time $t$ is not an item in $\R_u$ only if $\R_u$ has been exhausted (Line~\ref{Line:ExploitRand} of {\sc Exploit}).
At time $T_0$, there remain as least $(|\R_u| - T_0)_+$ unrecommended items in $\R_u$, so in the $T-T_0$ remaining time-steps $a_{u,t}\notin \R_u$ occurs at most $(T-|\R_u|)_+$ times.
Thus,
\begin{align*}
\Asf{2}
& \leq  
  \Ex\bigg[\sum_{u\in [N]}  {\sum_{t=\Tu+1}^T \ident\{ a_{u,t}\notin \R_u\}}\bigg]
\leq 
\sum_{u\in [N]} \Exp{(T -|\R_{u} |)_+} \leq 4N + 2T\,,
\end{align*}
where the last inequality is a consequence of Lemma~\ref{l:RuLarge} with the proper choice of parameters $\IR$ and $\IE$ to be large enough. 


\paragraph*{Bounding $\Asf{3}$}
Term $\Asf{3}$ in~\eqref{eq:ju-reg-decomp} is the expected number of bad recommendations  made in \textsc{Exploit} as a result of misclassification of users. 
By Lemma \ref{l:ju-user-mis-pr}, $\Pr[\Bcal]\leq 1/N$, so
\begin{align}\label{eq:ju-bound-A3}
	\Asf{3}
	\leq
	  \Ex\bigg[{\sum_{u=1}^N \sum_{t=\Tu+1}^T \ident\{\BB\}}\bigg]\, {\leq} \, T\,.
\end{align}


\paragraph*{Bounding $\Asf{4}$}
Term $\Asf{4}$ in~\eqref{eq:ju-reg-decomp} is the expected number of bad recommendations  made in \textsc{Exploit} as a result of misclassifications of items in $\IsetE$. 

\begin{align}
	\Asf{4}
	& \leq
	\Ex{\sum_{u=1}^N \sum_{t=\Tu+1}^T \sum_{i\in\R_u} \ident\{a_{u,t}=i, \EE_i\}}
			 \notag\\
	& \overset{(a)}{=}  
	\,
	\Ex{\sum_{u=1}^N \sum_{t=\Tu+1}^T \sum_{i\in\IsetE} \ident\{a_{u,t}=i, \EE_i\}}
		 \notag\\
	& \overset{(b)}{\leq}  
\, N\mathbb{E}\sum_{i\in\IsetE} \ident\{\EE_{i}\}
\overset{(c)}{\leq}  
N\IE \frac{2}{N} =2 \IE
\,.\label{eq:ju-bound-A4}
\end{align}
Here (a) holds since by construction,  $\R_u\subseteq \IsetE \cup \IsetR$ for all $u$ and $\EE_i$ does not occur for $i\in \IsetR$ by definition.
(b) holds since 
 each item $i\in\IsetE$ is recommended at most $N$ times (at most once to each user); (c) is by Lemma~\ref{l:PrErrItemClass}, which states that $\Pr[\EE_i]\leq 2/N$ for $i\in\IsetE$. 
 

\paragraph*{Bounding $\Asf{5}$}
From Lemma~\ref{l:NoExploitErrors} we conclude that $\Asf{5} = 0$.


\subsection{Upper Bounding Regret}
We now combine all the bounds above and plug them into Eq.~\eqref{eq:ju-reg-decomp} to get
\begin{align}
N\reg(T) \leq 
 N\IU  +(\tiu+2)\IE + 6N + 3T +
\begin{cases}
(2\qu+1) \IR, &\text{ if } \IU \geq \tuu
\\
N\IR, &\text{ if } \IU < \tuu
\end{cases}
\end{align}

The choice of parameters in Section~\ref{sss:params} gives the statement of Theorem~\ref{th:Item-upper}. 
%

\paragraph*{System parameters $I^{\mathsf{NoItemClust}}$ }
Here $\reg(T)$ is increasing in $\IU$ where we require $\IU\geq \tuu$. Hence $\IUs=\tuu$.

\begin{align}
\label{eq:RegNoItemClust}
\reg(T) \leq 
 \tuu  + 6 + 3\frac{T}{N} +
12\frac{\qu}{N} T
\end{align}

 \paragraph*{System parameters $I^{\mathsf{ItemClust}}$}
\begin{align}
N\reg(T) &\leq 
6N + 3T +
N \tuu \ind{ \ell>\tuu } +
 (\tiu+2)\left\lceil 16\frac{\qi}{\ell}T \right\rceil
 \\&+
 \begin{cases}
N\lceil 3\ell \rceil, 
&\text{ if } \ell\leq \tuu\,,
\\
(2\qu+1)\lceil 3\ell \rceil, 
&\text{ if } \tuu<\ell\leq \frac{\qi}{3}\,,
\\
(2\qu+1)\lceil \qi \log(N\qi)\rceil,
& \text{ if } \ell >\frac{\qi}{3}
 \end{cases}
\end{align}
where we used $\qi/3>\tuu$.
 Note that in this regime, $\ell\geq 16\log T$ is required.  
 Also, note that the third and last terms are nondecreasing in $\ell$ and the fourth term is nonincreasing in $\ell$.
Optimizing the regret with respect to  parameter $\ell$ gives  the choice of $\ell$ in Eq.~\eqref{eq:ParamChoice2ell} and subsequently regret can be bounded as
\begin{align}
\label{eq:RegItemClust}
\reg(T) &\leq 
6 + 3\frac{\tiu}{N}+ \frac{3}{N}
\begin{cases} 
48\log T + 18\sqrt{\frac{\qi \tiu T}{N}} 
\quad  &\text{ if }\,\, T\leq \TupI\,,\\ \\
24\tuu + 18 \frac{\sqrt{\qi \qu\tiu T} }{N}
& \text{ if }  \TupI< T\leq \TupU \\\\
  2\qu\qi\log\qi +48 \frac{\tiu}{N}T
& \text{ if }  \max\{\TupI,\TupU\}<T \,.
\end{cases}
\end{align}

where 
\begin{align*}
\TupI &:= \max\{T: \kI = 16\log T+ 2\sqrt{\frac{\qi\tiu}{N}T} \leq \tuu\}  
  \\
\TupU &:= \max\{T: \kH = 8\tuu + 2\sqrt{\frac{\tiu \qi T}{\qu}} \leq \qi/3\}\,.
\end{align*}

The algorithm chooses the system parameters $I^{\mathsf{NoItemClust}}$ or $I^{\mathsf{ItemClust}}$ depending on the regime of parameters by comparing the predicted performance of each case in Eq.~\eqref{eq:RegNoItemClust} and~\eqref{eq:RegItemClust} and choosing the better one.

When $T/2<\min\{\Rtt_I(T), \Rtt(U)\}$, a trivial algorithm recommending random items to users regardless of the feedback received so far trivially achieves regret $T/2$. Note that in the regime, the $\recsys$ algorithm with either system parameters $I^{\mathsf{NoItemClust}}$ or  $I^{\mathsf{ItemClust}}$ is exactly doing that.

Sections~\ref{sec:UUparam} and~\ref{sec:IJparam} contain the conditions of optimality for each operating regime.


\subsection{Algorithm Guarantee}

The regret of  algorithm \recsys is bounded in Theorem~\ref{th:Item-upper}. We give the proof next.

\begin{thmn}[\ref{th:Item-upper}]
Define the parameters $\tuu$  and $\tiu$ as in Equation~\eqref{eq:defthrjuu}. Then the following bounds for regret of the $\recsys$ algorithm holds:
\[\reg(T)
\leq
 \min\Big\{ \frac{T}{2}, \,\,C \Rtt_U(T), C\Rtt_I(T)\Big\}\]
for a constant $C$ where the functions $\Rtt_U(T)$ and  $\Rtt_I(T)$ are defined as follows:

\begin{align}
\label{eq:RegUSer2}
\Rtt_U(T) =
\tuu + \frac{\qu}{N} T
\end{align}
and
\begin{align}
\label{eq:RegItem2}
\Rtt_I(T) &=
\begin{cases}
\vspace{.1in}
 \log T + \sqrt{\frac{\qi \tiu}{N}T},
& \text{ if }
T<\TupI
\\ 
\vspace{.1in}
 \tuu +   \frac{1}{N} \sqrt{\qi\qu \tiu T},&
\text{ if }
\TupI \leq T <\TupU
\\  
\tuu  + \frac{\qu}{N} \qi \log(N\qi) + \frac{\tiu}{N} T,&
\text{ if }
T_1\leq T
\,.
\end{cases}
\end{align}
where
\begin{align*}
\TupI&:= \max\{T: \kI  \leq \tuu\}  \qquad \qquad \qquad
\TupU := \max\{T: \kH  \leq \qi/3\}\,.
\end{align*}
\end{thmn}

\begin{figure}[h]
\includegraphics[width=8cm]{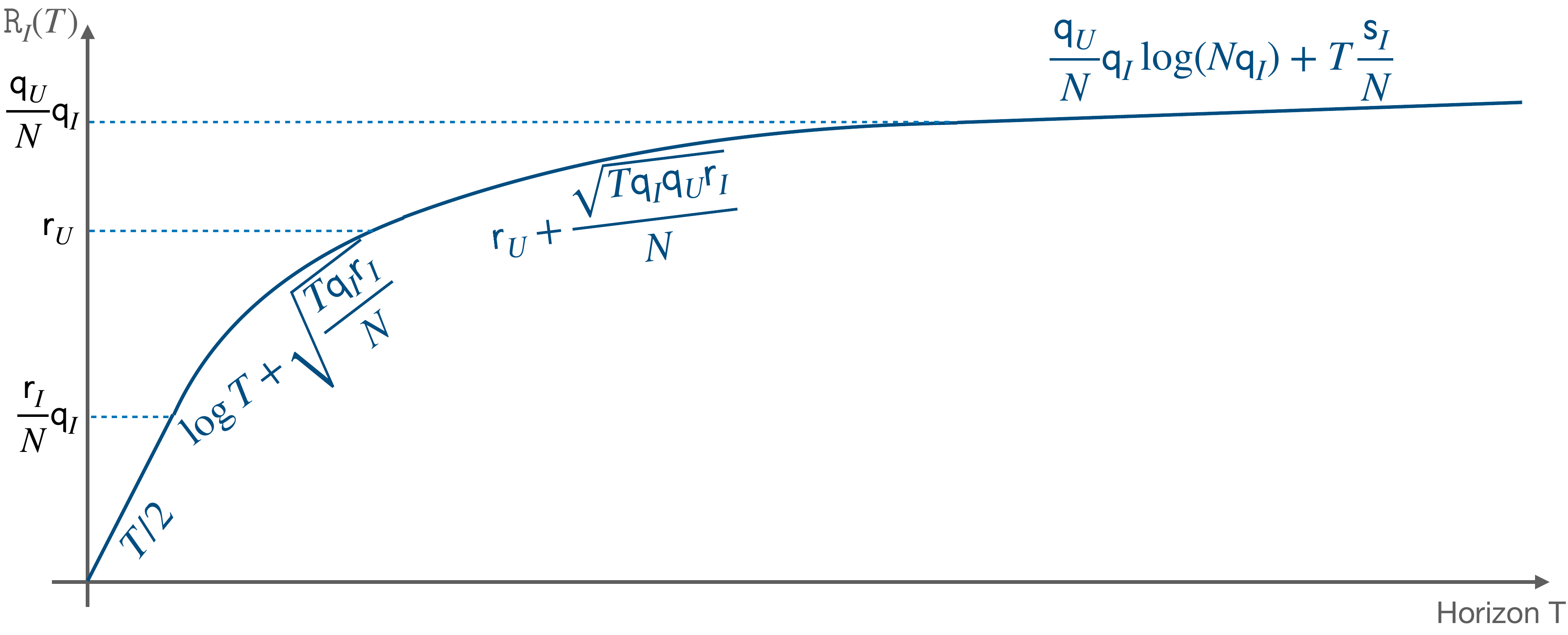}
\includegraphics[width=6cm]{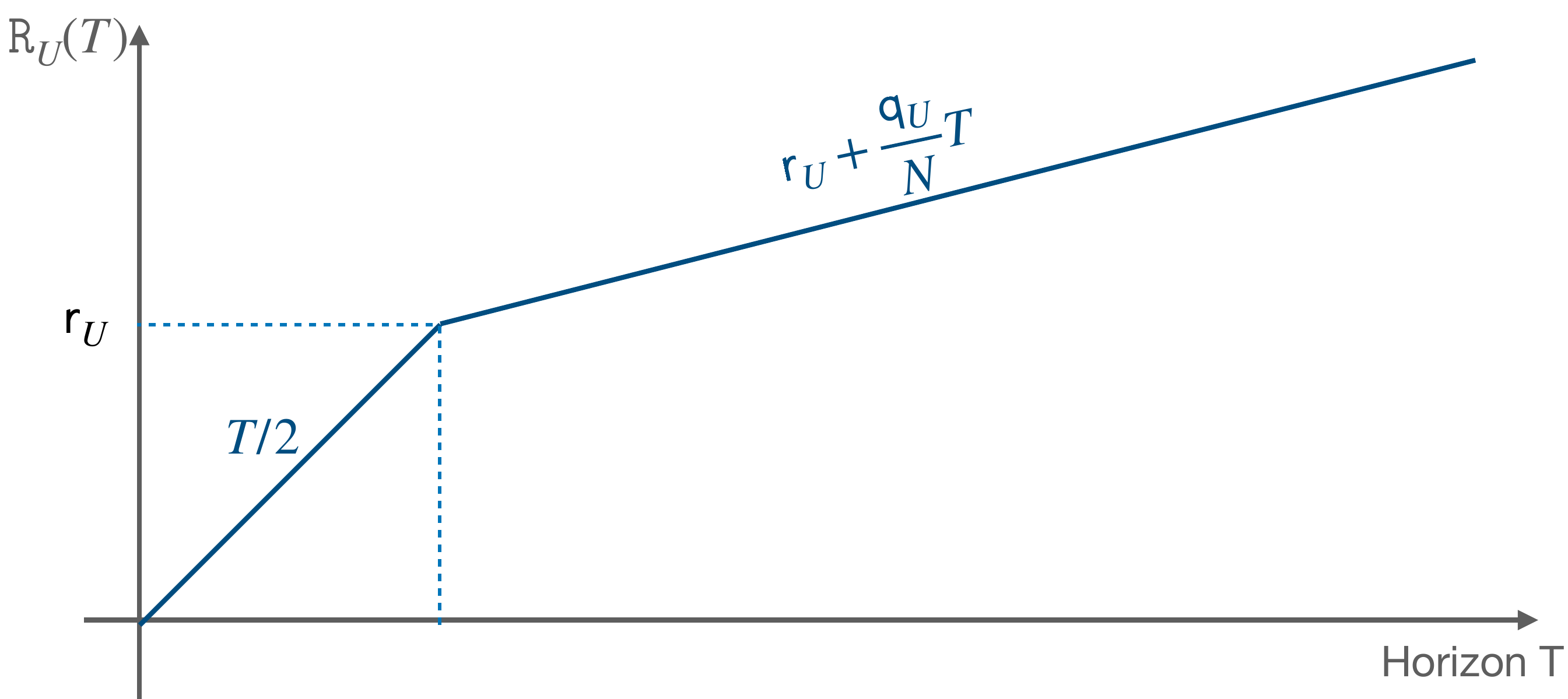}
\caption{The schematics for the functions $\Rtt_I(T)$ and $\Rtt_U(T)$.}
\centering
\end{figure}

\section{Proof of lemmas in the performance analysis of the algorithm} 
The lemmas in this section use the assumptions of the model, the choice of parameters of the algorithm $\tiu, \tuu$ (defined in Equations~\eqref{eq:defthrjuu}) $\IU, \IR$ and $\IE$ in the regimes $I^{\mathsf{NoItemClust}}$ and $I^{\mathsf{ItemClust}}$ (Eq.~\eqref{eq:ParamChoice1} and ~\eqref{eq:ParamChoice2} respectively). Lemma~\ref{l:boundsizeLu} also requires $\ell>16\log T$ in the $I^{\mathsf{ItemClust}}$ regime. 

We go through the details of the proofs of some lemmas used in analyzing the performance of the proposed algorithm in this section. 
We will use the following definition throughout.
\begin{definition}[Number of distinct types]
\label{def:distinct} For a set of items $\mathcal{I}$, let $\disi(\mathcal{I})$ be the number of distinct item types appearing in $\mathcal{I}$. Similarly, for a set of users $\mathcal{U}$, let $\disu(\Uc)$ be the number of distinct user types appearing in $\Uc$.
\end{definition}

When the set of users or items have independent and uniform types, simple balls and bins arguments in Section~\ref{sec:BallsBins} bound the quantities of $\disi(\mathcal{I})$ or $\disu(\Uc)$. 

\subsection{Bounding the Exploration Time $\Tu$}
Eq.~\eqref{eq:Tzero} determines the exploration time by 
\begin{align*}
    \Tu = \IU + \Big\lceil\frac{\IR}{\min |\Pc_w|}\Big\rceil  + \Big\lceil\frac{\tiu \IE}{N}\Big\rceil
    \end{align*}
Eq.~\eqref{eq:ju-bound-A1} bounds the term $\Asf{1}$ using the following lemma:

\begin{lemma}\label{l:minPw}
Suppose that $N > 20\qu \log^2 \qu$. Then 
$$\Exp{{N}/{\min |\Pc_w|}}\leq
\begin{cases}
(2\qu+1), &\text{ if } \IU \geq \tuu
\\
N, &\text{ if } \IU < \tuu\,.
\end{cases}
$$
\end{lemma}
\begin{proof}
When $\IU< \tuu$, we use the bound $\min_w |\Pc_w|\geq 1$. 
When $\IU\geq \tuu$, we upper bound $\min_w |\Pc_w|$ by the number of users in the smallest user type since all users of the same type always end up in the same cluster. A priori, the type of users are independent uniform on $[\qu]$. Hence, an application of Chernoff and union bound gives
\[\Pr\b( \min_w |\Pc_w|\geq {N}/{2\qu} \b)\geq 1- \qu \exp(-N/8\qu) \geq 1-1/N\,.\]
The last inequality uses $N>100$ and $N>20\qu \log^2\qu$ to get 
\[\frac{N}{\log N}\geq \frac{20\qu \log^2\qu}{\log\b(20\qu \log^2\qu\b)}\geq 8 \qu \log \qu\,,\]
Again, the last inequality uses $\qu\geq 200$. 
Hence, $N\geq 8 \qu (\log \qu)(\log N) \geq 8\qu \log(N\qu)$ and $\qu \exp(-N/8\qu) \leq 1/N$.

We also bound $\min |\Pc_w|$ trivially by $1$ to get the statement of the lemma. 
\end{proof}

\subsection{Error in User Classification }
\label{sec:ProofErrUser}

 In Equation~\eqref{e:ju-bcal}, we defined  $	\Bcal^c$ to be  the event that each user cluster contains users of only one type. Since user clustering errors are one-sided and users of the same type always end up in the same cluster, $\Bcal$ fully captures the event of an error in user clustering. 



\begin{lemman}[\ref{l:ju-user-mis-pr}]
The probability of error in user clustering is $\Pr[\Bcal]\leq 1/N$.
\end{lemman}

\begin{proof}
When $\IU < \tuu$ the clustering over the users is the trivial one with one user per cluster and the event $\Bcal^c$ holds.
Hence we only bound $\Pr[\Bcal]$ when $\IU \geq \tuu$.

 

We start by bounding $\disi(\IsetU)$, the number  of item types appearing in $\IsetU$. 
The types of items in $\IsetU$ are i.i.d. $\mathrm{Unif}([\qi])$.
The assumptions $\qu\leq N$ and $\qi> 100 \log N$ gives 
$ \tuu = \lceil 2\log (N\qu^2)\rceil\leq \qi/4$.
Therefore, when $\IU\geq \tuu$ an application of the balls and bins lemma (Lemma~\ref{l:ballsbins}) gives 
$$\Pr[\disi(\IsetU)< \tuu/2 ]\leq\exp(-\tuu/2)\,.$$

Let $\Bcal_{\tu,\tu'}$ be the event that distinct user types $\tu$ and $\tu'$ are clustered together.
Note that $\Bcal_{\tu,\tu'}^c$ occurs precisely when there is at least one item in $\IU$ that is rated differently by these user types. 
The elements of the preference matrix $\Xi$ are i.i.d. $\mathrm{Unif}(\{-1,+1\})$ independent of items in $\IsetU$ and their types.
 Hence, conditioning on the choice of item types in $\IsetU$, the probability 
 of $\Bcal_{\tu,\tu'}$
 is $2^{-\disi(\IsetU)}$.
 It follows that
\begin{align*}
\Pr[\Bcal_{\tu,\tu'}]
&\leq 
\Pr\left[\Bcal_{\tu,\tu'}\,|\, \disi(\IsetU) \geq \tuu/2\right]
+
\Pr\left[\disi(\IsetU)<\tuu/2 \right]
\\
&\leq
2^{-\tuu/2}+
\exp(-\tuu/2) 
\leq  2^{-\tuu/2+1} \leq 2/(N\qu^2)\,,
\end{align*}
where the last inequality uses  $\tuu = \lceil 2\log(N\qu^2)\rceil$. A union bound over ${\qu \choose 2}$ distinct pairs of user types gives the statement of lemma.  \end{proof}

\subsection{Error in Item Classification}
\label{sec:ErrItem}
The event $\EE_i$, defined in~\eqref{e:ju-EE} occurs when the item $i\in \IsetE$ is misclassified into $\S_{\ti}$ for some $\ti\in\IsetR$ such that $\tau_I(i)\neq \tau_I(\ti)$. 
To bound $P[\EE_i]$ for any item $i$ in $\IsetE$, define $U_i=\rated_{\Tu}(i)$ for $\Tu$ 
where $\Tu$ (as
defined in Eq.~\eqref{eq:Tzero}) is the time required to finish {\sc Explore}. By time $\Tu$, items in $\IsetE$ are recommended only in Line~\ref{Line:exploreM1} of {\sc ItemClustering} where the feedback is used to compare them to representatives in $\IsetR$ in Line~\ref{Line:defSj2} of   {\sc ItemClustering}.

For item $i\in\IsetE$ and subset of users $U\subseteq [N]$, define \textit{the potential error event}   
\begin{equation}\label{e:ju-EEbu-def}
\EEb_{i,U} =\{\text{ there exists } \ti : \tau_I(\ti)\neq \tau_I(i) \text{ and } L_{u,i}= L_{u,j}\,, \text{ for all } u\in U\}
\end{equation} 
to be the event that the preferences of users in $U$ for item $i$ agree with their preferences for items with some item type other than $\tau_I(i)$. 
Remember that we define the event $\BB^c$ in~\eqref{e:ju-bcal} to be the event that the partitioning over the users match the partitioning induced by user types.

The following claim is immediate from the definitions.
\begin{claim}
\label{c:ItemErrInclusion}
For any $i$ in $\IsetE$,  the error event $\EE_i$ defined in~\eqref{e:ju-EE} happens only when either the potential error event $ \EEb_{i,U_i}$ defined in~\eqref{e:ju-EEbu-def} occurs or there is an error in partitioning the users:
$
\EE_i\subseteq \EEb_{i,U_i}\cup\BB\,.
$
\end{claim}


\begin{lemman}{\ref{l:PrErrItemClass}}
For any $i \in \IsetE$, the probability of clustering error is
$\Pr[ \EE_i]
\leq {2}/{N}$.
\end{lemman}

\begin{proof}
Note that in the $I^{\mathsf{NoItemClust}}$ regime specified in Section~\ref{sec:UUparam} there is no item clustering and  $\IE=0$.
So the lemma is proved for the $I^{\mathsf{ItemClust}}$ regime specified in Section~\ref{sec:IJparam} where $\IE>0$. Note that this regime requires $\qu\geq 9\log(N\qi)$. So we will use this assumption in the proof.

 Claim~\ref{c:ItemErrInclusion} and a union bound gives
\[\Pr[ \EE_i]
\leq
\Pr[\EEb_{i,U_i}]+ 
\Pr[ \BB]\,.
\]
Lemma~\ref{l:ju-user-mis-pr} shows that $\Pr[ \BB]\leq 1/N$.
We bound $\Pr[\EEb_{i,U_i}]$ next.

Remark~\ref{r:usersToItems} explains how the items in $\IsetE$ are recommended in Line~\ref{Line:exploreM1} of {\sc ItemClustering}  to guarantee that the set of users assigned to each item is uniform at random. Hence, the set of users in $U_i$ is uniformly distributed over all subsets of size $|U_i| = \tiu$ of $[N]$ and the type of each user in $U_i$ is uniformly independently distributed over $[\qu]$.

Since $\qu\geq 9\log(N\qi)$, we have  $\tiu= \lceil 2\log (N\qi)\rceil\leq \qu/4$. 
 Since $\tiu\leq \qu/4$,  a balls and bins argument as in Lemma~\ref{l:ballsbins} upper bounds the number of user types aooearing in $U_i$ for any $i$
\[\Pr[ \disu(U_i) \leq \tiu/2]
\leq 
\exp(-\tiu/2)\,.
\]

Since the elements of matrix $\Xi$ are independently uniformly distributed on $\{-1,+1\}$ and $\Xi$ is independent of the user types in $U_i$, for any  item $\ti$  such that $\tau_I(\ti)\neq \tau_I(i)$ we have 
\[
\Pr\big[ L_{u,i}= L_{u,j} \text{ for all } u\in U_i
\,\big|\, 
\disu(U_i) > s
\big] \leq 2^{-s}\,.\]
Hence, a union bound over $\qi-1$ other item types and definition of event $\EEb_{i,U_i}$ in~\eqref{e:ju-EEbu-def} gives
\begin{align*}
\Pr\big[\, \EEb_{i,U_i} 
\,\big|\, 
\disu(U_i) >s \,\big] 
&\leq 
(\qi-1)  2^{-s}\,,
\end{align*}
which implies
\begin{align*}
\Pr[ \EEb_{i,U_i} ] 
&\leq 
\Pr[ \EEb_{i,U_i} 
\,\big|\, 
\disu(U_i)>\tiu/2]  +
\Pr[ \disu(U_i) \leq \tiu/2]
\\
&\leq 
(\qi-1)  2^{-\tiu/2} + e^{-\tiu/2} 
\leq \qi 2^{-\tiu/2} \,.
\end{align*}
Then, application of Claim~\ref{c:ItemErrInclusion}, Lemma~\ref{l:ju-user-mis-pr} and the above inequality gives
\begin{align*}
\Pr[ \EE_i]
\leq
\qi 2^{-\tiu/2}  + \frac{1}{N} \leq \frac{2}{N}\,,
\end{align*}
where we used $\tiu = \lceil 2\log (N\qi)\rceil$. Note that throughout the paper, all the logarithms are to the base of 2.
\end{proof}


\subsection{Obtaining Sufficiently Many Exploitable Items}
\label{s:ManyExpItems}
For any $u\in [N]$, $\R_u$ is the set of items exploitable by user $u$ in the algorithm {\sc Exploit} formed in Line~\ref{L:defRu} of {\sc Explore}.

\begin{lemma}\label{l:RuBig}
For any user $u\in[N],$
\[\Exp{(T- |\R_u |)_+}\leq 4+ \frac{2T}{N}\,.
\]
\end{lemma}

We consider two cases depending on the regime of system parameters $I^{\mathsf{NoItemClust}}$ and $I^{\mathsf{ItemClust}}$.
\subsubsection{Proof of Lemma~\ref{l:RuBig} with system parameters $I^{\mathsf{NoItemClust}}$.}
The parameters are $\IE=0$, $\IU=\tuu$, and $\IR = 6T$ as indicated in \eqref{eq:ParamChoice1}. 
Here the users are clustered, but structure in item space is ignored.
Hence $\R_u$ in this regime consists of the items in $\IsetR$ which are liked by some user in the same cluster as $u$:
\[\R_u = \{i\in\IsetR:  L_{\Pc_\tu,i} = +1, u\in\Pc_{\tu}\}\,.
\]

To bound $\R_u$, roughly speaking, we first argue that at least $\qi/3$ item types among $[\qi]$ types are liked by user $u$. Then using $\IR = 6T$ we show that with high probability $T$ items in $\IsetR$ are among the types liked by user $u$.

Define $\Itypes^{u+}$ to be the set of item types liked by user $u$:
\begin{align}\label{eq:defLcuq}
\Itypes^{u+}
&:=
 \big\{ j \in [\qi]:  \xi_{\tau_U(u),j} = +1  \big\}\,.
\end{align}
Since each item type is liked by user $u$ independently with probability $1/2$, an application of Chernoff bound in Lemma~\ref{l:Chernoff} gives 
\[
\Pr \b[ |\Itypes^{u+}|< \qi/3 \b] \leq \exp(-\qi/36)\,.
\]

 On event $\BB^c$ that users are clustered correctly, the set $\R_u$ is the set of items in $\IsetR$ liked by user $u$.
 \begin{equation*}
\text{If } \BB^c \text{ holds }, \quad \quad
 \R_u =  \{i\in\IsetR:  L_{u,i} = +1\}
 \end{equation*}
 So
\begin{equation}
\text{If } \BB^c \text{ holds and condition on given set } \Itypes^{u+}, \quad \quad
 \R_u =  \{i\in\IsetR:  \tau_I(i) \in \Itypes^{u+}\}\,. 
 \label{eq:RuCase1}
 \end{equation}

The set $\Itypes^{u+}$ depends on the preference matrix $\Xi$. The event $\BB^c$ depends on the preference matrix $\Xi$ and the type of items in $\IsetU$. Also, the type of items in $\IsetR$ are uniformly distributed on $[\qi]$ and independent of $\Xi$ and item types in $\IsetU$. So Eq.~\eqref{eq:RuCase1} implies that conditioning on event $\BB^c$ and set $\Itypes^{u+}$, the random variable $|\R_u|$ is an independent $\bin(\IR, |Lc_u|/\qi)$ random variable. Using Chernoff bound, 
\begin{align*}
 \Pr\B[  |\R_u|< T  \B| \Itypes^{u+} ,  |\Itypes^{u+}|\geq \qi/3, \BB^c\B] 
& = \Pr\B[  \bin(\IR,  |\Itypes^{u+}|/\qi)< T  \B| \Itypes^{u+} ,  |\Itypes^{u+}|\geq \qi/3, \BB^c,  \B] 
 \\& \leq 
  \Pr\B[  \bin(\IR,  1/3)< T  \B] 
 \leq \exp(-T/4)
\end{align*}
where we used the choice $\IR=6T$ in the last line. 
Using tower property of expectation over $\Itypes^{u+}$, this gives 
\[
 \Pr\B[  |\R_u|< T  \B|   |\Itypes^{u+}|\geq \qi/3, \BB^c\B] 
 \leq \exp(-T/4)
\]

Now to bound $|\R_u|$, 
\begin{align*}
\Pr\b[
|\R_u| <T 
\b]
&\leq 
\Pr\B[  |\R_u|< T  \B| |\Itypes^{u+}|\geq \qi/3, \BB^c \B] 
+
\Pr \b[ |\Itypes^{u+}|< \qi/3 \b] +
\Pr\b[
\BB\b]  
\\&\leq \exp(-T/4) +\exp(-\qi/36) + 1/N\leq \exp\b(- T/4\b) 
+ {2}/{N}\,,
\\
\text{and } \quad \Exp{(T- |\R_u |)_+}
& \leq 
T\exp( - T/4) +  {2T}/{N}
\leq
2  + {2T}/{N}
\,,
\end{align*} 
where we used $T\exp( - T/4)<2$ in the last inequality.

\subsubsection{Proof of Lemma~\ref{l:RuBig} with system parameters $I^{\mathsf{ItemClust}}$.}
The parameters are as indicated in \eqref{eq:ParamChoice2}:
\begin{align*} 
\IE & = \left\lceil 16\frac{\qi}{\ell}T\right\rceil
,\,
\IU =
 \tuu \ind{ \ell>\tuu } ,\,
 \IR
=
\begin{cases}
\lceil 3\ell \rceil, 
&\text{if } \ell\leq \frac{\qi}{3}
\\
\lceil \qi \log(N\qi)\rceil,
& \text{if } \ell >\frac{\qi}{3}
 \end{cases}\,.
\end{align*}

Here the items are clustered and the users are sometimes clustered. 
To bound $\R_u$, roughly speaking, we first argue that there are enough (at least $\ell$) item types in $\IsetR$ liked by user $u$. Then using the assignment of $\IE$, we show that with high probability $T$ items in $\IsetE$ are among the types seen in $\IsetR$ and liked by user $u$.

For given user $u$, let $\Lcrep$ be those item types learned by the algorithm through recommending $\IsetR$ which are liked by user $u$: 
\begin{align}\label{eq:defLcu}
\Lcrep
&:=
 \big\{ \tau_I(j): j\in \IsetR, L_{u,j} = +1  \big\}
\,.
\end{align}
According to Lemma~\ref{l:boundsizeLu}, which the choice of $\IR$ in~\eqref{eq:ParamChoice2}
\begin{equation*}
\Pr\big[ |\Lcrep|\leq \ell/8\big] \leq {2}/{T} + {1}/{N}\,.
\end{equation*}

The event  $\BB^c$  and the set $ \Lcrep$ are functions of the preference matrix $\Xi$, the types of items in $\IsetR$, $\IsetU$ and user types, all of which are independent of 
the types of items in $\IsetE$.
Hence, for given fixed set $ \Lcrep$ and condition on events,  $ \BB^c$  and $|\Lcrep|> \ell/8$,  the random variable $|\{i\in\IsetE: \tau_I(i) \in \Lcrep\} |$  is a $\bin(\IE,  |\Lcrep|/\qi)$. Using the choice of $\IE =\lceil 16 T \qi/\ell\rceil$ and an application of Chernoff bound gives:
\begin{align}
\Pr\B[ & |\{i\in\IsetE: \tau_I(i) \in \Lcrep\} |< T  \,\B|\, \Lcrep,  |\Lcrep|>\ell/8,\,\,  \BB^c \B]  \notag
\\&=
\Pr\B[  \bin(\IE,  |\Lcrep|/\qi)< T  \B| \Lcrep,  |\Lcrep|>\ell/8,\,\,  \BB^c \B]   \notag
\\&\leq 
\Pr\b[  \bin(\IE, \frac{\ell/8}{\qi})< T\b]
\leq
\exp( - T/4)\,.
\label{eq:sizeIexpLUrep}
\end{align}

To bound the size of set $\R_u$, note that
\begin{align}
\Pr&\b[|\R_u| <T \b]
\leq 
\Pr\b[ |\R_u|< T,\,\,  |\Lcrep|>\ell/8,\,\,  \BB^c \b]
+
\Pr\b[ |\Lcrep|\leq \ell/8 \b]
+
\Pr\b[ \BB\b]\,,
\label{eq:BoundRudecompose}
\end{align}
where the event $\BB^c$, defined in~\eqref{e:ju-bcal},
holds when each user cluster contains user of the same type. 

To bound the first term note that conditioned on $\BB^c$, the set of items $\R_u$ contains the items in $\IsetE$ with types in $\Lcrep$:  
\begin{align}
\text{If } \BB^c \text{ holds }, \quad \quad
 \R_u
  &\overset{(a)}{\supseteq}  \{i\in\IsetE:  L_{\Pc_\tu,\ti} = +1, u\in\Pc_\tu, i\in \S_j\}\notag
 \\
 &\overset{(b)}{=}
 \{i\in\IsetE:  L_{u,j} = +1, i\in \S_j\}\notag
 \\
  &\overset{(c)}{=}
 \{i\in\IsetE:  \tau_I(j)\in\Lcrep, i\in \S_j\}\notag
 \\
 &\overset{(d)}{\supseteq}  \{i\in\IsetE: \tau_I(i) \in \Lcrep\} 
 \label{eq:RuCase2}
 \end{align}
 where (a) uses the construction of $\R_u$ in Line~\ref{L:defRu} of {$\textsc{Explore}$} and Line~\ref{Line:defSj2} of {$\textsc{ItemClustering}$}.
 (b) holds condition on $\BB^c$, when $u\in\Pc_\tu$,  $ L_{\Pc_\tu,\ti} = +1$ if and only if $ L_{u,j} = +1$.
  (c) uses the definition of $\Lcrep$ in~\eqref{eq:defLcu}.
 Condition on $\BB^c$, the errors in item partitioning are also one-sided: the item $i\in\IsetE$ is added to $\S_j$ for any $j\in\IsetR$  if their ratings are the same for a subset of $\tiu$ random users (see Alg.~\ref{alg:ItemClustering} \textsc{ItemClustering}). Hence, if $\tau_I(i) = \tau_I(j)$ then $i\in\S_j$. Note that there could be misclassified items in $i\in\S_j$ such that $\tau_I(i) \neq \tau_I(j)$, hence the set inequality (d) can be strict. 
 
 Eq.~\eqref{eq:RuCase2} shows that $|\R_u|$ stochastically dominates $|\{i\in\IsetE: \tau_I(i) \in \Lcrep\} |$ which gives 
\begin{align*}
\Pr\b[ |\R_u|< T,\,\,  |\Lcrep|>\ell/8,\,\,  \BB^c \b]
&\leq \Pr\B[ |\{i\in\IsetE: \tau_I(i) \in \Lcrep\} |< T,\,\, |\Lcrep|>\ell/8,\,\,  \BB^c \B]  
\\
&\leq \Pr\B[ |\{i\in\IsetE: \tau_I(i) \in \Lcrep\} |< T  \,\B|\,  |\Lcrep|>\ell/8,\,\,  \BB^c \B]  
\\
&\leq \exp( - T/4) 
\end{align*} 
where the last inequality uses the tower property of expectation on the right hand side of Eq.\eqref{eq:sizeIexpLUrep}

 Plugging in the above display, Lemmas~\ref{l:boundsizeLu} and~\ref{l:ju-user-mis-pr} in Eq.~\eqref{eq:BoundRudecompose} gives
\begin{align*}
\Pr\Big[ 
|\R_u|<  T
\Big]
\leq
\exp( - T/4) + {2}/{T}+ {2}/{N}\,.
\end{align*}
Hence,
\begin{align*}
\Exp{(T- |\R_u |)_+}
& \leq 
T\exp( - T/4) + 2 + 2T/N
\leq
4+ 2T/N
\,,
\end{align*}
where we used $T\exp( - T/4)<2$ in the last inequality.

\begin{lemma}
\label{l:boundsizeLu}
Let $\Lcrep$ be defined in Equation~\eqref{eq:defLcu}. 
Under the choice of system parameters in $I^{\mathsf{ItemClust}}$, Eq.~\eqref{eq:ParamChoice2},
\begin{align*}
\Pr\big[|\Lcrep|< \ell/ 8\big] \leq {2}/{T}+{1}/{N}
\,.
\end{align*}
\end{lemma}

\begin{proof}
Let $\Itypes_{\mathrm{rep}}$ be the set of item types in $\IsetR$ learned by the algorithm after the exploration phase.
Hence,
\begin{align}
\label{e:defLc}
\Itypes_{\mathrm{rep}}
& := \big\{ \tau_I(j): j\in \IsetR  \big\}\,.
\end{align}
Note that, with this definition, $|\Itypes_{\mathrm{rep}}|= \disi(\IsetR)$.

Under the choice of system parameters in $I^{\mathsf{ItemClust}}$,
\begin{align}
    \IR = 
    \begin{cases}
  \lceil 3\ell \rceil, \quad  &\text{ if } \ell \leq \qi/3
    \\ 
  \big\lceil \qi \log(N \qi)\big\rceil, \quad  &\text{ if }\qi/3 <  \ell\leq \qi\,.
    \end{cases}
\end{align}
 applying balls and bins Lemmas~\ref{l:ballsbins2} and~\ref{l:ballsbins3} to the two above regimes gives $$\Pr[|\Itypes_{\mathrm{rep}}| < \ell] \leq \exp(-\ell) +1/N\,.$$ 

Each item type is liked by user $u$ with probability $1/2$ independently of whether they appear in $\Itypes_{\mathrm{rep}}$ or not. Hence, 
\begin{align*}
\Pr\big[|\Lcrep|< |\Itypes_{\mathrm{rep}}|/ 4 \, \big|\, |\Itypes_{\mathrm{rep}}| \big] \leq  \exp(- |\Itypes_{\mathrm{rep}}| / 16)
\,.
\end{align*}

Overall, 
\begin{align*}
\Pr\big[|\Lcrep|< \ell/ 4\big] &\leq
\Pr\big[|\Lcrep|< |\Itypes_{\mathrm{rep}}|/ 4 \, \big|\, |\Itypes_{\mathrm{rep}}|\geq \ell\big]  + \Pr[|\Itypes_{\mathrm{rep}}| < \ell]
\\&\leq 2\exp (-\ell/16) +\frac{1}{N}
\leq \frac{2}{T} + \frac{1}{N}
\,.
\end{align*}
where we used the choice of system parameters in $I^{\mathsf{ItemClust}}$  which gives  $\ell\geq 16\log T$.
\end{proof}


\section{Optimal Regret -- Proof of Theorem~\ref{thm:MainResult}}
\label{sec:ProofMainResult}
 The statement of the main result in Theorem~\ref{thm:MainResult} is a looser version of the lower bound on regret in Theorem~\ref{t:joint-L} and the upper bound of regret of $\recsys$ algorithm given in Theorem~\ref{th:Item-upper}. 

Subsequantly, to prove Theorem~\ref{thm:MainResult}, in Section~\ref{sec:mainResultProofUB}, we show that under the modeling assumptions in Figure~\ref{f:assumptions}, 
there are universal constants $c$ and $C$ such that for any time horizon $T$, algorithm {\recsys}  achieves $N\reg(T) \leq C\,N\mathtt{R}(T) \log^{3/2} (N\mathtt{R}(T))$.

This proof is primarily computational and doesn't offer significant conceptual insights. The reader can skip it without missing crucial information or essential takeaways.  We included it for the sake of thoroughness.

 In Section~\ref{sec:mainResultProofLB}, we show that if  $\qu> (\log\qi)^2$ and $\qi>(\log N)^5$  then any algorithm must incur $N\reg(T)\geq c\,\frac{N\mathtt{R}(T)}{\log (N\mathtt{R}(T))}$.

The function $\mathtt{R}(T)$ 
is defined in a piece-wise manner in the table below, with the support of each piece given in terms of the following functions of system parameters: 
\begin{align*}
    T_1 = \log\qu,  \quad
    T_2 = \qi /N,\quad
    T_3 = \qi / \qu,
    \quad
    T_4 = N(\log\qu)^2/\qi,\quad
    T_5 = \qi\qu \,.
\end{align*}

\renewcommand{\arraystretch}{1.6}
\begin{center}
\begin{tabular}[c]{ c | c | c | c}
 $\mathtt{R}(T)$ & Operating Regime &  Range of $T$ & Conditions For Occurring \\
\hline\hline
T &$\TintCold$ & $[1,\, \min\{T_1,T_2\} ]$ \\ 
\hline
$1+\sqrt{\frac{\qi T}{N}}$&$ \TintItem$ & $(T_2,\, T_4]$ & $T_2 \leq T_1$\\ 
\hline
$\log\qu +\frac{\qu}{N}T$ &$\TintUser$ & $(T_1, \, T_3]$ & $T_1<T_2$ and $\log\qi\leq \qu$ \\ 
 \hline
 $ \log \qu+\frac{\sqrt{\qi\qu T}}{N}$ 
& $\TintHyb$ &$(T_3,\, T_5]$ & $T_1<T_2$ and $\log\qi\leq \qu$ \\  
  &&$(T_4,\, T_5]$ & $T_2\leq T_1$ and $\log\qi\leq \qu$ \\  
 \hline
$\log\qu + \frac{\qi \qu}{N}+\frac{\log \qi}{N} T$  & & $(T_5,\, \infty) $ &  $\log\qi\leq \qu$ \\
$\log\qu + \frac{\qu}{N} T$ &  $\TintAsym$  & $ (T_1,\, \infty)$ & $T_1<T_2$ and  $\qu<\log\qi$\\
$\log\qu + \frac{\qu}{N} T $ &   & $(T_4,\, \infty) $ & $T_2\leq T_1$ and  $\qu<\log\qi$
\\ \hline
\end{tabular}
\end{center}


For the sake of simplicity, we will ignore all universal constant factors and denote them by $c$ or $C$ in this proof, although these constants may take different values in differet steps.


\subsection{Approximate Upper Bound on Regret}
\label{sec:mainResultProofUB}
Let $\Rtt_U(T)$ and $\Rtt_I(T)$ defined in Eq.~\eqref{eq:RegUSer} and~\eqref{eq:RegItem}. 
Theorem~\ref{th:Item-upper} states that  algorithm $\recsys$ achieves regret
$$ \reg(T)\leq  c \min\Big\{ \frac{T}{2}, \Rtt_U(T), \Rtt_I(T)\Big\} \,.$$

Next, we will show that in all regimes of parameterx $\qi, \qu, N$ and $T$
\begin{align}
\label{eq:desireBdUBMain}
    \min\Big\{ \frac{T}{2}, \Rtt_U(T), \Rtt_I(T)\Big\}  \leq   c\, \Rtt(T) \log^{3/2}(N \,\Rtt (T))\,,
\end{align}
to get
$$ \reg(T)\leq  C  \Rtt(T) \log^{3/2}(N \,\, \Rtt (T))\,.$$

Throughout, we will use the $2\log (N\qu)<\tuu <  4\log (N\qu)$ and $2\log(N\qi)<\tiu <3\log(N\qi) $. We will also use the assumptions in Figure~\ref{f:assumptions} including $N>100$.

\paragraph*{Cold start $T\in \TintCold$:}
In this regime, $\Rtt(T)=T$ and
\begin{align*}
    \frac{T}{2} <  \frac 1 2  \, \Rtt(T) \log(N \Rtt (T))\,.
\end{align*}
which gives~\eqref{eq:desireBdUBMain} for  $T\in \TintCold$. 
\paragraph*{Item-Item Regime $T\in \TintItem$}
For $T\in \TintItem= (T_2, T_4]$ we have \[\Rtt(T)=1+\sqrt{\frac{\qi T}{N}}\,.\]
In this regime, we show that $\Rtt_I(T)< c\, \Rtt(T) \log(NR(T))$ whish shows~\eqref{eq:desireBdUBMain}. Note that $\Rtt_I(T)$ takes three different forms.

\vspace{0.1in}
\underline{Case I:} $T\in (T_2, T_4]\cap (1,\TupI)$
\begin{align*}
    \Rtt_I(T) &= \log T + \sqrt{\frac{\qi\tiu}{N}T}
    \leq
    \log T + \sqrt{\frac{3\qi\log(N\qi)}{N}T}
    \\ &<  2\Big[1+ \sqrt{\qi T/N}\Big] \log (NT\qi)
   < \frac{1}{4} \Rtt(T) \log(NR(T))\,.
\end{align*}
where we used $\log (N\,\Rtt(T)) \geq 1/2\log (N\qi T)$.

\vspace{0.1in}
\underline{Case II:} $T\in (T_2, T_4]\cap [\TupI,\TupU)$. 
In this case,
\begin{align*}
\Rtt_I(T) &= \tuu +   \frac{1}{N} \sqrt{\qi\qu \tiu T} 
\end{align*}
For $T\leq T_4= N\frac{(\log \qu)^2}{\qi}$
\begin{align*}
\Rtt_I(T) &= \tuu +   \frac{1}{N} \sqrt{\qi\qu \tiu T} 
\leq 4\log (N\qu) + \sqrt{\qu\tiu /N} \leq  8\log (N\qu) \sqrt{ \log (N\qi)}\,.
\end{align*}
Using $T>  \TupI > N\tuu^2/(4\qi \tiu):$
\begin{align*}
\Rtt(T) \log(NR(T))> \sqrt{\frac{\qi T}{4N}} \log(NT\qi) 
 &> \frac{\tuu}{4\sqrt{\tiu}}  \log(NT\qi) 
 > \frac{1}{8}\log(N\qu) \sqrt{\log(N\qi)}.
\end{align*}
where we used $\log (N\,\Rtt(T)) \geq c\log (N\qi T)$. 

The above two displays show~\eqref{eq:desireBdUBMain} for $T\in (T_2, T_4]\cap [\TupI,\TupU)$.

\vspace{0.1in}
\underline{Case III:} $T\in (T_2, T_4]\cap [\max\{\TupU,\TupI\},\infty)$
\begin{align*}
\Rtt_I(T) &= \tuu  + \frac{\qu}{N} \qi \log(N\qi) + \frac{\tiu}{N} T
\\
\text{Since } T>\TupI> N\tuu^2/(4\qi \tiu) \text{ (as above)}:&
\\
 \sqrt{\frac{\qi T}{N}} \log(NT\qi) 
 &>  \frac{1}{8}\log(N\qu) \sqrt{\log(N\qi)} > \frac{1}{10}\tuu\,.
 \\
 \text{Using } T>\TupU >\frac{\qi \qu}{36\tiu}:\qquad\qquad\qquad&
 \\
 \sqrt{\frac{\qi T}{N}} \log(NT\qi) 
 &> \frac{\qi}{6} \sqrt{\frac{\qu}{N}} \log(N\qi) >  {\frac{\qi\qu}{6N}} \log(N\qi)
\end{align*}
Finally, $T\leq T_4 = N(\log\qu)^2/\qi < N\qi$ (since $\log\qu < \log N <\qi$). So, 
\[
\quad\frac{\tiu}{N} T < \sqrt{\frac{\qi T}{N}} \log(NT\qi)\,.\]
The above three displays show~\eqref{eq:desireBdUBMain} for $T\in(T_2, T_4]\cap [\max\{\TupU,\TupI\},\infty)$.

\paragraph*{User-User regime $T\in \TintUser$}
For $T\in (T_1, T_3]$ we have \[\Rtt(T)=\log\qu +\frac{\qu}{N}T\,.\]
In this case, to get~\eqref{eq:desireBdUBMain}, we will show that $\Rtt_U(T) = \tuu + \frac{\qu}{N}T\leq C \Rtt(T) \log(NR(T))$. 
\begin{align*}
    \Rtt_U(T) &= \tuu + \frac{\qu}{N}T
    < 4\log(N\qu) + \frac{\qu}{N}T
    \\
    &< 4\Big(\log\qu + \frac{\qu}{N}T\Big) \log(N)
   <4 \Rtt(T) \log(NR(T))
    \,.
\end{align*}

\paragraph*{Hybrid regime $T\in \TintHyb$:}
For $T\in \TintHyb$ we have 
\[\Rtt(T)=\log\qu +\frac{\sqrt{\qi \qu T}}{N}\,.\]
In this case, to get~\eqref{eq:desireBdUBMain}, we will show that  $\Rtt_I(T)\leq C \Rtt(T) \log^{3/2}(NR(T))$ 
and will use 
$$\log^{3/2}(N\, \Rtt(T)) \geq \log^{3/2}(N + \sqrt{\qi\qu T})
> \frac{1}{4}\max\big\{\log^{3/2} N, \log^{3/2} \qi\big\}\,.$$
Note that $\Rtt_I(T)$ takes three different forms.

\underline{Case I:} ($T\in \TintHyb \cap (1,\TupI)$)
When $T<\TupI$, 
\[\Rtt_I(T) \leq C\tuu < 4C\log (N) <  c\, \Rtt(T) \log^{3/2}(NR(T))\,.\]


\underline{Case II:} ($T\in \TintHyb \cap [\TupI, \TupU)$)
In this regime, 
\[\Rtt_I(T) =  \tuu +   \frac{1}{N} \sqrt{\qi\qu \tiu T} 
\,.\]
As above, $\tuu <  C \Rtt(T) \log^{3/2}(NR(T))$.
Also, $$\sqrt{\tiu}< C\sqrt{\log(N\qi)} <  C'\max\{\log(N), \log(\qi)\}\,.$$
So,
\begin{align*}
    \frac{1}{N} \sqrt{\qi\qu \tiu T} & = C\frac{\sqrt{\qi \qu T}}{N} \max\{\log(N), \log(\qi)\} <
   C \frac{\sqrt{\qi \qu T}}{N} \log^{3/2}(N\Rtt(T))\,.
\end{align*}


\underline{Case III:} $T\in \TintHyb\cap [\max\{\TupU,\TupI\},\infty)$
\begin{align*}
\Rtt_I(T) &= \tuu  + \frac{\qu}{N} \qi \log(N\qi) + \frac{\tiu}{N} T
\end{align*}

As above, $\tuu <  C \Rtt(T) \log^{3/2}(NR(T))$.

For $ T>\TupU$, we have $\tuu + \sqrt{\frac{\tiu \qi T}{\qu}} > \qi/24$. Using $\qi> 100\log N > 30 \tuu $ we get $\sqrt{\qi T /\qu} > c \frac{\qi}{\sqrt{\log(N\qi)}}$. Hence, 
\begin{align*}
\Rtt(T) \log^{3/2}(NR(T)) &>C \frac{\sqrt{\qi \qu T}}{N} \max\{\log^{3/2} N, \log^{3/2} \qi\}
\\&> C \frac{\qu\qi}{N\sqrt{\log(N\qi)}} \max\{\log^{3/2} N, \log^{3/2} \qi\}
\\&>
C\frac{\qu}{N} \qi \log(N\qi)
\end{align*}

Finally, $T\leq T_5 = \qi \qu$. So, 
\[\frac{\tiu}{N} \sqrt{T} < \frac{\sqrt{\qi\qu}}{N} \tiu<  \frac{\sqrt{\qi\qu}}{N}\max\{\log N, \log\qi\}\]
Hence, 
\[ \frac{\tiu}{N} T < C\frac{\sqrt{\qi \qu T}}{N} \max\{\log(N), \log(\qi)\} < C\, \Rtt(T) \log^{3/2}(NR(T))\,.\]

Putting it all together gives~\eqref{eq:desireBdUBMain} for $T\in \TintHyb\cap [\max\{\TupU,\TupI\},\infty)$.

\paragraph*{Asymptotic regime $T\in \TintAsym$ and $\log\qi \leq \qu$}
In this regime, we have
\[\Rtt(T) = \log\qu + \frac{\qi \qu}{N}+\frac{\log \qi}{N} T\]

We will show that $\Rtt_I(T)\leq C \Rtt(T) \log(NR(T))$ and use 
$$\log(N\, \Rtt(T)) > \max\{\log N, \log (\qu\qi), \log T\}.$$
Note that $\Rtt_I(T)$ takes three different forms.

\underline{Case I:} ($T\in \TintAsym \cap (1,\TupI)$)
When $T<\TupI$, 
\[\Rtt_I(T) \leq \tuu <
8\log N <  8\, \Rtt(T) \log(NR(T))\,.\]


\underline{Case II:} ($T\in \TintHyb \cap [\TupI, \TupU)$)
In this regime, 
\[\Rtt_I(T) =  \tuu +   \frac{1}{N} \sqrt{\qi\qu \tiu T} 
\,.\]
As above, $\tuu <  8 \Rtt(T) \log(NR(T))$.
Also, 
$$ \frac{1}{N} \sqrt{\qi\qu \tiu T} < \frac{1}{2N}(\qi\qu + T\tiu)< \frac{\qi\qu}{2 N} + 2 \frac{\log\qi}{N} T \log(N) $$
Hence, $\Rtt_I(T)\leq 8 \Rtt(T) \log(NR(T))$ for $T\in \TintHyb \cap [\TupI, \TupU)$.

\underline{Case III:} $T\in \TintAsym\cap [\max\{\TupU,\TupI\},\infty)$
\begin{align*}
\Rtt_I(T) &= \tuu  + \frac{\qu}{N} \qi \log(N\qi) + \frac{\tiu}{N} T
\end{align*}
Note that
\begin{align*}
\tuu &<  8 \Rtt(T) \log(NR(T))
\\
\frac{\qu}{N} \qi \log(N\qi) &\leq \frac{\qi\qu}{N} \max\{\log N, \log \qi \}<  \Rtt(T) \log(NR(T))
\\
\frac{\tiu}{N} T & < 3\frac{\log\qi}{N} T \log(N) < 3 \Rtt(T) \log(NR(T))
\end{align*}
Hence, $\Rtt_I(T)\leq 8 \Rtt(T) \log(NR(T))$ for $T\in \TintAsym\cap [\max\{\TupU,\TupI\},\infty)$.
\paragraph*{Asymptotic regime $T\in \TintAsym$ and $\log\qi > \qu$ :}
In this regime, we have
\[\Rtt(T) = \log\qu + \frac{\qu}{N} T\]
Using the proof in the User-User regime, $T\in \TintUser$, we get  $\Rtt_U(T)\leq C \Rtt(T) \log(NR(T))$.


 \subsection{Approximate Lower Bound on Regret}
\label{sec:mainResultProofLB}

Theorem~\ref{t:joint-L} states that  any recommendation system must incur the regret lower bounded as follows.
\begin{align*}
   \reg(T)\geq  c \max\Big\{&N, \min\{NT, N\log \qu, \sqrt{\qi}\}, \min\{\qu T, \sqrt{\qi}\}, \notag
    \\
   & \min\Big\{\frac{NT}{\log\qi},\sqrt{T\qi N}, N\log\qu\Big\},
    \min\{\frac{\qu T}{\log\qi}, \sqrt{T\qi\qu}\},\, \notag
    \\&\min\big\{T\log\qi, T\qu\big\}\Big\}
\end{align*}

In this section, we will show that in all regimes of parameterx $\qi, \qu, N$ and $T$
\begin{align}
\label{eq:LBRegretApproxtl}
   \frac{ N\Rtt(T)}{\log(N \,\Rtt (T))} \leq c \max\Big\{&N, \min\{NT, N\log \qu, \sqrt{\qi}\}, \min\{\qu T, \sqrt{\qi}\}, \notag
    \\
   & \min\Big\{\frac{NT}{\log\qi},\sqrt{T\qi N}, N\log\qu\Big\},
    \min\{\frac{\qu T}{\log\qi}, \sqrt{T\qi\qu}\},\, \notag
    \\&\min\big\{T\log\qi, T\qu\big\}\Big\}
 \,.
\end{align}
Note that to show this, it suffices to prove that the LHS is upper bounded by one or two terms in the max. 
This proves that for any recommendation system,  
\begin{align}
\label{eq:LBRegretRT}
    N\reg(T) \geq c \frac{ N\Rtt(T)}{\log(N \,\Rtt (T))}\,.
\end{align}

To prove Eq.~\eqref{eq:LBRegretApproxtl}, we will also use the assumptions in Figure~\ref{f:assumptions}. 
Additionally, we will use use the assumption $\qu> (\log\qi)^2$ and $\qi>(\log N)^5$  and definitions of $\tul$ and $\til$ to get $c\log\qu<\tul\leq C\log \qu$ and  and $c\log\qi < \tiu < C\log \qi$.
 For the sake of simplicity, we will ignore all universal constant factors denoting them by $c$.
 
We go over each regime in the statement of Theorem~\ref{thm:MainResult} separately to show Eq.~\eqref{eq:LBRegretApproxtl} for various forms that $\Rtt(T)$ can take. 
\paragraph*{Cold start $T\in \TintCold$:}
When $T\in [1, \min\{T_1,T_2\}]$ with $T_1=\log\qu$ and $T_2=\qi/N$ we have $\Rtt(T)=T$.

For $T< \min\{\log\qu, \sqrt{\qi}/N\}$ we have 
$$  \min\{NT, N\log \qu, \sqrt{\qi}\} = NT = N\Rtt(T)\,.$$
If $T_1 = \log\qu \leq \sqrt{\qi}/N <\qi/N =T_2$, we are done with proving~\eqref{eq:LBRegretApproxtl} for the cold start regime. 

Otherwise, if $ \sqrt{\qi}/N \leq T< \min\{\log\qu, \qi/N\}$, we have  $\log(N \,\Rtt(T)) \geq 1/2\log\qi$ and 
$$ \min\Big\{\frac{NT}{\log\qi},\sqrt{T\qi N}, N\log\qu\Big\} 
= \frac{NT}{\log\qi} 
\geq \frac{ N\Rtt(T)}{2\log(N \,\Rtt (T))} \,.$$

\paragraph*{Item-Item Regime $T\in \TintItem$:}

For $T\in  \TintItem = (T_2, T_4]$ we have \[\Rtt(T)=1+\sqrt{\frac{\qi T}{N}}\,.\]
In this regime, 
\begin{align*}
    \frac{NT}{\log\qi} &> \frac{\sqrt{T\qi N}}{\log\qi} \qquad \text{ since } \quad T> T_2 = \qi/N  \text{ and, }\\
    N\log\qu & \geq \sqrt{T\qi N} \qquad \text{ since } \quad T\leq  T_4 = N(\log \qu)^2/\qi\,.
\end{align*}
Using these give
 \begin{align*}
    \max\Big\{N,& \min\Big\{\frac{NT}{\log\qi},\sqrt{T\qi N}, N\log\qu\Big\} \Big\}
     \geq  \max\Big\{N, \frac{\sqrt{ T\qi N}}{\log \qi} \Big\} 
    \\&>  \frac{N+ \sqrt{T \qi N}}{2\log \qi} >  \frac{ N\Rtt(T)}{2\log(N \,\Rtt (T))}\,.
 \end{align*}
 which proves~\eqref{eq:LBRegretApproxtl} for  $T\in \TintItem$.
\paragraph*{User-User regime $T\in \TintUser$:}
When $T_1< T_2,$ for $T\in (T_1, T_3]$ we have \[\Rtt(T)=\log\qu +\frac{\qu}{N}T\,.\]
In this regime, since $T<T_3 = \qi/\qu$, we have $\qu T/\log\qi < \sqrt{T\qi\qu}$.
 \begin{align*}
    \max\Big\{&\min\{\qu T, \sqrt{\qi}\}\,, 
    \min\{\frac{\qu T}{\log\qi}, \sqrt{T\qi\qu}\} \Big\}
    \\
    & =
    \max\Big\{\min\{\qu T, \sqrt{\qi}\}, 
    \frac{\qu T}{\log\qi}\Big\}
  > 
    \begin{cases}
    \qu T, \quad \text{ if } T<\sqrt{\qi}/\qu\\
     \frac{\qu T}{\log\qi} \quad \text{ if } T\geq \sqrt{\qi}/\qu
    \end{cases}
    \\& > \frac{\qu T}{\log(\qu T)} >  \frac{\qu T}{\log(N\Rtt(T))}\,.
 \end{align*}
where we used $\log(\qu T) < c \log \qi$ for $T\geq \sqrt{\qi}/\qu$.
Also, 
 \begin{align*}
  N\reg(T) &\geq 
    cN =c\frac{N\log\qu}{\log\qu}>
  c\frac{N\log\qu}{\log(N\Rtt(T))}\,.
 \end{align*}
The above two displays together gives~\eqref{eq:LBRegretRT} for $T\in \TintUser$. 

\paragraph*{Hybrid regime $T\in \TintHyb$:}

For $T\in  \TintHyb$ we have \[N\Rtt(T)=N\log\qu+ \sqrt{T\qi\qu}\,.\]
This regime occurs only when $\log\qi \leq \qu$. 

$\TintHyb = (T_3, T_5]$ with $T_3 = \qi/\qu$ and $T_5=\qi\qu$ when $T_1=\log \qu< \qi/N= T_2$.

$\TintHyb = (T_4, T_5]$ with $T_4 = N(\log\qu)^2/\qi$ when $T_1\geq T_2$.

Note that 
 \begin{align*}
  N\reg(T) &\geq 
   c N \geq 
    c\frac{N\log\qu}{\log\qu}
    >  c/2\frac{N\log\qu}{\log(N\Rtt(T))}\,.
 \end{align*}
Now, we have to show $ N\reg(T) \geq  c\frac{\sqrt{T\qi\qu}}{\log(N\Rtt(T))}$ to get Eq.~\eqref{eq:LBRegretRT}.

We will focus on three different regimes of horizon $T$:
\begin{itemize}
    \item Let $T>T_3 := \qi/\qu$ then $\qu T> \sqrt{T\qi \qu}$. Hence, using Eq.~\eqref{eq:LBRegretApproxtl}, 
     \begin{align*}
        N\reg(T) &\geq 
        c\min\Big\{\frac{\qu T}{\log\qi}, \sqrt{T\qi\qu}\Big\}  
        > c\frac{\sqrt{T\qi\qu}}{\log\qi}
        >  c/2\frac{\sqrt{T\qi\qu}}{\log(N\Rtt(T))}\,.
     \end{align*}
    \item Let $\max\{\log\qu \log\qi, N(\log\qu)^2/\qi\} < T \leq \qi/\qu$. The bounds on the interval guarantee that $ \min\Big\{\frac{NT}{\log\qi},\sqrt{T\qi N}, N\log\qu\Big\} 
       = N\log\qu $ and $\qi \geq \sqrt{T\qi\qu}$.
       
    This interval has overlap with $\TintHyb$ only when $T_1\geq T_2$, hence $N\log\qu \geq \qi$. Using Eq.~\eqref{eq:LBRegretApproxtl},
    \begin{align*}
        N\reg(T) &\geq 
       c\min\Big\{\frac{NT}{\log\qi},\sqrt{T\qi N}, N\log\qu\Big\} 
       \\& = cN\log\qu \geq c\qi \geq c\sqrt{T\qi\qu}\,.
    \end{align*}
    \item Let $T_4 = N(\log\qu)^2/\qi< T \leq \min\{\log\qu \log\qi,\qi/\qu\} $. 
     This interval has overlap with $\TintHyb$ only when $T_1\geq T_2$, i.e.,  $N\log\qu \geq \qi$ and consequently $N(\log\qu)^2/\qi \geq \qi/N$.
   
     Also, 
     \begin{align*}
        \qi/N\leq  N(\log\qu)^2/\qi< T \text{ implies }
         \begin{cases}
         \sqrt{T\qi N} < N \log\qu\\
          \sqrt{T\qi N} <NT
         \end{cases}
     \end{align*}
     Hence, 
      \begin{align*}
        N\reg(T) &\geq 
       c\min\Big\{\frac{NT}{\log\qi},\sqrt{T\qi N}, N\log\qu\Big\} 
       \\& \geq  c\frac{\sqrt{T\qi N}}{\log\qi} 
       > c\frac{\sqrt{T\qi \qu}}{\log\qi} \geq c/2 \frac{\sqrt{T\qi \qu}}{\log(N\, \Rtt(T))} \,.
    \end{align*}
\end{itemize}

\paragraph*{Asymptotic regime $T\in \TintAsym$ and $\log\qi \leq \qu$ :}

For $T\in  \TintHyb = (T_5, \infty)$ for $T_5 = \qi \qu$, we have 
\[N\Rtt(T)=N\log\qu+ \qi\qu + T\log\qi\,.\]

Using Eq.~\eqref{eq:LBRegretApproxtl}, 
\begin{align*}
    N\reg(T) & \geq c\max\Big\{N, \min\{T\qu, T\log\qi\}\Big\} = c/2(T\log\qi+ N)\,.
\end{align*}
Since $T>T_5 = \qi\qu$ then $ T\log\qi>T >\qu\qi$. 
Also $\log\qu <\log N < \log(N\,\Rtt(T))$. This all gives
\[ CN\reg(T) > T\log\qi+ N > c\frac{T\log\qi + \qi \qu + N\log\qu}{\log\qu}  > c\frac{N\,\Rtt(T)}{\log(N\,\Rtt(T))}\,.\]

\paragraph*{Asymptotic regime $T\in \TintAsym$ and $\log\qi > \qu$ :}

For $T\in  \TintHyb$, we have 
\[N\Rtt(T)=N\log\qu +  T\qu\,.\]
Using Eq.~\eqref{eq:LBRegretApproxtl}, 
\begin{align*}
    N\reg(T) & \geq c\max\Big\{N, \min\{T\qu, T\log\qi\}\Big\} = c/2(T\qu+ N)\,.
\end{align*}
Also $\log\qu <\log N < \log(N\,\Rtt(T))$. This all gives
\[ CN\reg(T) > T\qu + N > \frac{T\qu + N\log\qu}{\log\qu}  > \frac{N\,\Rtt(T)}{\log(N\,\Rtt(T))}\,.\]

\section{Bounding Cold Start Time -- proof of Corollary~\ref{cor:coldstart}}

\label{sec:ProofColdStart}

\begin{corn}[\ref{cor:coldstart}]
Under the modeling assumptions in Figure~\ref{f:assumptions},
and assuming  $\qu> (\log\qi)^2$ and $\qi>(\log N)^5$, there are universal constants $\gamma, \gamma' >0$ such that the following hold:

The cold start time of $\recsys$ is upper bounded as
\begin{align}
    \label{eq:UBcoldStart2}
    \textsf{coldstart}(\gamma) \leq 
\min\Big\{  \log^2 N, \max\big\{\frac{\qi \log \qi}{N}\log^2(N\qi),\,16\big\}
\Big\}\,.
\end{align}

Also, the cold start time of any algorithm is lower bounded as
\begin{align}
    \label{eq:LBcoldStart2}
   \textsf{coldstart}(\gamma') \geq 
\min\Big\{  \log N\,\log\qu, \max\big\{ \frac{\qi \log^2 \qi}{N},16\big\}\Big\}\,.
\end{align}
\end{corn}

Proving this Corollary is primarily computational and doesn't offer significant conceptual insights. The reader can skip it without missing crucial information or essential takeaways. We included it for the sake of thoroughness.

\subsection{Upper Bound on Cold Start}
In this section, we prove the upper bound on the cold start time. Using the definition of cold start time in Eq.~\eqref{eq:coldstartdef}, to get the upper bound in cold start time in Eq.~\eqref{eq:UBcoldStart2}, we will show that there exists $\gamma>0$ such that if there exists some $T$ such that 
\begin{align}
\label{eq:RecSysCold}
    T>  \min\Big\{  \log^2 N, \max\big\{\frac{\qi \log \qi}{N}\log^2(N\qi), 16\big\}
\Big\}
\end{align}
then 
\begin{align}
\label{eq:RecSysRegCold}
\reg(T) \leq \gamma \frac{T}{\log(NT)}
\end{align}

Note that according to Theorem~\ref{th:Item-upper}, $\recsys$ algorithm achieves
\[\reg(T)
\leq
 \min\Big\{ \frac{T}{2}, \,\,C \Rtt_U(T), C\Rtt_I(T)\Big\}\,,\]
 for $\Rtt_U(T)$ and $\Rtt_I(T)$ defined in Eq.~\eqref{eq:RegUSer} and~\eqref{eq:RegItem}.
So it suffices to show that if there exists some $T$ lower bounded as in~\eqref{eq:RecSysCold}, then 
$$ \min\Big\{ NT,  N\Rtt_U(T), N\Rtt_I(T)\Big\}\leq \gamma\frac{NT}{\log(NT)}$$

Throughout, we will use the $2\log (N\qu)<\tuu <  4\log (N\qu)$ and $2\log(N\qi)<\tiu <3\log(N\qi) $. We will also use the assumptions in Figure~\ref{f:assumptions} including $N>100$\footnote{Throughout, we will use the property that if $x>2A\log A$ then $x/\log x > A$ and if $x<A\log A$ then $x/\log x <A$.} .

\paragraph*{Step 1}
We consider some $T$ such that $\log^2 N < T < N$. Then we show $N\Rtt_U(T)<\gamma \frac{NT}{\log NT}$ to get~\eqref{eq:RecSysRegCold}\footnote{Such $T$ exists since $N>100$.}.

Since $T>\log^2 N$ then 
$$NT> N\log^2N > 1/2(N\log N) \log(N\log N)$$ and  $$ \frac{NT}{\log(NT)}> (N/4)\log N >\frac{1}{32} N\tuu\,.$$

Also, since $T<N$ and $N>20\qu\log^2\qu$ 
then $N/\log N \geq \qu$ and hence
$$ \frac{NT}{\log(NT)}> \frac{NT}{2\log N} > \frac{1}{2}\qu T\,.$$

so fo any $\log^2 N<T<N$, 
$$\reg(T) < c\Rtt_U(T)<\frac{T}{\log(NT)}\,.$$

\paragraph*{Step 2}
If $T>16$, then $ \frac{NT}{\log(NT)}>N\log T $.

If $T>\frac{\qi \log \qi}{N}\log^2(N\qi)$,
then 
\begin{align*}
    \sqrt{NT}&> \sqrt{\qi(\log\qi) \log^2(N\qi)}
>\sqrt{\qi\log(N\qi)} \, \sqrt{(\log\qi )\log(N\qi)}
\\&>\sqrt{\qi\log(N\qi)} \, \log(\qi \log(N\qi)) > \sqrt{\qi\tiu/3} \log(\qi\tiu/3)
\end{align*}
and hence $$ \frac{NT}{\log(NT)}>c\sqrt{\qi T\tiu N}$$

So if $T>\max\Big\{16, \frac{\qi \log \qi}{N}\log^2(N\qi)
\Big\}$, then $ \frac{NT}{\log(NT)}>c N\Rtt_I(T)$.

\subsection{Lower Bound on Cold Start}

In this section, we prove the lower bound on the cold start time. Using the definition of cold start time in Eq.~\eqref{eq:coldstartdef} we will show that there exists $\gamma'>0$ such that for all $T$ such that
\begin{align}
\label{eq:coldLBproof}
    T<  \min\Big\{  \log N\,\log\qu, \max\big\{ \frac{\qi \log^2 \qi}{N},16\big\}
\Big\}
\end{align}
then 
\begin{align}
\label{eq:RegColdLB}
\reg(T) > \gamma'\frac{T}{\log(NT)}
\end{align}

Note that according to Theorem~\ref{t:joint-L},
\begin{align*}
    N\reg(T) \geq c
    \max\Big\{&N, \min\{NT, N\log \qu, \sqrt{\qi}\},
    \min\Big\{\frac{NT}{\log\qi},\sqrt{T\qi N}, N\log\qu\Big\}\Big\}\,.
\end{align*}
We use this to show that for all $T$ lower bounded as in~\eqref{eq:coldLBproof}, the lower bound in Eq.~\eqref{eq:RegColdLB} holds. 

We will also use the assumptions in Figure~\ref{f:assumptions}. 
Additionally, we will use use the assumption $\qu> (\log\qi)^2$ and $\qi>(\log N)^5$  and definitions of $\tul$ and $\til$ to get $c\log\qu<\tul\leq C\log \qu$ and  and $c\log\qi < \tiu < C\log \qi$.
 For the sake of simplicity, we will ignore all universal constant factors denoting them by $c$\footnote{Throughout, we will use the property that if $x>2A\log A$ then $x/\log x > A$ and if $x<A\log A$ then $x/\log x <A$.} .

We look at different regimes for $T$ seprately:
\paragraph*{Step 1}
If $T<16$ then 
\begin{align*}
    N\reg(T) \geq cN  > c' \frac{NT}{\log(NT)}\,.
\end{align*}

\paragraph*{Step 2}
If $T<\min\{\log \qu, \sqrt{\qi}/N\}$ for a constant $\Gamma$, we have 
\begin{align*}
    N\reg(T) \geq c\min\{NT, N\log \qu, \sqrt{\qi}\}
    \geq cNT > c\frac{NT}{\log(NT)}\,.
\end{align*}

\paragraph*{Step 3}
If $\log\qu \leq T<\min\{(\log N)(\log \qu), \sqrt{\qi}/N\}$ we have 
\begin{align*}
    N\reg(T) \geq c\min\{NT, N\log \qu, \sqrt{\qi}\}
    \geq cN \log\qu = cN \log\qu \frac{\log N}{\log N}> c\frac{NT}{\log(NT)}\,.
\end{align*}

\paragraph*{Step 4}
If $\sqrt{\qi}/N \leq T<\min\{(\log N)(\log \qu), (\qi/N) \log^2\qi \}$ we have 
\begin{align*}
    N\reg(T) \geq c  \min\Big\{\frac{NT}{\log\qi},\sqrt{T\qi N}, N\log\qu\Big\}
    \geq c\frac{NT}{\log\qi}> c\frac{NT}{\log(NT)}\,.
\end{align*}
where in the last inequality we used $NT\geq \sqrt{\qi}$ hence $\log(NT) \geq c\log\qi$.

\section{Concentration Lemmas} \label{s:lemmas}

\subsection{Tail Bounds}

The following lemma is derived by application of Chernoff bound to Binomial variables \cite{chung2006concentration}.
\begin{lemma}[Chernoff bound]\label{l:Chernoff}
Let $X_1,\cdots,X_n \in [0,1]$ be independent random variables. Let $X=\sum_{i=1}^n X_i$ and $\bar{X}=\sum_{i=1}^n \Ex X_i$. Then, for any $\epsilon>0$,
\begin{align*}\Pr\big[X\geq (1+\epsilon)\bar{X}\big]
&\leq \exp\Big( - \frac{\epsilon^2}{2+\epsilon}\bar{X}\Big)
\nonumber\\&
\leq \max\Big\{\exp\Big( - \frac{\epsilon^2}{3}\bar{X}\Big), \exp\Big( - \frac{\epsilon}{2}\bar{X}\Big)\Big\}
\\
\Pr\big[X\leq (1-\epsilon)\bar{X}\big]
&\leq \exp\Big( - \frac{\epsilon^2}{2}\bar{X}\Big)
\\
\Pr\big[|X-\bar{X}|\geq \epsilon\bar{X}\big]
&\leq 2 \max\Big\{\exp\Big( - \frac{\epsilon^2}{3}\bar{X}\Big), \exp\Big( - \frac{\epsilon}{2}\bar{X}\Big)\Big\}
\end{align*}
\end{lemma}

\begin{lemma}[McDiarmid~\cite{mcdiarmid1998concentration}]\label{l:martingaleBound}
	Let $X_1,\cdots,X_n$ be a martingale adapted to filtration $(\Fc_n)$ satisfying 
	\begin{enumerate}
		\item[(i)] $\mathrm{Var}(X_i|\Fc_{i-1})\leq \sigma_i^2$,  for $1\leq i\leq n$, and
				\item[(ii)]	$|X_i - X_{i-1}|\leq M$, for $1\leq i\leq n$.
\end{enumerate}
Let $X = \sum_{i=1}^n X_i$.
Then
$$
\Pr\big[X - \Ex X \geq r\big] \leq \exp\bigg( -\frac{r^2}{2\big(\sum_{i=1}^n \sigma_i^2+Mr/3\big)}\bigg)\,.
$$
\end{lemma}

\subsection{Lemmas on Balls and Bins Probelm}
\label{sec:BallsBins}
\begin{lemma}[Balls and bins: tail bound for number of nonempty bins]
\label{l:ballsbins}
Suppose $m \leq \frac{4}{15}n$. If $m$ balls are placed into $n$ bins each independently and uniformly at random, then with probability at least $1-\exp(-m/2)$ at least $m/2$ bins are nonempty.
\end{lemma}

\begin{proof}
Let random variable $E$ be the number  of empty bins and $n-E$ is the number of nonempty bins. We want to show that $\Pr[n-E \geq m/2] >1-\exp(-m/2)$.

Any configuration with at most $m/2$ nonempty bins has at least $n-m/2$ empty bins. Thus to upper bound $\Pr[E > n-m/2] $ we may bound the probability of having \emph{some} set of $n-m/2$ bins be empty. 
There are ${n\choose n-m/2}={n\choose m/2}$ possible choices for these empty bins, and each ball has to land outside of these, which has probability $[(m/2)/n]^m$. Thus, the probability of at most $m/2$ nonempty bins is bounded by
 $${n\choose m/2}\left(\frac{m/2}{n}\right)^m\leq \left(\frac{n\cdot e}{m/2}\right)^{m/2} \left(\frac{m/2}{n}\right)^m \leq \left(\frac{m e}{2n}\right)^{m/2}\leq \exp(-m/2)\,,$$
where we used $m \leq \frac{4}{15}n$ in the last inequality.
\end{proof}

\begin{lemma}[Balls and bins: tail bound for number of nonempty bins]
\label{l:ballsbins2}
Suppose $m, n \geq k$. If $m$ balls are placed into $n$ bins each independently and uniformly at random, then with probability at least $1-\exp(-k/3)$ at least $k/3$ bins are nonempty.
\end{lemma}

\begin{proof}
The proof is similar to the proof of Lemma~\ref{l:ballsbins2}.

Any configuration with at most $k/3$ nonempty bins has at least $n-k/3$ empty bins. 
There are ${n\choose n-k/3}={n\choose k/3}$ possible choices for these empty bins, and each ball has to land outside of these, which has probability $[(k/3)/n]^m$. Thus, the probability of at most $k/3$ nonempty bins is bounded by
 $${n\choose k/3}\left(\frac{k/3}{n}\right)^m
 \leq 
 \left(\frac{n\cdot e}{k/3}\right)^{k/3} \left(\frac{k/3}{n}\right)^m 
 \leq 
  \left(\frac{n\cdot e}{k/3}\right)^{k/3} \left(\frac{k/3}{n}\right)^k
  \leq
 \left(\frac{ e}{9}\right)^{k/3}\leq \exp(-k/3)\,,$$
where we used $m,n\geq k$.
\end{proof}


\begin{lemma}[Balls and bins: tail bound for number of nonempty bins]
\label{l:ballsbins3}
Suppose $m\geq n\log (n/\delta)$. If $m$ balls are placed into $n$ bins each independently and uniformly at random, then with probability at least $1-\delta$ all bins are nonempty.
\end{lemma}

\begin{proof}
The probability of having at least on empty bin is at most $n (1-1/n)^m \leq n \exp(-m/n)$.
\end{proof}

%
%
%

The following lemma records a simple consequence of linearity of expectation.
\begin{lemma}[Balls and bins: bound for the expected number of nonempty bins]
\label{l:ballsbins4}
If we throw $m$ balls into $n$ bins independently uniformly at random, then, the expected number of nonempty bins is $n(1-(1-1/n)^m)\,.$
\end{lemma}

\section{Converting to Anytime Regret} \label{s:any-time-reg-alg}
The \textit{doubling trick} converts 
an online algorithm designed for a finite known  time horizon to an algorithm that does not require knowledge of the time horizon and yet achieves the same regret (up to multiplicative constant) at any time \cite{cesa2006prediction,lattimore2020bandit} (i.e., \textit{anytime regret}). 

The trick is to divide time into intervals and restart algorithm at the beginning of each interval. 
Let $A(T)$ be an online algorithm taking the known time horizon as input and achieving regret $\mathrm{R}(T)$ at time $T$. 
There are two regret scalings of interest.  (1) If $\mathrm{R}(T) = O(T^\alpha)$ for some $0<\alpha < 1$, then to achieve anytime regret, the doubling trick uses time intervals of length $2, 2^2, 2^3,.. , 2^m$. This achieves regret of at most $\mathrm{R}(T)/(1-2^\alpha)$ at time $T$ for any $T$. (2) Alternatively, if $\mathrm{R}(T)= O(\log T)$, then using intervals of length $2^2, 2^{2^2}, .., 2^{2^m}$ achieves regret of at most $4\mathrm{R}(T)$ at time $T$ for any $T$.

Clearly, different scalings can be used before and after some threshold $\mathsf{\tau}$ if the algorithm achieves regret $O (\log T)$ for $T<\mathsf{\tau}$ and $O(\sqrt{T})$ if $T\geq \mathsf{\tau}$, as is the case for the proposed algorithm.

\end{document}